\documentclass[10pt,journal,compsoc]{IEEEtran}
\usepackage{subfigure}

\usepackage{booktabs}
\usepackage{enumitem}
\usepackage{makecell}
\usepackage{algorithm}
\usepackage{amsfonts, amssymb}
\usepackage{algorithmic}
\usepackage{mathtools}
\usepackage{amsthm}
\usepackage{setspace}
\usepackage{bm}
\usepackage{graphicx}
\usepackage{multirow}
\usepackage{amssymb} 
\usepackage{ulem}
\usepackage{hyperref}
\usepackage{cuted}%%\stripsep-3pt
\usepackage[utf8]{inputenc}
\usepackage{ragged2e}
\usepackage[english]{babel}
\newtheorem{theorem}{Theorem}
\newtheorem{corollary}{Corollary}[theorem]
\newtheorem{lemma}[theorem]{Lemma}
\newtheorem{definition}{Definition}
\newtheorem{assumption}{Assumption}

\hypersetup{hidelinks,
	colorlinks=true,
	allcolors=black,
	pdfstartview=Fit,
	breaklinks=true}

\usepackage{color}

\ifCLASSOPTIONcompsoc
  \usepackage[nocompress]{cite}
\else
  \usepackage{cite}
\fi
\ifCLASSINFOpdf
\else
\fi
\begin{document}
\renewcommand\arraystretch{0.5}
\title{Transferable Time-Series Forecasting under Causal Conditional Shift}

\author{Zijian Li,
        Ruichu Cai*,
        Tom Z. J. Fu, Zhifeng Hao ~\IEEEmembership{Senior Member,~IEEE,} and Kun Zhang% <-this % stops a space
\IEEEcompsocitemizethanks{\IEEEcompsocthanksitem
Zijian Li is with the School of Computing, Guangdong University of Technology, Guangzhou China, 510006.\protect
E-mail: leizigin@gmail.com
\IEEEcompsocthanksitem Ruichu Cai is with the School of Computer Science, Guangdong University of Technology, Guangzhou, China, 510006 and Peng Cheng Laboratory, Shenzhen, China, 518066.
Email: cairuichu@gmail.com\protect
\IEEEcompsocthanksitem Tom Z. J. Fu is with the School of Computer Science, Guangdong University of Technology, Guangzhou China, 510006.\protect Email: fuzhengjia@gmail.com\protect
\IEEEcompsocthanksitem Zhifeng Hao is with the College of Science, Shantou University, Shantou, Guangdong, China, 515063.\protect Email: haozhifeng@stu.edu.cn\protect
% \IEEEcompsocthanksitem Zhifeng Hao is with the College of Science, Shantou University, Shantou, Guangdong, 515063\protect
\IEEEcompsocthanksitem Kun Zhang is with the Department of Philosophy, Carnegie Mellon University, Pittsburgh, PA 15213 USA. E-mail: kunz1@cmu.edu\protect
}% <-this % stops an unwanted space
\thanks{This research was supported in part by  the National Key R\&D Program of China (2021ZD0111501), National Science Fund for Excellent Young Scholars (62122022), Natural Science Foundation of China (61876043, 61976052), the major key project of PCL (PCL2021A12). (*Ruichu Cai is the Corresponding author.)}}


\markboth{Journal of \LaTeX\ Class Files,~Vol.~14, No.~8, August~2015}%
{Shell \MakeLowercase{\textit{et al.}}: Bare Demo of IEEEtran.cls for Computer Society Journals}



\IEEEtitleabstractindextext{%
\begin{abstract}\justifying
This paper focuses on the problem of semi-supervised domain adaptation for time-series forecasting, which is underexplored in literature, despite being often encountered in practice. Existing methods on time-series domain adaptation mainly follow the paradigm designed for static data, which cannot handle domain-specific complex conditional dependencies raised by data offset, time lags, and variant data distributions. In order to address these challenges, we analyze variational conditional dependencies in time-series data and find that the causal structures are usually stable among domains, and further raise the causal conditional shift assumption. Enlightened by this assumption, we consider the causal generation process for time-series data and propose an end-to-end model for the semi-supervised domain adaptation problem on time-series forecasting. Our method can not only discover the Granger-Causal structures among cross-domain data but also address the cross-domain time-series forecasting problem with accurate and interpretable predicted results. We further theoretically analyze the superiority of the proposed method, where the generalization error on the target domain is bounded by the empirical risks and by the discrepancy between the causal structures from different domains. Experimental results on both synthetic and real data demonstrate the effectiveness of our method for the semi-supervised domain adaptation method on time-series forecasting. 

% The compact causal structures can not only address the offset problem by avoiding directly aligning the representation like traditional domain adaptation methods but also simultaneously portray domain-invariant and the domain-specific modules of the conditional distribution. This further enlightens us to devise an end-to-end model 
% for transferable time-series forecasting. The proposed method can not only discover the cross-domain \textit{Granger Causality} but also address the cross-domain time-series forecasting problem, it can even provide the interpretability of the predicted result to some extend. We further theoretically analyze the superiority of the proposed methods. Experimental results on both synthetic and real data demonstrate the effectiveness of the proposed method for transferable time-series forecasting. 

\end{abstract}
% 1. 解决什么问题（DA）
% 2. 难点在哪儿，联合分布不知道，可以说什么是变的
% 3. 解决方法，提出motvation（一句话让别人知道大概怎么做，而且是惊人的想法）花两三句话来详细介绍自己的方法
% 4. 这样做的好处有什么，实验结果，代码等等
% 我的版本
% 解决时间序列预测的DA问题
% 难点在于基于多维时间序列的复杂依赖是根据domain变化而变化的，这个变化主要体现在三个方面：Offset，Time Lag, 和条件分布的变化。
% 为了解决这个问题，我们使用稳定的简洁的格兰杰因果关系来建模这种复杂的domain-specific的依赖关系。这种简洁的格兰杰因果关系不但可以通过避免直接对齐数据值从而绕开offset问题，而且通过考虑不同time lag之间的因果图从而可以很好地刻画条件分布的变和不变的部分。
% 通过这样做，我们提供了一种端对端的时间序列预测迁移模型，它不仅仅能够跨领域学习格兰杰因果关系，而且很好地通过学习不同domain的依赖关系来预测多维时间序列，更能为预测结果提过一定的可解释性。我们进一步从理论上分析了基于格兰杰因果关系对齐的时间序列预测迁移模型的优越性，我们在生成数据集和真实数据集上验证了我们的方法。代码公布xxxx

}

\maketitle

\IEEEdisplaynontitleabstractindextext

% \fi
\IEEEpeerreviewmaketitle

\IEEEraisesectionheading{\section{Introduction}\label{sec:introduction}}

\textcolor{black}{Science is all about generalizations. Scientific applications of new theorems or discoveries also need to transfer the experimental discoveries from the lab or the virtual environments to real environments.} Domain adaptation \cite{pan2009survey,zhang2019bridging,li2021learning,stojanov2021domain}, which leverages the labeled source domain data (or a few labeled target domain data) to make a prediction in the unlabeled data from the target domain, essentially aims to address the intractable \textit{domain shift} phenomenon and find a robust forecasting model.

Various methods have been proposed for domain adaptation \cite{cai2021graph,hao2021semi,shui2021aggregating,li2021causal}, with \textit{covariate shift assumption} that the marginal distribution $P(X)$ varies with different domains while the conditional distribution $P(Y|X)$ is stable across domains. Based on this assumption, methods based on MDD \cite{long2015learning,pan2010domain,Cai_Chen_Li_Chen_Zhang_Ye_Li_Yang_Zhang_2021} or adversarial training \cite{cai2019learning,ganin2015unsupervised,xie2018learning,cai2019learning}, which aim to extract the domain-invariant representation \textcolor{black}{ for ideal transferability}, have been proposed and achieved great successes in non-time series data. Considering that the covariate shift scenario may not be satisfied, Zhang et.al \cite{zhang2013domain} leverage the thoughts of causality and further study three different possible scenarios: the target shift, the conditional shift, and the generalized target shift. Recently, Zhang et.al \cite{zhang2020domain} describe the changes of the data distribution by leveraging the bayesian treatment and propose a causality-based method to address the unsupervised domain adaptation problem.

Recently, domain adaptation for time-series data is receiving more and more attention. Though many researchers expand the aforementioned ideas on non-time series data to the field of time-series data \cite{da2020remaining,purushotham2016variational} and make a brave push, these methods implicitly reuse the covariate shift assumption for non-time series data to time-series data. 
\textcolor{black}{The covariate shift is one of the most conventional assumptions that is used by the existing time-series domain adaptation methods. It simply assumes $P^S(\bm{z}_{t+1}|\bm{z}_{t},\cdots,\bm{z}_1)=P^T(\bm{z}_{t+1}|\bm{z}_{t},\cdots,\bm{z}_1)$ and essentially takes all the relationships among variables into consideration, which is shown in Figure~\ref{fig:motivation1} (a). }

\begin{figure}
	\centering
	\includegraphics[width=0.9\columnwidth]{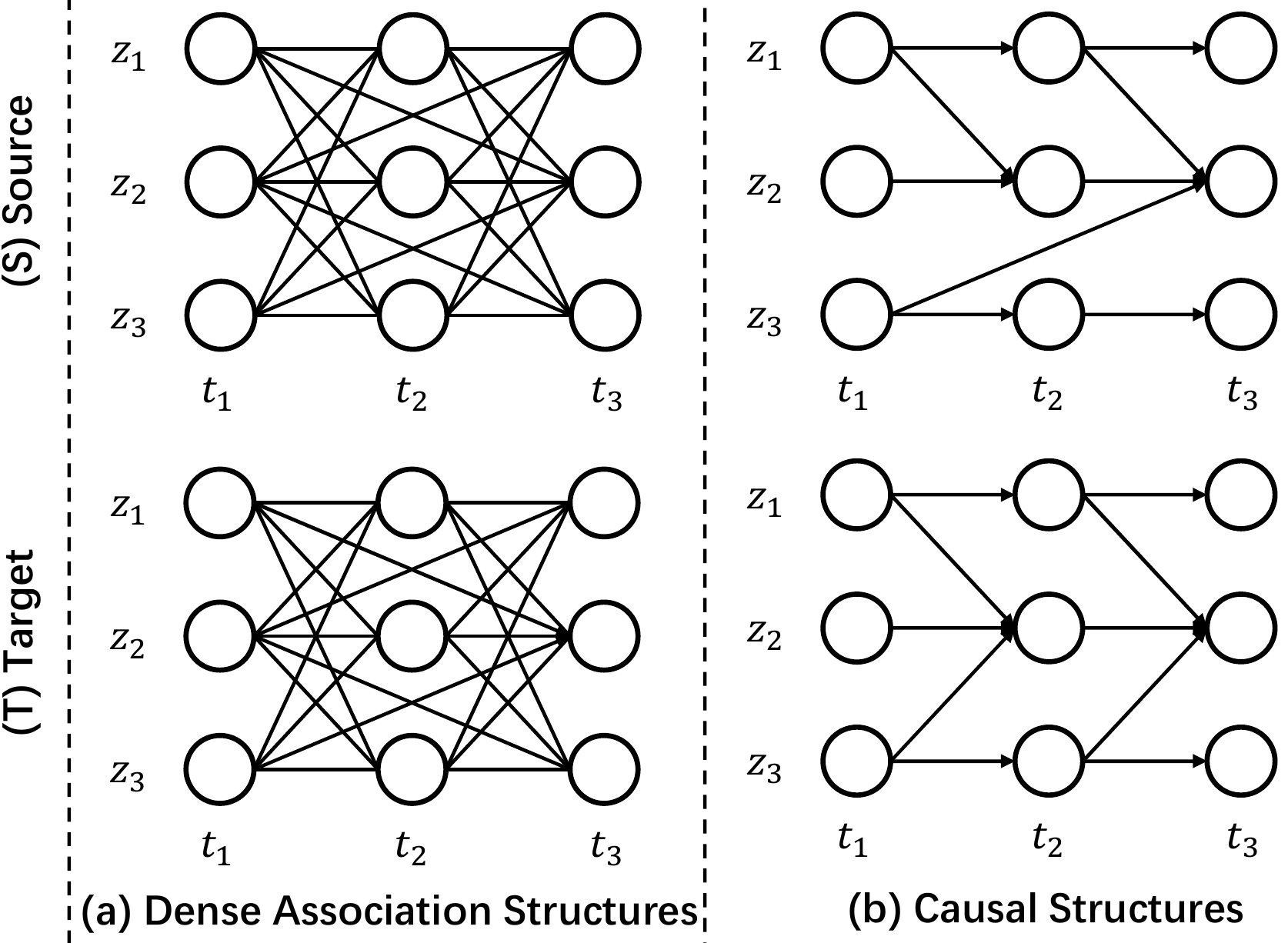}
	\caption{The illustration of \textcolor{black}{different methods of time-series domain adaptation in} three dimensions time-series data. (a) The existing methods, which reuse the covariate shift assumption, take all the relationships into account \textcolor{black}{(including the fake relationships)} and may result in \textcolor{black}{the degenerated performance}. (b) Considering the stable and \textcolor{black}{similar} causal generalization process of time-series data, the proposed method addresses the domain adaptative time-series forecasting problem by extracting the causal structures and modeling the domain-specific conditional distribution.}
	\label{fig:motivation1}
\end{figure}

% However, the covariate distribution shift assumption \textcolor{black}{is hard to be satisfied when it comes to time-series data}, since a tiny change of $\bm{z}_1$ may lead to a huge change of $\bm{z}_{t+1}$.

\textcolor{black}{However, because of the complex dependencies of time series, (i.e., the 1st-order Markov dependence results in the dependence between any two-time steps), the conditional distribution of time-series data $P(\bm{z}_{t+1}|\bm{z}_{1}, \cdots \bm{z}_{t})$ varies sharply \textcolor{black}{among} different domains, making the conventional covariate shift assumption hard to be satisfied in time-series data.}
\textcolor{black}{Let us give an easily comprehensive example: \textcolor{black}{We first assume the following data generation process:} $z_t^1 = 1.5 \times z_{t-1}^1;
    z_t^2 = 3 \times z_{t-1}^2 + 2 \times (z_{t-1}^1)^3.$ 
We further let $z_0^1=1, z_0^2=1$, so $z_5^2=208403$. If we change  $z_0^1=1.1$, then $z_5^2 \approx 277303$. This example shows that even a small modification of $z_0^1$ may result in a huge change of $z_5^2$.} \textcolor{black}{Hence}, it is a challenging task to \textcolor{black}{devise a transferable and robust model for} time-series data.
% Focusing on this problem, other methods like \cite{Cai_Chen_Li_Chen_Zhang_Ye_Li_Yang_Zhang_2021} consider that most of the previous time-series domain adaptation methods take all the relationships among variables into account and extract the stable and sparse associative structures among variables, which is shown in Figure~\ref{fig:motivation1} (b).
% However, \textcolor{black}{they only consider the sparse associative structures of the time-series data.}. What is more, they do not consider the domain-specific strengths among the sparse associative structures. 
Fortunately, as shown in Figure~\ref{fig:motivation1}(b), we find that the causal structures among different domains are stable and compact. Since the generation process of time-series data usually follows physical rules, which actually denote causality. Therefore, the causal structures among the source and the target domains are also similar and stable. \textcolor{black}{Give an example in the medicinal field}, in the physiological mechanism, the causal structure among `Blood Glucose'', ``Glucagon'' and ``Insulin'' stably holds for everyone. 
Inspired by the aforementioned observation, it is natural and intuitive to avoid the conventional covariate shift assumption and tackle the semi-supervised time-series forecasting domain adaptation problem with the help of causality. 

Following the aforementioned intuition, we propose a solution to semi-supervised transferable time-series forecasting by assuming that the causal structures are similar and stable across domains. We first analyze the drawbacks in the existing time-series domain adaptation methods and raise the \textbf{Causal Conditional Shift Assumption}. Under this assumption, we explore the time-series forecasting domain adaptation problem by borrowing the idea of the causal generation process of time-series data. Based on this insight, we devise the \textbf{G}ranger \textbf{C}ausality \textbf{A}lignment model (\textbf{GCA} in short). Technically, we devise a Recurrent Granger Causality Variational Autoencoder, \textcolor{black}{which is composed of the recurrent Granger Causality Reconstruction module and the domain-sensitive Granger-Causality-based prediction module. The recurrent Granger Causality Reconstruction module} is used to discover the Granger Causality with different lags from the source and the target domains. Considering that the strengths of causal structures from different domains may be different, we devise the domain-sensitive Granger-Causality-based prediction module. We further propose the Granger Causality \textcolor{black}{discrepancy} regularization, which is beneficial to reconstruct the target causal structures and restrict the discrepancy of the causal structures from different domains. \textcolor{black}{Moreover, by discovering and aligning the causal structures, the proposed method can not only provide model interpretability but also facilitate the usage of expert knowledge.} Our proposed method outperforms the state-of-the-art domain adaptation methods for time-series data on synthetic and real-world datasets. Moreover, we not only provide a theoretical analysis of the advantages of the causality mechanisms as well as the visualization and interpretability results.

\textcolor{black}{The rest of the paper is organized as follows. Section \ref{related_works} reviews the existing studies on domain adaptation on non-time series data and time-series data as well as the works about Granger Causality. We also elaborate on the proposed semi-supervised time-series domain adaptation model under causal conditional shift assumption in section \ref{model}. \textcolor{black}{In section \ref{sec:vae}, we introduce how to learn the distribution of time-series data with Granger Causality and variational inference.} Then we provide the implementation detail of the Granger Causality Alignment approach in section \ref{sec:imple}. Section \ref{theorm} presents the theoretical analysis of the proposed method. And section \ref{exp} shows the experimental results as well as the insightful visualization. We conclude the paper with future works discussion in section \ref{conclusion}.} 

% 存在问题：
% 1. 格兰杰因果出现得有点突然
% 2. 应该强调时序的因果结构
% 3. 强调时序的可迁移性质
% 4. 推出SASA仅仅是一个特例，再提及SASA的不足

\section{Related Works}\label{related_works}
In this section, we mainly focus on the existing techniques on domain adaptation on non-time series data, time-series domain adaptation as well as Granger Causality.

\subsection{Domain Adaptation on Non-Time Series Data} Domain Adaptation \cite{pan2009survey,zhang2019bridging,ganin2015unsupervised,xie2018learning,long2015learning,zhang2013domain, stojanov2021domain}, which leverages the labeled source domain data and limited labeled or unlabeled target domain data to make predictions in the target domain, have applications in various fields \cite{hao2021semi,shui2021aggregating,li2021causal,cai2021graph}. The Approaches of domain adaptation follow the covariate shift assumption and aim to extract the domain-invariant representation. Technically, these methods can be grouped into Maximum Mean Discrepancy-based methods and adversarial training methods. 

In view of the causal generation process, some researchers find that the conditional distribution $P(Y|X)$ usually changes and the covariate shift assumption does not hold, and they address such limitations from the views of causality. Zhang et.al \cite{zhang2013domain, zhang2015multi} address this problem by assuming only $P(Y)$ or $P(X|Y)$ changes and raising the target shift and conditional shift. 
Cai et.al \cite{cai2019learning,cai2021graph} are motivated by the causal generation process and raise disentangled semantic representation for unsupervised domain adaptation. Zhang et.al \cite{zhang2020domain} consider domain adaptation as the problem of graphical model inference and model the changes of the conditional distribution. \textcolor{black}{Recently, Stojanov and Li et.al \cite{stojanov2021domain} prove that the domain-invariant representation can not be extracted by a single encoder and take into account the domain-specific information. Li et.al \cite{li2021learning} address the semi-supervised domain adaptation under label shift. Finding that the unsupervised domain adaptation methods perform poorly with a few labeled target data, Saito et.al \cite{saito2019semi} propose the MME approach.} 
In this paper, we focus on semi-supervised domain adaptation on time-series data, which is more challenging because the complex conditional dependencies among time stamps are hard to be modeled.

\subsection{Domain Adaptation on Time-Series Data}
Time-series data is another type of common data. Recently, increasing attention is paid to the domain adaptation on time-series data. Previously, Da Costa et al. \cite{da2020remaining} straightforwardly extend the idea of domain adaptation for non-time series data and leverage the RNN as the feature extractor to extract the domain-invariant representation. Purushotham et al. \cite{purushotham2016variational} further improve it by using the variational recurrent neural network \cite{chung2015recurrent}. However, these methods are not able to well extract the domain-invariant information because of the complicated dependency between timestamps. \textcolor{black}{ Li et.al \cite{li2021causal} proposed the Causal Mechanism Transfer Network (CMTN), which captures and transfers the dynamic and temporal causal mechanism to alleviate the influence of time lags and different value ranges among different domains.} Recently, Cai et.al  \cite{Cai_Chen_Li_Chen_Zhang_Ye_Li_Yang_Zhang_2021} consider that the sparse associate structures among variables are stable across domains, and propose the sparse associative structure alignment methods for adaptative time series classification and regression tasks. 
\textcolor{black}{And Jin et.al \cite{jin2022domain} proposed the Domain Adaptation Forecaster (DAF), which leverages the statistical strengths from a relevant domain to improve the performance on the target domains. }
In order to address these problems, we assume that the causal structures are stable across domains and \textcolor{black}{propose the GCA approach that discovers the causal structure behind data and models how the conditional distribution changes.}

\subsection{Granger Causality}
Granger Causality \cite{diks2006new, granger1969investigating,Seth:2007, marcinkevivcs2021interpretable,9376668,lowe2020amortized}, which is a set of directed dependencies among multivariate time-series data, is widely used to determine which past portion of time-series data aids in predicting the future evolution of univariate time-series. Inferring Granger Causality is a traditional research problem and has been applied in several fields \cite{runge2019inferring,chiou2008economic,seth2015granger}. One of the most classical methods for inferring Granger Causality is the vector autoregressive (\textbf{VAR}) model \cite{10.1145/1557019.1557085, wei2006time}, which uses the linear time-lag effects as well as some sparsity techniques like the Lasso \cite{tibshirani1996regression} or the group lasso \cite{yuan2006model}. With the quick growth of the computation power, more Granger Causality inference methods that borrow the expressive power of neural networks have been proposed. Tank et al. \cite{9376668} devise a neural network-based autoregressive model and apply the sparsity penalties on the weights of the neural networks. \cite{marcinkevivcs2021interpretable} is motivated by the interpretability of self-explaining neural networks and proposes the generalized vector autoregression model to detect the signs of Granger Causality. In this paper, we are motivated by the generated process of time-series data and take the Granger-causal structure as the latent variables. We combine the variational inference framework into the vector autoregression model and further used it for semi-supervised domain adaptation of time-series forecasting.

\section{Semi-Supervised Time-Series Domain Adaptation Model  Under  Causal Conditional Shift Assumption}\label{model}
In this section, we first give a definition of the time-series generation process \textcolor{black}{based on Granger Causality}. Based on this process, we further define the problem of semi-supervised domain adaptation of time-series forecasting and finally propose the causal condition shift assumption. 

\subsection{Time-Series Generation Process \textcolor{black}{under Granger Causality}}
% 离散时间索引的时间序列数据, 一开始说明稳态 discrete time 
% 时间序列由之前的时间步骤和其他变量之前的时间步影响
% 实际上，时间序列由因果关系决定
% 基于这个因果关系，即使是一阶马尔可夫假设（即相邻两个时间步数据之间的依赖存在），使得z_1的变化（例如time lag，值域变化）对z_t有重要的影响

% 条件联合分布对齐很难--》父亲节点条件集，一种直接的观点替换assumption，告诉别人是不鲁棒的事情（针对所有条件概率的攻击，使用time lag，offset，value range攻击） p(z|pa(z))
% fortunately，因果结构是稳定的，上面的例子

In order to illustrate the time-series causal generation process, we first consider the discrete multivariate time-series data with the length of $\Upsilon$ as $\bm{z}=(\bm{z}_1, \bm{z}_2, \cdots, \bm{z}_t,\cdots \bm{z}_{\Upsilon})$, in which $\bm{z}_t\in \mathbb{R}^D$ are \textcolor{black}{the $t$-th} variables with $D$ dimensions. \textcolor{black}{With the abuse of notation, we further let $\bm{z}^d$ be the univariate time-series of $d$-th variables and $z_t^d$ be the value of $d$-th variables at $t$-th timestamp.} Intuitively, according to real-world observation, we usually observe that future values of univariate time-series data are not only influenced by its past values but also decided by the other time series. 
% For example, the changes of air temperature are not only influenced by the its previous value, but also other factors like the hydrometeor and the activities of human being. 

In fact, the data generation processes are usually controlled by a stable causal mechanism among variables. In order to describe the data generation process in the causal view, the \textcolor{black}{nonlinear Granger Causality} \cite{diks2006new, granger1969investigating,Seth:2007, marcinkevivcs2021interpretable,9376668,lowe2020amortized,lutkepohl2005new} is applied for modeling the dependencies between the child nodes and their parents.

Given the observed time-series $\bm{z}=(\bm{z}_{1}, \cdots, \bm{z}_t, \cdots, \bm{z}_{\Upsilon})$, each value of $z_t^d$ can be represented in terms of some functions $F_d$ and its parents $\bm{pa}(z_t^d)$, \textcolor{black}{i.e., those variables directly influence $z_t^d$,} \textcolor{black}{and we let the nonlinear Granger Causality follow the structural equation model \cite{pearl1998graphs} shown as Equation (\ref{equ:granger_generation}).}
\begin{equation}\label{equ:granger_generation}
	z_t^d = F_d(\bm{pa}(z_t^d),\epsilon_d),
\end{equation}
\textcolor{black}{where $\epsilon_d$ are the independent noise terms} \textcolor{black}{and $F_d(\cdot)$ denotes any flexible types of nonlinear function. Actually, Equation (\ref{equ:granger_generation}) specifies that how the future variables $\bm{z}^d_t$ relies on its past values, and $\bm{pa}(\cdot)$, which essentially coincides the nonlinear Granger Causality.} 
% Given a straightforward example shown in Figure. \ref{fig:fulltime} (a), the value of $z_3^1, z_3^2, z^3_3$ can be obtained as follow:
% \begin{equation}
% \begin{split}
% z_3^1 &= F_1(z_2^1, \epsilon_1)\\
% z_3^2 &= F_1(z_2^2, z_1^3, \epsilon_1)\\
% z_3^3 &= F_1(z_2^3, \epsilon_1).
% \end{split} 
% \end{equation}
Based on the aforementioned Granger-causal generation process for time-series data, it is natural to find that the value of $z_t^d$ actually relies on all the previous observation $(\bm{z}_{1}, \cdots,\bm{z}_{t-1})$. 
% \textcolor{black}{Combing the toy example in Section \ref{sec:introduction}, it is not hard to find that small changes on $\bm{z}_{1}^d$ will result in violent changes on $z_t^d$ even when the trivial 1st-order Markov dependence exists in any two timestamps.}
% Based on the aforementioned time-series causal generation process and from the perspective of time-series forecasting, we aim to estimate $\mathbb{E}(z_i^j|\bm{pa}(z_i^j))$.

\subsection{Semi-Supervised Domain Adaptation for Time-Series Forecasting}

% Due to the notorious phenomenon named domain shift, the observed time-series data violate the independent and identically distributed (\textbf{I.I.D}) assumption. In order to mitigate this problem, domain adaptation on time-series data is proposed.
%\textcolor{black}{In this subsection, we first provide the problem definition of semi-supervised domain adaptation for time-series data. Then we explain why the existing covariate shift assumption is hard to be satisfied in time-series data and further propose the causal conditional shift assumption. Finally, we provide the insight of how to solve semi-supervised domain adaptation for time-series forecasting via the causal conditional shift assumption.}

%\subsubsection{\textcolor{black}{Problem Definition}}
We first let the past observation and the future ground truth be $\bm{x}=(\bm{z}_1,\cdots,\bm{z}_t)$ and $\bm{y}=(\bm{z}_{t+1},\cdots,\bm{z}_{t+\tau})$. \textcolor{black}{It is noted that $\bm{y}$ can be the multivariate time-series or univariate time-series, and here we let $\bm{y}$ be multivariate time-series for generalization.} We further assume that $\{(\bm{x}_i^S, \bm{y}_i^S)\}_{i=0}^N$ and $\{(\bm{x}_i^T, \bm{y}_i^T)\}_{i=0}^M$ are respectively independent and identically distributed (I.I.D.) drawn from the source domain $P^S(\bm{x},\bm{y})$ and the target domain $P^T(\bm{x},\bm{y})$. \textcolor{black}{Note that} $M$ and $N$ are the numbers of random samples respectively drawn from the source and the target domains, and we have $N \gg M$.  Semi-supervised Domain adaptation for time-series forecasting aims to find the model that performs well on the dataset drawn from the target distribution by leveraging the sufficient labeled source domain data and the limited labeled target domain data. 

%\subsubsection{\textcolor{black}{Causal Conditional Shift Assumption for time-series forecasting problem}}
Existing methods on domain adaptive time-series forecasting straightforwardly follow the idea of the conventional domain adaption methods for non-time series data and typically assume that the conditional distribution $P(\bm{z}_{t+1}|\bm{z}_1,\cdots,\bm{z}_{t})$ remain fixed, which is shown as follows:
\begin{equation}
\label{equ:cov}
\begin{split}
    P^S(\bm{z}_{t+1}|\bm{z}_1,\cdots,\bm{z}_{t})&=P^T(\bm{z}_{t+1}|\bm{z}_1,\cdots,\bm{z}_{t})\\
    P^S(\bm{z}_1,\cdots,\bm{z}_{t}) &\neq P^T(\bm{z}_1,\cdots,\bm{z}_{t}).
\end{split}
\end{equation}
However, according to the aforementioned Granger-causal generation process for time-series data, \textcolor{black}{since $\bm{z}_{t+1}$ depends on $(\bm{z}_1,\cdots,\bm{z}_{t})$, the inappreciable value changes of $\bm{z}_1$ may magnify or minify the influence to $\bm{z}_{t+1}$ by the stable causal relationships, which makes the conditional distribution from different domains vulnerable to changes}. Therefore, the covariate shift assumption is very difficult to be satisfied in time-series data.

In order to bypass the aforementioned difficulties, one straightforward solution is to replace the covariate shift assumption with the ``Parent Shift Assumption'', which is shown in Equation (\ref{equ:pa}). This assumption can mitigate the accumulated changes from $\bm{z}_1$ and essentially hints that the future values only depend on their direct parents. 
\begin{equation}
\label{equ:pa}
\begin{split}
    P^S(\bm{z}_{t+1}|\bm{pa}(\bm{z}_{t+1}))&=P^T(\bm{z}_{t+1}|\bm{pa}(\bm{z}_{t+1}))\\
    P^S(\bm{pa}(\bm{z}_{t+1})) &\neq P^T(\bm{pa}(\bm{z}_{t+1})).
\end{split}
\end{equation}

% Though mitigating the butterfly effect to some extent
However, the non-negligible dependencies still exist in time-series data with the aforementioned assumption, few changes of the parents $\bm{pa}(\bm{z}_t)$ like the different time-lags and value ranges result in palpable variation of the future values, which further leads to the unsatisfactory of Equation (\ref{equ:pa}).

Fortunately, given any domains, the causal mechanism among variables is stable \textcolor{black}{and similar}, though the conditional distributions are easily influenced by the different value ranges, time lags, and offsets. \textcolor{black}{Recall the intuitive medicinal example mentioned in Section \ref{sec:introduction}, where the physiological mechanism stably works in everyone}. Therefore, with the inspiration of the domain-invariant causal mechanism between the source and the target domain, 
we propose the following \textbf{causal conditional shift assumption}.
\begin{assumption}
\label{ass:csa} \textbf{Causal Conditional Shift Assumption}: Given causal structures $A$, $\bm{z}_1, \cdots, \bm{z}_t$, and $\bm{z}_{t+1}$, we assume that the conditional distribution $P(\bm{z}_{t+1}|\bm{z}_1,\cdots,\bm{z}_{t})$ varies with different domains while the causal structures remain fixed. Formally, we have:
\begin{equation}
\label{equ:ccsa}
    \begin{split}
    % P^S(\bm{z}_{t+1}|\bm{z}_{1},\cdots,\bm{z}_{t})\neq&P^T(\bm{z}_{t+1}|\bm{z}_{1},\cdots,\bm{z}_{t}),\\
    P^S(\bm{z}_{t+1}|\bm{pa}(\bm{z}_{t+1}))\neq&P^T(\bm{z}_{t+1}|\bm{pa}(\bm{z}_{t+1})),\\
    % &P_S(\bm{z}_{1},\cdots,\bm{z}_{t},\bm{A}_S)\neq P_S(\bm{z}_{1},\cdots,\bm{z}_{t},\bm{A}_T),
    \bm{A}^S =& \bm{A}^T
\end{split}
\end{equation}
in which $\bm{A}^S$ and $\bm{A}^T$ are the causal structures from the source and the target domain. 
\end{assumption}

\textcolor{black}{Based on the aforementioned assumption, the difference between our causal conditional shift and other distribution shift learning scenarios can be summarized as follows. First, our causal conditional shift is totally different from the existing scenarios. Specifically, our causal conditional shift takes the causal structures among covariate variables i.e., $\bm{x}$ into account, while the existing scenarios mainly focus on the relationships among the label $\bm{y}$, covariate variables $\bm{x}$ and domain variables $D$ but ignoring the relationships behind $\bm{x}$. Second, our causal conditional shift is more reasonable than the existing scenarios on the time series. In detail, due to the temporal dependence among the time-series data, even a small perturbation may result in a huge change in the high dimensional feature space $\bm{x}$ and further results in $p(\bm{y}|\bm{x})$ changes sharply. Our causal conditional shift solves this problem by exploring the stable causal structure behind $\bm{x}$, but the existing scenarios (without considering the structure behind $\bm{x}$) may face the challenge that how to describe the distribution shifts among domains. Thus, compared with existing learning scenarios, causal conditional shifts can characterize complex distribution shifts in a compact way and enables the development of effective models with better generalization performance.}

Based on the aforementioned \textbf{causal conditional shift assumption}, we can find that the key to semi-supervised domain adaptation for time-series forecasting is to leverage the domain-invariant causal structures to model the domain-specific conditional distribution. In summary, given the labeled source domain data and the limited labeled target domain data, our goal is to (1) discover the causal structures of time-series data from the source and the target domains; (2) forecast the future time-series data over the target distribution with the help of Granger causality.

% For generalization, we assume that the difference of causal structures from different domains are small. Therefore, the key of semi-supervised domain adaptation from time-series forecasting is to simultaneously model $P_S(\bm{A}_S|\bm{z}_1,\cdots,\bm{z}_t)$ and $P_T(\bm{A}_T|\bm{z}_1,\cdots,\bm{z}_t)$. Therefore, given the labeled source domain data and the limited labeled target domain data, our goal is to (1) respectively extract the causal structures of time-series data from the source and the target domain; (2) forecast the future time-series data from target distribution with the help of time-series causality inference. 
However, the causal structures are usually changed by other domain-specific properties along with different domains\textcolor{black}{, e.g. different time-lags}, which further leads to the \textcolor{black}{violation} of the \textbf{causal conditional shift assumption}. \textcolor{black}{Recall the medicinal example again}, the responding time between ``Blood Glucose'' and ``Insulin'' in the elder is usually longer than that of the young, which leads to the different causal structures. In order to address this issue, \textcolor{black}{we not only need to extract the domain-invariant causal substructures among domains but also model the domain-specific components.}
% we aim to extract the common causal structures between the source and the target domains.

Therefore, the objective function of the semi-supervised domain adaptation for time-series forecasting can be summarized as follow:
\begin{equation}
\label{equ:loss1}
\begin{split}
    \min&\;\mathbb{E}_{\bm{A}^S}P\left(\bm{y}^S|\bm{x}^S,\bm{A}^S\right) +\mathbb{E}_{\bm{A}^T}P\left(\bm{y}^T|\bm{x}^T,\bm{A}^T\right) \\& + Disc(\bm{A}^S,\bm{A}^T),
\end{split}
\end{equation}
where $Disc(\bm{A}^S,\bm{A}^T)$ is the \textcolor{black}{discrepancy} regularization term of the Granger-causal structures between the source and the target domains. \textcolor{black}{This regularization term can not only bridge the common causality from the source to the target but also allow the model to preserve the domain-specific parts.} 
Note that $\mathbb{E}_{\bm{A}^S}P\left(\bm{y}^S|\bm{x}^S,\bm{A}^S)\right)$ and $\mathbb{E}_{\bm{A}^T}P\left(\bm{y}^T|\bm{x}^T,\bm{A}^T)\right)$ can be considered as an autoregression of time-series data for different domains.

\textcolor{black}{According to Equation (\ref{equ:loss1}), to address the semi-supervised domain adaptation problem for time-series data we should address the following two challenges: 1) how to model time-series data with  causal structures, i.e., optimizing the expectation terms in Equation (\ref{equ:loss1}); 2) how to transfer the causal structures from the source to target domain in a unified framework, i.e., optimizing $Disc(\bm{A}^S,\bm{A}^T)$. As for the first challenge, we take the Granger Causal structures as latent variables and learn the distribution of time-series data via variational inference, which is illustrated in Section \ref{sec:vae}. As for the second challenge, we devise Granger Causality Alignment (GCA) model, which is illustrated in Section \ref{sec:imple}.}

\begin{figure}[t]
	\centering
	\includegraphics[width=0.9\columnwidth]{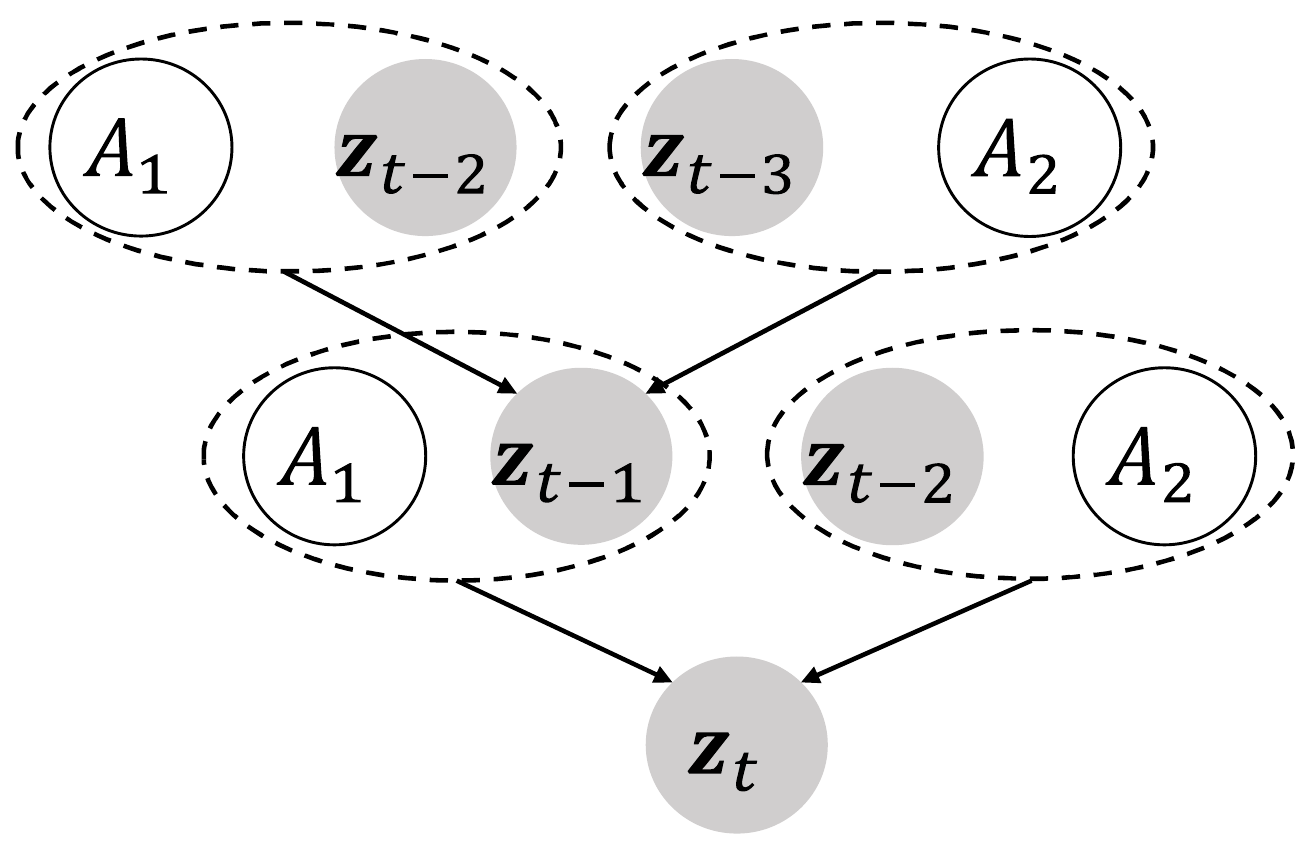}
	\caption{The illustration of how time-series data is iteratively generated via Granger-causal structures with a maximum lag of two. (The parent of $\bm{z}_{t-2}$ are omissive due to the limited space.) $A_1, A_2$ are the substructures of the Granger-causal structure with the lag of one and two. The Granger-causal structures are stable in all the timestamps. $\bm{z}$ is the observed data of different timestamps.}
	\label{fig:generation}
\end{figure}

% 基于以上的假设,为了解决时间序列的领域自适应问题,我们的目标有两个,1. 是分别学习source和target的因果结构,2.使用学习到的因果结构进行推理
% 1.5 学习因果结构,提一下使用格兰杰因果结构
% 2. 学习因果结构,第一步介绍基于因果结构的概率图模型,ELBO
% 3. 学习因果结构,第二步介绍如何encode因果图(插入因果图细微不同的处理方法:target domain 数据量比较少所以需要对齐,)
% 4. 学习因果结构,第三步介绍利用因果图来推理(插入因果强度随着domain变化的处理方法)
% 5. 总结
% In order to address the aforementioned task, the following key obstacles need to be cleared away. (1) How to inference the full time graphs of Granger Causality whose summary graph is stable across domains? (2) How to model the domain-specific conditional distributions among timestamps? 
% \subsection{\textcolor{black}{Modeling Granger-Causality-based Time-Series Data via Variational Inference}}
\section{{Learning the Distribution of Time-Series Data with Granger Causality}}\label{sec:vae}

% According to the aforementioned goal, we aim to simultaneously discover the Granger-causal structures from the different domains and predict the future values with the help of Granger Causality. Therefore, we intend to address the semi-supervised domain adaptation problem for time-series forecasting by \textcolor{black}{discovering the Granger Causality and predicting based on the Granger Causality.}

% Regarding how to discover the Granger Causality from different domains with labeled source domain data and limited labeled target domain data,
% In order to extract Granger causality from the observed time-series data, 
{To Simultaneously discover the Granger Causality and predict based on the Granger Causality,} we start with the graphical model of the causal generation process for time series data\textcolor{black}{, which is shown in Figure \ref{fig:generation}}. 
% As shown in Figure \ref{fig:generation}, we let $A_1, \cdots, A_k$ be the local substructures of Granger causality with the lag of $1, \cdots, k$. 
% First, we let $\bm{A}=(A_1, \cdots, A_j, \cdots, A_k)$ be the causal structures with the maximum lag of $k$ (k=2 in Figure \ref{fig:generation}), in which the italic symbol $A_j$ denotes the local substructures with the lag of $j$. 
\textcolor{black}{First, we assume the causal structure $\bm{A}$ is composed of $k$ substructures with the maximum lag of $k$, which is shown as follows:}
\begin{equation}
    \bm{A}=(A_1, \cdots, A_j, \cdots, A_k),
\end{equation}
\textcolor{black}{in which the italic symbol $A_j$ denotes the local substructures with the lag of $j$. For convenient illustration, we let $k=2$ in Figure \ref{fig:generation}.}
% Notably, it is reasonable to split the Granger Causality accroding to different lags since each module of causal mechanism is independent of each other. 
Given the observed data $\bm{z}_t$, it is generated from the $\bm{z}_{t-1}$ and $\bm{z}_{t-2}$ under the substructure of Granger-causal graph whose lag is one and two (We assume that the maximum lag is two for convenient illustration.). Furthermore, $\bm{z}_{t-1}$ and $\bm{z}_{t-2}$ are recursively generated from the previously observed data under the same substructures of Granger Causality. 

% Based on the aforementioned graphical model for time-series generation process, we consider the substructures of Granger-causal graph $A_1, A_2$ as the discrete latent variables and reconstruct the Granger-Causality by developing the technique of stochastic variational inference procedure. 
\textcolor{black}{In light of the power of variational autoencoder (VAE) \cite{kingma2013auto,jang2016categorical} in reconstructing the latent variables, we consider the Granger-causal structures as the latent variables and iteratively reconstruct the Granger-causal structures with different lags as shown in Figure \ref{fig:rgc-vae}.}
\textcolor{black}{According to this causal generation process (the solid lines in Figure 3), since $z_t$ is the common effect of $A_1$, $z_{t-1},z_{t-2}$ and $A_2$, there are three ''V-Structures`` \footnote{\textcolor{black}{``V-structure'' \cite{pearl1998graphs} is an ordered triple of variables ($a,b,c$) such that (1) contains the arcs $a\rightarrow b$ and $b\leftarrow c$, and (2) $a$ and $c$ are not adjacent.}} (i.e., \textcolor{black}{$A_1 \rightarrow z_t \leftarrow A_2$}, $A_1\rightarrow z_t \leftarrow (z_{t-1},z_{t-2})$ and $A_2\rightarrow z_t \leftarrow (z_{t-1},z_{t-2})$), $A_1, A_2$ are independent of each other. As for the inference process (the dashed lines in Figure 3, the reconstruction of $A_1$ and $A_2$ can be separated into two steps. First, in the ``V-structure'' $A_1\rightarrow z_t \leftarrow (z_{t-1}, z_{t-2})$, we can reconstruct $A_1$ given $z_{t-1},z_{t-2}$ and $z_t$. Second, we can reconstruct $A_2$ given $A_1, \bm{z}_t$ and $(\bm{z}_{t-k}, \cdots,\bm{z}_{t-1})$} 

% \textcolor{black}{According to the causal generation process, since each component in causal structures is independent, we can find that $A_1, \cdots, A_k$ are independent of each other. As for the inference process in Figure \ref{fig:rgc-vae}, the reconstruction of $A_1$ and $A_2$ can be finished in two steps. In step (1), $\bm{z}_t, A_1$ and $(\bm{z}_{t-k}, \cdots,\bm{z}_{t-1})$ consist of a \textcolor{black}{``V-structure''\footnote{\textcolor{black}{``V-structure'' \cite{pearl1998graphs} denotes a special directed acyclic graph (DAG), where two independent variables are the causes of the same effect variable.}}} \textcolor{black}{(i.e., $A_1\rightarrow z_t \leftarrow (z_{t-1}, z_{t-2})$)}, so we can reconstruct $A_1$ given $\bm{z}_t$ and $(\bm{z}_{t-k}, \cdots,\bm{z}_{t-1})$. We also find that $A_1, A_2$ and $(\bm{z}_{t-k}, \cdots,\bm{z}_{t-1})$ are not independent give $\bm{z}_t$, so we can infer $A_2$ given $A_1, \bm{z}_t$ and $(\bm{z}_{t-k}, \cdots,\bm{z}_{t-1})$ in the step (2). }

% Based on the data generation process shown in Figure \ref{fig:generation}, we can model the time-series data and derive the logarithm of the joint likelihood $lnP(\bm{z}_t|\bm{z}_{t-1}, \bm{z}_{t-2}, \cdots, \bm{z}_{t-k})$, which is shown as follows.
% Therefore, we can model the time-series data by first reconstructing the latent causal structures and then using them to forecast future values. Mathematically, modeling 
\textcolor{black}{Therefore, we can learn the distributions of time-series data by first reconstructing the latent causal structures and using them to forecast future values. Mathematically, we learn the distributions of the time-series data by modeling the logarithm of the probability density $lnP(\bm{z}_t|\bm{z}_{t-1}, \bm{z}_{t-2}, \cdots, \bm{z}_{t-k})$, which can be derived as shown in Theorem \ref{the:derive}.}

\begin{figure}[t]
	\centering
	\includegraphics[width=1.0\columnwidth]{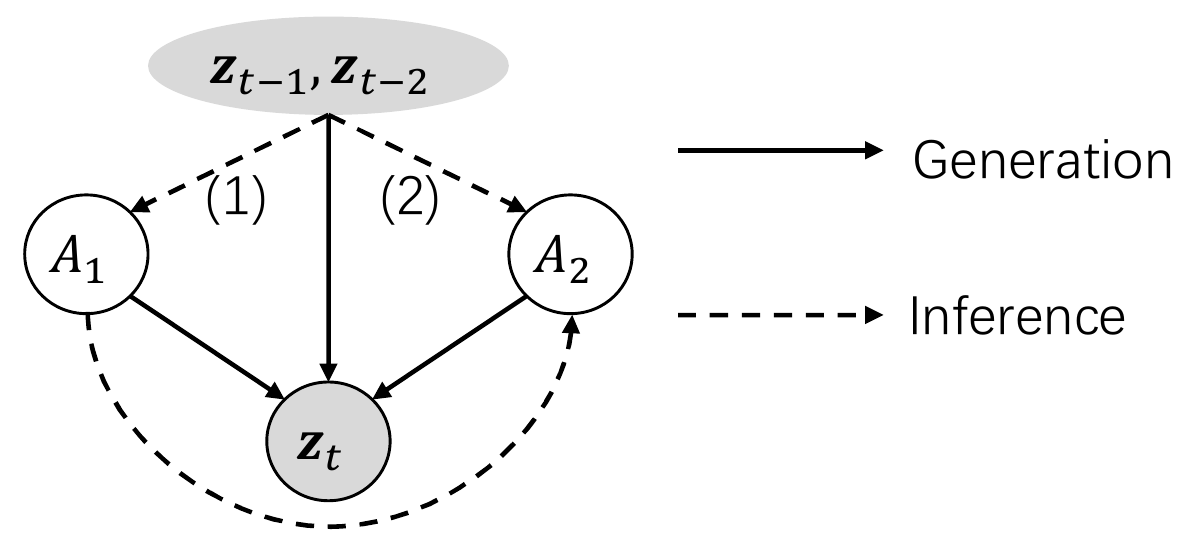}
	\caption{{
	The solid lines denote the generation process and the dashed lines denote the inference process. The inference process iteratively reconstructs the Granger-Causal structures. In this toy example, the inference process consists of two steps. Step (1) is used to infer $A_1$ and step (2) is used to infer $A_2$ based on $A_1$. The generative process forecasts future data with the help of historical data and reconstructed causal structures.
% 	(a) is the inference model that iteratively reconstructs the Granger-Causal structures, in which the left side of (a) is used to infer $A_1$ and the right side of 
% 	(a) is used to infer $A_2$, respectively. Note that $A_1$ and $A_2$ is actually independent and we infer $A_2$ conditioned on $A_1$ just in the implementation of inference model.  (b) is the generative model that forecasts the future data with the observed data and the causal structures.
	}}
	\label{fig:rgc-vae}
\end{figure}

\begin{theorem}
\label{the:derive}
    \textcolor{black}{Suppose that the maximum lag is $k$, and then $\ln P(z_t|z_{t-1}, z_{t-2}, \cdots, z_{t-k})$ can be derived as follows:}
\begin{equation}
\small
\color{black}
\label{equ:variation}
\begin{split}
    lnP(&\bm{z}_t|\bm{z}_{t-1}, \bm{z}_{t-2}, \cdots, \bm{z}_{t-k})= \\& \mathbb{E}_{Q(A_1|\cdot)},\cdots\mathbb{E}_{Q(A_k|\cdot)}\left[\ln P(\bm{z}_t|\bm{z}_{t-1},\cdots,\bm{z}_{t-k},A_1,\cdots, A_k)\right. \\&+ \left.\ln \frac{P(A_1|\cdot)\prod_{j=2}^k P(A_j|\cdot)}{Q(A_1|\cdot)\prod_{j=2}^k Q(A_j|\cdot)}\right] + \sum_{j=1}^k{D_{KL}(Q(A_j|\cdot)||P(A_j|\cdot)}),
\end{split}
\end{equation}
\textcolor{black}{in which $P(A_j|\cdot)=P(A_j|A_1,\cdots,A_{j-1},\bm{z}_{t-1},\cdots,\bm{z}_{t-k})$ and $Q(A_j|\cdot) = Q(A_j|A_1,\cdots,A_{j-1},\bm{z}_{t-1},\cdots,\bm{z}_{t-k})$ denote the prior distribution and the approximated distribution, respectively.}
\end{theorem}

\textcolor{black}{The proof of Theorem \ref{the:derive} can be found in the Appendix A. This result tells us that we can model the logarithm of the probability density by modeling an expectation term and a Kullback–Leibler divergence term. However, optimizing the Kullback–Leibler divergence $\sum_{j=1}^k{D_{KL}(Q(A_j|\cdot)||P(A_j|\cdot)})$ is intractable. Fortunately, since the values of the Kullback–Leibler divergence are always greater than zero, we can model the logarithm of the probability density $\ln P(z_t|z_{t-1}, z_{t-2}, \cdots, z_{t-k})$ by maximizing the evidence lower bound (\textit{ELBO}), which is shown as Corollary \ref{cor1}.}
% can be Assuming the maximum lag is $k$, the logarithm of joint likelihood $lnP(\bm{z}_t|\bm{z}_{t-1}, \bm{z}_{t-2}, \cdots, \bm{z}_{t-k})$ can be written as follows:
% \begin{equation}
% \small
% \label{equ:variation}
%     \begin{split}
%         lnP(\bm{z}_t&|\bm{z}_{t-1}, \bm{z}_{t-2}, \cdots, \bm{z}_{t-k}) = \mathcal{L}_{ELBO} +  \sum_{j=1}^k{D_{KL}(Q_j||P_j})\\
%         &Q_j = Q(A_j|A_1,\cdots,A_{j-1},\bm{z}_{t-1},\cdots,\bm{z}_{t-k})\\
%         &P_j=P(A_j|A_1,\cdots,A_{j-1},\bm{z}_{t-1},\cdots,\bm{z}_{t-k}),
%     \end{split}
% \end{equation}
% in which the second term is the summation of the Kullback–Leibler divergence between the approximate distribution and the true posterior. 
% \textcolor{black}{ In practice, inferring $A_j$ given $(A_1, \cdots, A_{j-1})$ benefits to obtain the sparse and precise causal structures. If we estimate the causal structures in the form like $p(A_j|\bm{z}_{t-1},\cdots,\bm{z}_{t-k})$, some overlapped and redundant edges might be brought in.}

% \textcolor{black}{Based on Theorem \ref{the:derive}, we can model the logarithm of the joint likelihood $\ln P(z_t|z_{t-1}, z_{t-2}, \cdots, z_{t-k})$ by maximizing the evidence lower bound (\textit{ELBO}), which is shown as follows (More details are shown in Appendix A.).}
\begin{corollary}
\label{cor1}
    \textcolor{black}{Assuming that the maximum lag is $k$, we can model the logarithm of the probability density $\ln P(z_t|z_{t-1}, z_{t-2}, \cdots, z_{t-k})$ by maximizing the evidence lower bound (\textit{ELBO}) as shown as follows:}
\begin{equation}
\label{elbo}
\color{black}
\begin{split}
\mathcal{L}_{\textit{ELBO}}=&\mathbb{E}_{Q(A_1|\cdot)},\cdots\mathbb{E}_{Q(A_k|\cdot)}\ln P(\bm{z}_t|\bm{z}_{t-1}, \cdots, \bm{z}_{t-k},A_1, \cdots, A_k) \\& - D_{KL}(Q(A_1|\cdot)||P(A_1|\cdot)) \\&-\sum_{j=2}^k \mathbb{E}_{Q(A_1|\cdot)},\cdots\mathbb{E}_{Q(A_k|\cdot)} D_{KL}(Q(A_j|\cdot)||P(A_j|\cdot)),
\end{split}
\end{equation}
\end{corollary}
\textcolor{black}{The proof of Corollary \ref{cor1} can be found in the Appendix A. According to this corollary, we can use Granger Causal structures to learn the distribution of time-series data under the framework of variation inference. For easy comprehensibility, we let $\mathcal{L}_{\textit{ELBO}}^S$ and $\mathcal{L}_{\textit{ELBO}}^T$ be the ELBO for the source and target data, respectively.}

% \textcolor{black}{in which $Q_j = Q(A_j|A_1,\cdots,A_{j-1},\bm{z}_{t-1},\cdots,\bm{z}_{t-k})$ is the approximated distributions of $A_j$. }
\section{\textcolor{black}{Algorithm and Implement}}\label{sec:imple}
\subsection{\textcolor{black}{Overview of the Granger Causality Alignment Model}}

\begin{figure*}[htbp]
\label{fig:model}
	\centering
	\includegraphics[width=2.0\columnwidth]{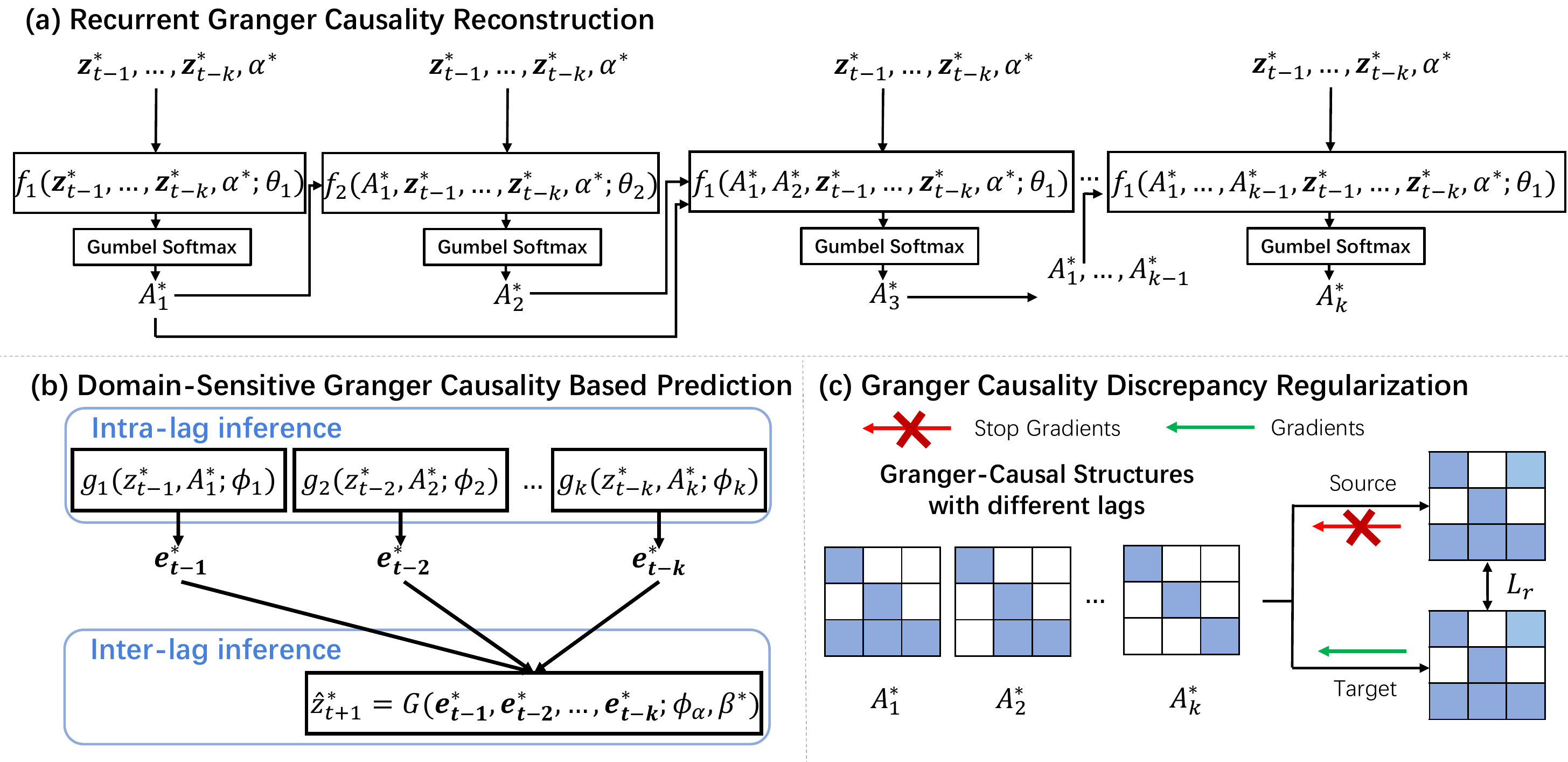}
	\caption{The framework of the proposed Granger Causality Alignment (GCA) model. (a) The Recurrent Granger Causality reconstruction module. $\alpha$ denotes the structured domain latent variables that are used to reconstruct the domain-specific Granger-Causal substructure. (b) The Domain-Sensitive Granger Causality Inference module aims to model how the conditional distribution changes across domains with the help of the learned Granger Causality. $\beta$ is the trainable domain-sensitive latent variables. (c) \textcolor{black}{The Granger Causality Discrepancy Regularization between the source and the target domain is used to bridge the domain-invariant Granger Causality from the source to the target domain.}}
	\label{fig:model}
\end{figure*}

\textcolor{black}{Based on the aforementioned theories, we aim to devise a model to address the semi-supervised time-series domain adaptation problem.}
\textcolor{black}{Since we can simultaneously estimate the causal structures from different domains and model $P(z_t|z_{t-1}, \cdots, z_{t-k}, A_1, \cdot, A_k)$ by maximizing $\mathcal{L}_{ELBO}^S$ and $\mathcal{L}_{ELBO}^T$, we can replace the negative of evidence lower bounds (i.e., $\mathcal{L}^S_{ELBO}$ and $\mathcal{L}^T_{ELBO}$) with $\mathbb{E}_{\bm{A}^S}P\left(\bm{y}^S|\bm{x}^S,\bm{A}^S\right) $ and $\mathbb{E}_{\bm{A}^T}P\left(\bm{y}^T|\bm{x}^T,\bm{A}^T\right)$ in Equation (\ref{equ:loss1}), which can be rewritten as follows:}
\begin{equation}
\color{black}
    \mathcal{L}_{total} = -\mathcal{L}_{ELBO}^S - \mathcal{L}_{ELBO}^T + Disc(\bm{A}^S, \bm{A}^T).
    \label{equ:9}
\end{equation}

\textcolor{black}{By optimizing Equation (\ref{equ:9}), we can reconstruct the sufficient structures that benefit the time-series forecasting. However, simply optimizing Equation (\ref{equ:9}) might bring some redundant structures which do not belong to the ground-truth data generation process and are usually influenced by different domains. If these redundant structures are taken into account, the learned model might not well generalize to the target domain. To solve this problem, we need to prefer the sparsity of the Granger Causal structure to remove redundant structures and preserve the necessary causal structures.}

\textcolor{black}{To achieve this, \textcolor{black}{we further include a sparsity-inducing penalty term $\mathcal{R}(\cdot)$ and hence} obtain the loss function based on the causal time-series generation process, which is shown as follows:}
\begin{equation}
\label{equ:loss2}
\color{black}
\small
\begin{split}
    \!\!\mathcal{L}_{total}\!&=\!-\mathcal{L}_{ELBO}^S \!-\! \mathcal{L}_{ELBO}^T\!+\! \gamma Disc(\!\bm{A}^S,\! \bm{A}^T\!) \!+\! \lambda(\mathcal{R}(\bm{A}^S)\!+\!\mathcal{R}(\bm{A}^T)),
    % \\+& \gamma\sum_{j=1}^{k}\mathbb{E}_{Q_j^S}\mathbb{E}_{Q_j^T}Disc\left(A_j^S,A_j^T\right)\\+& \lambda\left(\sum_{j=1}^{k}\mathbb{E}_{A^S_j \sim Q_j^S}\mathcal{R}(A^S_j) + \sum_{j=1}^{k}\mathbb{E}_{A^T_j \sim Q_j^T}\mathcal{R}(A^T_j)\right),
\end{split}
\end{equation}
\textcolor{black}{in which $\mathcal{L}_{\text{ELBO}^S}$ and $\mathcal{L}_{\text{ELBO}^T}$ equal to evidence lower bound in Equation (\ref{elbo}), and we employ the superscript $S$ and $T$ to distinguish if the data used to optimize ELBO come from the source or the target domain.}
% \begin{equation}
% \small
% \color{black}
% \begin{split}
% \\\!\!\mathcal{L}_{\text{ELBO}}^*\!&=\!\mathbb{E}_{\tiny{Q(A_1^*|\cdot)}},\!\cdots\mathbb{E}_{\tiny{Q(A_k^*|\cdot)}}\!\!\left[\ln\! \tiny{P\!\left(\bm{z}_t^*|\bm{z}^*_{t-1},\!\cdots\!,\bm{z}^*_{t-k},A_1^*,\!\cdots\!,A_k^*\right)}\right] \\&- D_{KL}(Q(A_1^*|\cdot)||P(A_1^*|\cdot)) \\&- \sum_{j=2}^k \mathbb{E}_{Q(A_1^*|\cdot)},\cdots\mathbb{E}_{Q(A_k^*|\cdot)} D_{KL}(Q(A_j^*|\cdot)||P(A_j^*|\cdot)),
% % ,\\\!\!\mathcal{L}_{\text{ELBO}}^T\!&=\!\mathbb{E}_{\tiny{Q(A_1^T|\cdot)}},\!\cdots\mathbb{E}_{\tiny{Q(A_k^T|\cdot)}}\!\!\left[\ln\! \tiny{P\!\left(\bm{z}_t^T|\bm{z}^T_{t-1},\!\cdots\!,\bm{z}^T_{t-k},A_1^T,\!\cdots\!,A_k^T\right)}\right] \\&- D_{KL}(Q(A_1^T|\cdot)||P(A_1^T|\cdot)) \\&- \sum_{j=2}^k \mathbb{E}_{Q(A_1^T|\cdot)},\cdots\mathbb{E}_{Q(A_k^T|\cdot)} D_{KL}(Q(A_j^T|\cdot)||P(A_j^T|\cdot)) .
% \end{split}
% \end{equation}
\textcolor{black}{
% where we use the superscript $*$ to denote source or target domains. 
Note that $\gamma$ is a hyper-parameter for the causal structure discrepancy regularization and $\lambda$  is a hyper-parameter for the Granger causal structure sparsity regularization, respectively}; 
% $\mathcal{L}_{ELBO}^S$ and $\mathcal{L}_{ELBO}^T$ respectively correspond to $\mathbb{E}_{\bm{A}^S}P\left(\bm{y}^S|\bm{x}^S,\bm{A}^S\right)$ and $\mathbb{E}_{\bm{A}^T}P\left(\bm{y}^T|\bm{x}^T,\bm{A}^T\right)$ in Equation (\ref{equ:loss1}); 
\textcolor{black}{$Q(A_j^*|\cdot)$ 
\footnote{\textcolor{black}{we use superscript $*$ to denote source or target domains.}}
denotes the approximated distribution and $P(A_j^*|\cdot)$ denotes the prior distribution; the conditional distribution $P\left(\bm{z}_t^*|\bm{z}^*_{t-1},\cdots,\bm{z}^*_{t-k},A_1^*,\cdots,A_k^*\right)$ work as the generative process (as shown in Figure \ref{fig:rgc-vae}) and are used to predict the future values with the help of Granger causality; }\textcolor{black}{We also employ $\mathcal{R}(\cdot)$ as the sparsity-inducing penalty term, i.e., the elastic-net-style penalty \cite{nicholson2017varx} for $\mathcal{R}(\cdot)$ for better Granger causal structures, which can be defined as:}
\begin{equation}
\color{black}
\begin{split}
    \mathcal{R}(\bm{A}^*)=\frac{1}{2}||\bm{A}^*||_1 + \frac{1}{2}||\bm{A}^*||_2,
\end{split}
\end{equation}
\textcolor{black}{in which $||\cdot||_1$ and $||\cdot||_2$ respectively denote the L1 and L2 Norm. Note that \textcolor{black}{L1 Norm of a matrix is calculated by $||A||_1=\sum\limits_{uv}|A_{uv}|$ , where $A_{uv}$ is the $u$-th row $v$-th column element in matrix $A$. } And L2 Norm of a matrix is calculated by $\displaystyle  ||A^*||_2=\sqrt{ \sum\limits_{uv}(A_{uv}^*)^2}$, where $A_{uv}^* $ denotes the $u$-th row $v$-column element in matrix $A^*$.}

\textcolor{black}{Based on the objective function as shown in Equation (\ref{equ:loss2}), we devise the Granger Causality Alignment (GCA) model for semi-supervised domain adaptation for time-series data. The framework of the proposed GCA model is shown in Figure \ref{fig:model}. }
% The proposed method is built on the causal generation process for the time-series data and we further devise it by extending the variational autoencoder for categorical latent variables \cite{kingma2013auto,jang2016categorical} to the Recurrent Granger Causality Variational Autoencoder (RGC-VAE) 
\textcolor{black}{As shown in Figure \ref{fig:model} (a), the \textbf{Recurrent Granger Causality Reconstruction Module} (Subsection \ref{sec:encoder}) is used to reconstruct the Granger-causal structures with different lags with the observed time-series data from the source and the target domains, which is used to estimate the distribution of latent variables $Q(A_j^*|\cdot)$. Sequentially, the \textbf{Domain-Sensitive Granger Causality based Prediction Module} (Subsection \ref{sec:decoder}) as shown in Figure \ref{fig:model} (b) learns how the condition distribution $P(z_t^*|z^*_{t-1}, \cdots, z^*_{t-k}, A_1^*, \cdots,A_k^*)$ changes across different domains with the help of Granger-causal structures and the historical observed data. Since the causal structures are stable across domains, we introduce the implementation of the \textbf{Granger Causality discrepancy regularization} (Subsection \ref{sec:reg}) as shown in Figure \ref{fig:model} (c) to restrict the discrepancy of Granger Causality between the source and the target domain. Finally, we introduce the training and inference process of the proposed method (subsection \ref{sec:summary}). The implementation details will be introduced in the following subsections. The source code of the proposed GCA method can be found at \href{https://github.com/DMIRLAB-Group/GCA}{\textcolor{black}{https://github.com/DMIRLAB-Group/GCA}}
% \footnote{\textcolor{black}}
.}

% The framework of the proposed disentangled Granger Causality method is shown in Figure. \ref{fig:model}. As shown in the upper part of Figure. \ref{fig:model}, we extend the variational autoencoder for categorical latent variables \cite{kingma2013auto,jang2016categorical} to the Recurrent Granger Causality Variational Autoencoder (RGC-VAE), which is used to recurrently reconstruct the Granger-causal structures of different lags. Since the Granger-causal structures vary with domain, we use the structure domain latent variables $\alpha$ to learning the domain-specific module of causal structures. We further devise a ``Domain-sensitive Granger Causality Inference'' architecture which combines the Granger-causal structures, the previous observed data and the strength latent variables $\beta$ to model how the condition distribution changes across different domains. Finally, we employ the Granger Causality regularization term to align the causal structures from different domains.

\subsection{Recurrent Granger Causality Reconstruction Module} \label{sec:encoder}

% 这个模块是用来做什么的，有什么优势
% 怎么做
% \textcolor{black}{ In practice, inferring $A_j$ given $(A_1, \cdots, A_{j-1})$ benefits to obtain the sparse and precise causal structures. If we estimate the causal structures in the form like $p(A_j|\bm{z}_{t-1},\cdots,\bm{z}_{t-k})$, some overlapped and redundant edges might be brought in.}

\textcolor{black}{In this subsection, we provide the details of the Recurrent Granger Causality Reconstruction, which is illustrated in Figure \ref{fig:model} (a). The Recurrent Granger Causality Reconstruction is used to model the approximated conditional distribution $Q(A_j^*|A_1^*, \cdots, A_{j-1}^*, z_{t-1}^*, \cdots, z_{t-k}^*)$ in Equation (\ref{elbo}), from which we can sample the causal structures. Note that Granger causal structures with different lags are independent of each other. But we use $(A_1^*, \cdots, A_{j-1}^*)$ as the conditions when estimating $A_j^*$, which benefits obtaining the sparse and precise causal structures. If we estimate the causal structures without other structures like $P(A_j^*|\bm{z}_{t-1}^*,\cdots,\bm{z}_{t-k})^*$, some overlapped and redundant edges might be brought in. }

% \textcolor{black}{In this subsection, we provide the details of the Recurrent Granger Causality Reconstruction, whose illustration is shown in Figure \ref{fig:model} (a)}. As mentioned in Equation (\ref{elbo}), we aim to discover the Granger-causal structures from the source and the target domains by separately modeling $P^S\left(\bm{z}_t|\bm{z}_{t-1},\cdots,\bm{z}_{t-k},A_1^S,\cdots,A_k^S\right)$ and $P^T\left(\bm{z}_t|\bm{z}_{t-1},\cdots,\bm{z}_{t-k},A_1^T,\cdots,A_k^T\right)$. \textcolor{black}{Please note that the reconstructed causal structures \textcolor{black}{denote the adjacent matrix and} do not contain the strengths information. We let the strength parameters be learned \textcolor{black}{in the domain-sensitive Granger Causality based Prediction module}. This is because the strength parameters are usually flexible and may vary with different while the causal structures are stable. If we estimate and align the causal structures with strength parameters, the strength parameters will tend to become similar, which will further lead to the suboptimal results.}

\textcolor{black}{Technically, for $Q(A_j^*|A_1^*,\cdots, A_{j-1}^*,\bm{z}^*_1,\cdots,\bm{z}^*_{t-k})$, we implement it with the MLPs and the categorical reparameterization trick \cite{jang2016categorical}, where we let $A_j^*$ be the substructure with the lag of $j$ from the source or the target domains. In detail, we use Equation (\ref{encoder_imp_1}) as the universal approximator of $Q(A_j^*|A_1^*,\cdots,A_{j-1}^*,\bm{z}_1^*,\cdots,\bm{z}_{t-k}^*)$ to generate $A_j^*$ from the historical data. Consequently, we implement $Q(A_j^*|A_1^*,\cdots,A_{j-1}^*,\bm{z}_1^*,\cdots,\bm{z}_{t-k})^*$ in the following functional form:}
\begin{equation}\label{encoder_imp_1}
\color{black}
\begin{split}
     \hat{A_j}^*=f_j(A_1^*,&\cdots ,A^*_{j-1},\bm{z}^*_{t-1},\cdots,\bm{z}^*_{t-k};\theta_j),\\
     A_j^* &= Softmax((\hat{A^*_j}+\epsilon_g)/\tau),
\end{split}
\end{equation}
\textcolor{black}{in which $\hat{A}_j^*$ denotes the parameters of the categorical distribution; $\theta_j$ are the trainable parameters of MLP $f_j$; $\epsilon_g$ denotes the Gumbel distributed noise; and $\tau$ denotes the temperature hyperparameter.} 
\textcolor{black}{According to Equation (\ref{encoder_imp_1}), it is interesting to find that the process of reconstructing the Granger-causal substructures with different lags is in a recurrent form. Therefore, we follow the iterative property and first reconstruct  $A_1$, the Granger-causal structures whose lag is $1$ then we reconstruct $A_2$ of the Granger-causal structures whose lag is 2 since $Q_2=Q(A_2|A_1,\bm{z}_{t-1},\cdots,\bm{z}_{t-k})$ is based on $A_1$, and so on so forth. By parity of reasoning, the other substructures with larger lags follow the same rule.}

\textcolor{black}{Besides, since the causal structures may vary with different domains due to domain-specific factors like time lags, it is necessary to explicitly consider these factors.} For example, the responding time between ``Blood Glucose'' and ``Insulin'' of the elder is longer than that of the young, which leads to the different time lags of the causal relationship between ``Blood Glucose'' and ``Insulin''. 
One straightforward solution is to make use of two separate transformation functions for the source and the target domains. However, since the labeled target domain data is limited, it tends to result in suboptimal performance on the target domain. In order to address this problem, we further introduce the trainable structural domain latent variables $\alpha^S, \alpha^T \in \mathbb{R}^{d_{\alpha}}$ to separately assist to model the domain-specific module of the Granger Causality, \textcolor{black}{where $d_{\alpha}$ is the dimension of the trainable structural domain latent variables and the dimension of $\alpha^S$ and $\alpha^T$ are same. }Therefore, we extend Equation (\ref{encoder_imp_1}) to Equation (\ref{encoder_imp_2}) as follows:
\begin{equation}\label{encoder_imp_2} 
% \small
% \begin{split}
%     A_j^S&=f_j(A_1^S,\cdots,A_{j-1}^S,\bm{z}_{t-1},\cdots,\bm{z}_{t-k},\alpha^S;\theta_j)\\
%     A_j^T&=f_j(A_1^T,\cdots,A_{j-1}^T,\bm{z}_{t-1},\cdots,\bm{z}_{t-k},\alpha^T;\theta_j),
% \end{split}
\color{black}
\begin{split}
     \hat{A_j}^*=f_j(A_1^*,&\cdots ,A^*_{j-1},\bm{z}^*_{t-1},\cdots,\bm{z}^*_{t-k},\alpha^*;\theta_j),\\
     A_j^* &= Softmax((\hat{A^*_j}+\epsilon_g)/\tau),
\end{split}
\end{equation}
\textcolor{black}{in which $f_j$ is shared across different domains and $\alpha^*$ denotes $\alpha^S$ or $\alpha^T$}. For convenience, We further let $\Theta=(\theta_1, \cdots, \theta_j, \alpha^*)$ be all the trainable parameters of the recurrent Granger Causality reconstruction module.

\subsection{Domain-sensitive Granger Causality based Prediction Module}\label{sec:decoder}
\textcolor{black}{In this subsection, we provide the details of the Domain-sensitive Granger-Causality-based Prediction Module, which is used to model $P(z_t^*|z_{t-1}^*, \cdots, z_{t-k}^*, A_1^*,\cdots, A_k^*)$. As shown in Figure \ref{fig:model} (b), the Granger-Causality-based prediction process can be separated into two steps: the intra-lag inference step and the inter-lag inference step.}
% In order to predict the future time-series data with the help of Granger Causality, we further devise the domain-sensitive Granger Causality based prediction model, which is shown in Figure \ref{fig:model} (b).
% The Granger Causality based prediction process can be separated into two steps: the intra-lag inference step and the inter-lag inference step. 

\textcolor{black}{In the intra-lag inference step, we aim to calculate the effectiveness of the substructures of each time lags, \textcolor{black}{i.e., the contributions of each substructure $A_j^*$ to the final prediction}. In detail, given the $j$-th lag Granger-causal structure $A_j^*$, we first use $A_j^*$ to mask the observed input $\bm{z}_{t-j}^*$, e.g., $\bm{z}_{t-j}^*\odot A_j^*$. \textcolor{black}{Then we use the multilayer perceptron (MLP) in the form of  $g_j(\bm{z}_{t-j}^*, A_j^*;\phi_j)$ to calculate $e_{t-j}^*$ for the source and target domain respectively, note that $e_{t-j}^*$ denotes how $z_{t-j}^*$ contributes to the prediction of $z^*_{t+1}$ by the granger causal structure $A^*_j$ via function $g_j(\cdot)$, i.e., $e_j^*=g_j(z_{t-j}^*,A_j^*;\phi_j)$.}}

% the effectiveness of $A_j^*$. And $\phi_j$ are the trainable parameters. 

In the inter-lag step, we aim to aggregate the effectiveness of all the time lags and predict future data. In detail, we simply use another MLP-based architecture in the form of $G(e_{t-1}^*, e_{t-2}^*, \cdots, e_{t-k}^*;\phi_a)$ to predict the future value and $\phi_a$ are the trainable parameters. 
% Note that we let $\Phi=(\phi_a, \phi_1,\cdots,\phi_k)$ as the trainable parameters of the domain-sensitive Granger Causality inference.

Besides the domain-specific time lags, the strengths of the causal structures might be different even if the source and the target domains share the same causal structures. 
% For example, the  effectiveness of “Blood Glucose” and “Insulin” between the male and the female may be different even the age of these two population are similar. 
In order to cope with these differences, we further introduce another trainable domain-sensitive latent variable $\beta^S,\beta^T \in \mathbb{R}^{d_\beta}$ to model the different strengths of the causal structures, \textcolor{black}{in which $d_\beta$ is the dimension of the trainable domain-sensitive latent variables and the dimension of $\beta^S$ and $\beta^T$ are same.} Therefore, we can formulate the domain-sensitive Granger-Causality-based prediction as follows:
\begin{equation}
\color{black}
% \small
\begin{split}
    \widehat{\bm{z}}_{t+1}^* = G\left(g_1\left(z_1^*,A_1^*;\phi_1\right),\cdots,\right.g_j\left(z_j^*,A_j^*;\phi_j\right),\cdots,\\\left.g_k\left(z_k^*,A_k^*;\phi_k\right);\phi_a,\beta^*\right),
\end{split}
\end{equation}
where $\widehat{\bm{z}}_{t+1}^*$ is the predicted result of source or target domain. For convenience, we let $\Phi=(\phi_a,\phi_1, \cdots, \phi_k,\beta^*)$ be all the trainable parameters in the domain-sensitive Granger-Causality-based prediction module. 

\textcolor{black}{
\subsection{Granger Causality Discrepancy Regularization}\label{sec:reg}
% \subsubsection{A Relaxed Assumption for Domain-variant Causal Structures}}
\subsubsection{\textcolor{black}{A Relaxed variant for the Causal Conditional Shift Assumption}}}
\textcolor{black}{According to Assumption \ref{ass:csa}, we assume that the causal structures are fixed across different domains. In fact, this assumption may be too strong to some extent, the causal structures might vary with different domains.
% For example, the corresponding time between ``Blood glucose'' and ``Glucagon'' of the young may be shorter than that of the old, which reflects that the causal structures from different domains are slightly different. 
If we directly forecast the target values with the help of the source causal structures, we may end up with suboptimal performance. Therefore, we relax the Causal Conditional Shift assumption and assume that the causal structures from different domains are slightly different. Based on this relaxed assumption, we propose a regularization term, which can leverage the labeled source data to learn the domain-invariant components and exploit the limited labeled target data to learn the domain-specific components. There are different ways to achieve this goal, and we employ L1 Norm in our implementation. Formally, we have: 
\begin{equation}\label{summary_simi2}
\color{black}
\small
    Disc(\bm{A}^S, \bm{A}^T) = \frac{1}{D \times D \times k}\sum_{j=1}^k||A_j^S - A_j^T||_1,
\end{equation}
where $A_j^S$ and $A_j^T$ are the structures with lag $j$ and $||\cdot||_1$ denotes the L1 Norm. 
% Note that L1 Norm of a matrix is calculated by $||A||_1=\sum\limits_{uv}|A_{uv}|$ , where $A_{uv}$ is the $u$-th row $v$-th column element in matrix $A$. 
}

% Since the causal structures usually vary with different domains, inferring the future values with the source causal structures may result in suboptimal performance. Moreover, the limited labeled target domain data make it difficult to discover the causal structures of the target domain. Therefore, we devise the simple but effective Granger Causality similarity regularization term to restrict the similarity between the causal structures of the source and the target domains. There are many methods to achieving it and we employ L1 Norm in our design. 
\textcolor{black}{
\subsubsection{Causal Structure Discrepancy Regularization with Gradient Stopping}}

% \textcolor{black}{In this part, we devise a simple but effective causal structure similarity regularization, which can not only transfer the domain-invariant substructures from the source to the target but also preserve the domain-specific substructures. There are many methods to achieve this goal, and we employ L1 Norm in our implementation.}

\textcolor{black}{
In the ideal case, the source structures are learned from the source labeled data while the target structures are optimized according to the regularization term shown in Equation (\ref{summary_simi2}). However, straightforwardly using L1 Norm can only make the causal structures from different domains become similar, but it can not guarantee that the model can learn correct components and may learn degenerated causal structures. In the worst case, the target structure solely learned on the insufficient labeled data is wrong. In the meanwhile, the source structures are pushed by the regularization term to become similar to the wrong target structures, which results in the negative transfer and finally the suboptimal results. 
}

\textcolor{black}{
In order to address the aforementioned issue, we further apply the gradient stopping operation $C(\cdot)$ \cite{valvano2021self,metz2016unrolled} on the source domain causal structures (as shown in Figure \ref{fig:model} (c)), which is devised to stop the gradients of regularization term propagating through the source structures estimation module. Therefore, it is confident to learn the correct source causal structures with the massive source labeled data and without obstruction from the target causal structures. Furthermore, the target causal structures are guided by the discrepancy term and the well-trained source causal structures. In this way, we let the source structures be an anchor and push the target structures to be close to the source ones, so we can avoid the negative transfer and achieve the ideal case.
}

% However, simply restricting the similarity between the causal structures may result in the negative transfer as the model reconstructs the invalid target causal structures and the regularization term further poisons the source causal structures. Fortunately, considering that there are massive labeled source domain data, the learned source causal structures are accurate with high probability, so we apply the gradient stopping operation on the source domain causal structures. With the help of gradient stopping operation, the regularization term can avoid the negative transfer by stopping the propagation of the gradient of the source causal structures and only optimizing the target causal structures.

\textcolor{black}{Therefore, we combine the gradient stopping operation with Equation (\ref{summary_simi2}), the Granger Causality discrepancy regularization term can be formulated as follows:
\begin{equation}\label{summary_simi}
\color{black}
    Disc(\bm{A}^S, \bm{A}^T) = \frac{1}{D \times D \times k}\sum_{j=1}^k||C(A_j^S) - A_j^T||_1,
\end{equation}
where $C(\cdot)$ is used to stop the gradients from propagating to the source causal structure estimator.\\} 
% The proposed Granger Causality similarity regularization term is not only beneficial to reconstruct the domain-invariant module of Granger Causality but also playing an important role in bridging the structure information from the source domain to the target domain.

\subsection{Training and Inference}\label{sec:summary}
By discovering the Granger Causality and further predicting the future values with the help of the Granger Causality, we summarize the model as follows.

% The labeled data from the source domain $(\bm{x}, \bm{y})\in (\mathcal{X}_S, \mathcal{Y}_S)$ and the few labeled data from the target domain $(\bm{x}, \bm{y})\in (\mathcal{X}_T, \mathcal{Y}_T)$ are respectively used as the input of the evident lower bound of the source and the target domain. And we further use all the source and target domain data for summary graph alignment regularization.
During the training steps, we take the labeled source data $(\bm{x}^S, \bm{y}^S)\in (\mathcal{X}_S, \mathcal{Y}_S)$ and the few labeled data from the target domain $(\bm{x}^T, \bm{y}^T)\in (\mathcal{X}_T, \mathcal{Y}_T)$ into the proposed method and predict the future values of the next timestamps. In the testing steps, we need to predict the future $\tau$ stamps, so we use the previous predicted value as the new input and obtain $\bm{z}_{t+1}^*,\cdots,\bm{z}_{t+\tau}$ in an autoregressive way.

\textcolor{black}{If the dimension of $\bm{y}$ equals that of $\bm{x}$ (e.g. the case of human motion prediction), the proposed \textbf{GCA} method needs to predict the multivariate time-series data. If the predicted results $\bm{y}$ is the univariate time-series (e.g. PM2.5 Prediction), it is not necessary to take all the predicted time-series into the loss function. In this case, we can directly use Equation (\ref{equ:loss2}) to optimize the proposed method.} 

\textcolor{black}{If the predicted result is the univariate time-series $\bm{z}^d$ and we need to estimate the causal structure $\bm{A}$ for interpretability, we need to take all the predicted time-series into the loss function. However, the performance will easily suffer from degeneration, because the model might optimize other variables that are unrelated to $\bm{y}$. In order to address this issue, we further devise an extra strengthen loss. We formulate the extra optimization term below:}
% Since the proposed method is an autoregressive model that needs to take all the predicted time-series into the loss function. If the task is to predict a univariate $\bm{z}^d$, it is easy to suffer from suboptimization, because the model might optimize other variables. We address this problem by adding an extra strengthen loss. We formulate the extra optimization term follows:
\begin{equation}
    \mathcal{L}_e = MSE(\widehat{\bm{z}}^d_{t+1}, \bm{z}^d_{t+1}), 
\end{equation}
in which $\widehat{\bm{z}}^d_{t+1}$ is the predicted time-series and $\bm{z}^d_{t+1}$ is the groud truth value; And $MSE$ denotes the mean square error. \textcolor{black}{We will investigate the effectiveness of the extra optimization term $\mathcal{L}_e$ in the ablation studies.} 

% \footnote{\textcolor{black}{Note that the L1 Norm and the L2 Norm of a matrix are respectively calculated by $||A||_1=\sum\limits_{uv}|A_{uv}|$ and $||A||_2=\sqrt{\sum\limits_{uv}A_{uv}^2}$, where $A_{uv}$ is the $u$-th row $v$-th column element in matrix $A$.} }.

Under the above objective functions, the total loss of the proposed GCA method can be summarized as the following equation:
\begin{equation}
\label{equ:loss_finall}
\small
\color{black}
\begin{split}
    &\!\!\mathcal{L}_{total}\!\\=&\!-\mathcal{L}_{\text{ELBO}}^S \!-\! \mathcal{L}_{\text{ELBO}}^T\!+\! \gamma Disc(\!\bm{A}^S,\! \bm{A}^T\!) \!+\! \lambda(\mathcal{R}(\bm{A}^S)\!+\!\mathcal{R}(\bm{A}^T))  + \delta \mathcal{L}_e \\=&-\mathbb{E}_{Q(A_1^S|\cdot)},\cdots\mathbb{E}_{Q(A_k^S|\cdot)}\left[\ln P\!\left(\bm{z}_t^S|\bm{z}^S_{t-1},\!\cdots\!,\bm{z}^S_{t-k},A_1^S,\!\cdots\!,A_k^S\right)\right] \\&+ D_{KL}(Q(A_1^S|\cdot)||P(A^S_1|\cdot)) \\&+ \sum_{j=2}^k \mathbb{E}_{Q(A^S_1|\cdot)},\cdots\mathbb{E}_{Q(A^S_k|\cdot)} D_{KL}(Q(A^S_j|\cdot)||P(A^S_j|\cdot)) \\&-\mathbb{E}_{Q(A_1^T|\cdot)},\cdots\mathbb{E}_{Q(A^T_k|\cdot)}\left[\ln P\!\left(\bm{z}^T_t|\bm{z}^T_{t-1},\!\cdots\!,\bm{z}^T_{t-k},A_1^T,\!\cdots\!,A_k^T\right)\right] \\&+ D_{KL}(Q(A^T_1|\cdot)||P(A^T_1|\cdot)) \\&+ \sum_{j=2}^k \mathbb{E}_{Q(A^T_1|\cdot)},\cdots\mathbb{E}_{Q(A^T_k|\cdot)} D_{KL}(Q(A^T_j|\cdot)||P(A^T_j|\cdot))\\&+ \gamma\sum_{j=1}^{k}\mathbb{E}_{Q(A_j^S|\cdot)}\mathbb{E}_{Q_j(A_j^T|\cdot)}Disc\left(A_j^S,A_j^T\right)\\&+ 
    \lambda\left(\sum_{j=1}^{k}\mathbb{E}_{Q(A_j^S|\cdot)}\mathcal{R}(A^S_j) + \sum_{j=1}^{k}\mathbb{E}_{Q(A_j^T|\cdot)}\mathcal{R}(A^T_j)\right) + \delta \mathcal{L}_e,
\end{split}
% \begin{split}
%     &\mathcal{L}_{total}=\\&\mathbb{E}_{A_1 \sim Q_1}^S \cdots \mathbb{E}_{A_k \sim Q_k}^S\left[\ln P^S\left(\bm{z}_t|\bm{z}_{t-1},\!\cdots\!,\bm{z}_{t-k},A_1^S,\!\cdots\!,A_k^S\right)\right] \\-& D_{KL}\left(Q_1^S||P_1^S\right) - \sum_{j=2}^k \mathbb{E}_{Q_1^S} \cdots \mathbb{E}_{Q_{j-1}^S} D_{KL}\left(Q_j^S||P_j^S\right) \\+& \mathbb{E}_{A_1 \sim Q_1}^T \cdots \mathbb{E}_{A_k \sim Q_k}^T\left[\ln P^T\left(\bm{z}_t|\bm{z}_{t-1},\!\cdots\!,\bm{z}_{t-k},A_1^T,\!\cdots\!,A_k^T\right)\right]\\-& D_{KL}\left(Q_1^T||P_1^T\right) - \sum_{j=2}^k \mathbb{E}_{Q_1^T} \cdots \mathbb{E}_{Q_{j-1}^T} D_{KL}\left(Q_j^T||P_j^T\right)\\+& 
%     \lambda\left(\sum_{j=1}^{k}\mathbb{E}_{A^S_j \sim Q_j^S}\mathcal{R}(A^S_j) + \sum_{j=1}^{k}\mathbb{E}_{A^T_j \sim Q_j^T}\mathcal{R}(A^T_j)\right)\\+& 
%     \gamma\sum_{j=1}^{k}\mathbb{E}_{Q_j^S}\mathbb{E}_{Q_j^T}Disc\left(A_j^S,A_j^T\right) + \delta \mathcal{L}_e,
% \end{split}
\end{equation}
in which $\gamma, \delta$ and $\lambda$ are the hyper-parameters. 

In summary, the proposed model is trained on the labeled source and target domain data using the following procedure:
\begin{equation} 
    (\hat{\Theta}, \hat{\Phi}, )=\mathop{\arg\min}\limits_{\Theta, \Phi} \mathcal{L}_{total}.
\end{equation}

% In the inference phase, we first reconstruct the full time graphs of different lags via Equation (\ref{encoder_imp_2}), and then we use the full time graphs and the observed value to predict the future value. 

\section{Theoretical Analysis}\label{theorm}
In this section, we discuss the theoretical insights and establish the generalization bound, with respect to the discrepancy of causal structures between the source and the target domain.

There are several works about the generalization theory of domain adaptation \cite{mohri2018foundations,zhang2019bridging,mansour2009domain,cortes2011domain,ben2007analysis}, but many of them require that the loss function is bounded by a positive value. In this paper, we aim to establish the generalization bound that bridges the causal structures and the error risk. We start with the relationship between the distance of the source and the target causal structures, as well as the error risks. As we focus on the time-series forecasting task, we let the loss function be the following form: $L: \bm{y} \times \bm{y}\rightarrow \mathbb{R}_+$, and it can be MSE, MAPE and so on. 

We let $\mathcal{H}$ be a hypothesis space that maps $(\bm{z}, \bm{A})$ to $\mathbb{R}$, and $h \in \mathcal{H}, h: (\bm{z}, \bm{A}) \rightarrow \mathbb{R}$ denotes a function take a time-series sample $\bm{z}$ and a causal structures $\bm{A}$ as input. We further let $h^S: (\bm{z}, \bm{A}^S) \rightarrow \mathbb{R}$ and $h^T: (\bm{z}, \bm{A}^T) \rightarrow \mathbb{R}$ be the function that respectively takes the source and target causal structures as input, and $h^S, h^T \in \mathcal{H}$.
% We let a hypothesis is a function $h: (\bm{z}, \bm{A}) \rightarrow \mathbb{R}$, which takes the pass time-series data $\bm{z}$ and causal structures $A$ as input. We further let $h^S: (\bm{z}, \bm{A}^S) \rightarrow \mathbb{R}$ and $h^T: (\bm{z}, \bm{A}^T) \rightarrow \mathbb{R}$
% we further let $h^S:A_S, x\rightarrow \mathbb{R}$ and $h^S:A_T, x\rightarrow \mathbb{R}$ be the functions that respective take the source causal structures and target causal structures as input. It is clear that $h^S, h^T \in \mathbb{H}$. 
Therefore, we can make the following assumption:
\begin{assumption}\textbf{Causal Structures Discrepancy Bound Assumption}: Given any two hypotheses $h^S, h^T$ and the causal structures $\bm{A}^S,\bm{A}^T$ from the source and target domain, a positive value $K$ exits that makes the following two inequalities hold:
\begin{equation}
\begin{split}
    \mathcal{L}^S(h^S,h^T)=\mathbb{E}_{\bm{x}\sim P^S(\bm{x})}L(h^S,h^T)\leq K ||\bm{A}^S-\bm{A}^T||_1,\\
    \mathcal{L}^T(h^S,h^T)=\mathbb{E}_{\bm{x}\sim P^T(\bm{x})}L(h^S,h^T)\leq K ||\bm{A}^S-\bm{A}^T||_1.
\end{split}
\end{equation}
\end{assumption}

This assumption implies that for any samples from any domains, the expectation of $L(h^S,h^T)$ between the $h^S$ and $h^T$ are bounded by the similarity between $\bm{A}^S$ and $\bm{A}^T$ over any domains. This assumption is reasonable when the loss function $L$ is local smoothness.

Before introducing the generalization bound of the proposed method, we first give the definition of \textit{Rademacher Complexity} and \textit{Rademacher Complexity Regression Bounds}, which are shown as follows:
\begin{definition}
\textbf{Rademacher Complexity.} \cite{mansour2009domain} Let $\mathcal{H}$ be the set of real-value functions defined over a set $X$. Given a dataset $\mathcal{X}$ whose size is $m$, the empirical Rademacher Complexity of $\mathcal{H}$ is defined as follows:
\begin{equation}
\widehat{\mathfrak{R}}_{\mathcal{X}}(\mathcal{H})=\frac{2}{m}\mathbb{E}_{\sigma}\left[\sup_{h\in \mathcal{H}}\left.|\sum_{i=1}^{m}\sigma_i h(\bm{x}_i)|\right|\mathcal{X}=(\bm{x}_1, \cdots,\bm{x}_m)\right],
\end{equation}
in which $\sigma_=(\sigma_1, \cdots, \sigma_m)$ are independent uniform random variables taking values in $\{-1, +1\}$. And the \textit{Rademacher Complexity} of a hypothesis set $\mathcal{H}$ is defined as the expectation of $\mathfrak{R}_{\mathcal{X}}(\mathcal{H})$ over all the dataset of size $m$, which is shown as follows:
\begin{equation}
    \mathfrak{R}_{m}(\mathcal{H})=\mathbb{E}_\mathcal{\zeta}\left[\left.\widehat{\mathfrak{R}}_{\mathcal{X}}(\mathcal{H})\right|\left|\mathcal{\zeta}\right|=m\right].
\end{equation}
\end{definition}

\begin{lemma}\label{def:rade_bound}
\textbf{(Rademacher Complexity Regression Bounds)} \cite{mansour2009domain} Let $L: \bm{y} \times \bm{y} \rightarrow \mathbb{R}_+$ be a non-negative loss upper bounded by $M>0$ ($L(\bm{y}, \bm{y'})\leq M$ for all $\bm{y},\bm{y'} \in \mathcal{Y}$) and we denote by $\eta^{\mathcal{D}}: \mathcal{X}\rightarrow\mathcal{Y}$ the labeling function of domain $\mathcal{D}$. Given any fixed $\bm{y}' \in \mathcal{Y}, \bm{y}\rightarrow L(\bm{y},\bm{y}')$ is $\mu-$Lipschitz for some $\mu>0$, and the \textit{Redemacher Complexity Regression Bounds} are shown as follows:
\begin{equation}
\small
\begin{split}
        \mathcal{L}^{\mathcal{D}}(h^{\mathcal{D}}, \eta^{\mathcal{D}})\leq&\frac{1}{m^\mathcal{D}}\sum_{i=1}^{m^\mathcal{D}} L(h^{\mathcal{D}}(\bm{x}_i^{\mathcal{D}}), \bm{y}_i^{\mathcal{D}}) +  2\mu\mathfrak{R}_{m^{\mathcal{D}}}(\mathcal{H})\\&+M\sqrt{\frac{\log\frac{2}{\sigma}}{2m^{\mathcal{D}}}},
\end{split}
\end{equation}
in which $m^{\mathcal{D}}$ is the dataset of domain $\mathcal{D}$ with the size of $m^{\mathcal{D}}$.
\end{lemma}

Based on the aforementioned definition and assumption, we propose the generalization bound of \textbf{GCA}, which is shown as follows.
\begin{theorem}
We first let $\eta^S$ and $\eta^T$ be the source and target labeling function. We further let $h^*$ be the ideal predictor that simultaneously obtains the minimal error on the source and the target domain. Formally, we have:
% \begin{equation}
%     h^*\triangleq\mathop{\arg\min}_{h\in\mathcal{H}: h^{S} = h,A^{S}}\{\mathcal{L}^{\mathcal{S}}(h^S,\eta^S)+\mathcal{L}^{\mathcal{T}}(h^T,\eta^T)\}.
% \end{equation}
\begin{equation}
h^*\triangleq\mathop{\arg\min}_{h^S,h^T\in\mathcal{H}}\{\mathcal{L}^{\mathcal{S}}(h^S,\eta^S)+\mathcal{L}^{\mathcal{T}}(h^T,\eta^T)\}.
\end{equation}
Then we can obtain the following inequation for any $h^S, h^T \in \mathcal{H}$:
\begin{equation}
\small
\begin{split}
    \mathcal{L}^T(h^T, \eta^T)\leq & \frac{1}{m^S}\sum^{m^S}_{i=1}L\left(h^S(\bm{x}_i^S),\bm{y}^S_i\right) + 2\mu\mathfrak{R}_{m^S}\left(\mathcal{H}\right) \\&+3M\sqrt{\frac{\log\frac{2}{\sigma}}{2m^S}} + \frac{1}{m^T}\sum^{m^T}_{i=1}L\left(h^T(\bm{x}_i^T),\bm{y}^T_i\right) \\&+ 2\mu\mathfrak{R}_{m^T}\left(\mathcal{H}\right) +3M\sqrt{\frac{\log\frac{2}{\sigma}}{2m^T}} \\&+ 3K|\bm{A}^S-\bm{A}^T| - \mathcal{L}^S\left(h^*,\eta^S\right),
\end{split}
\end{equation}
where $\mathcal{L}^S\left(\eta^S,h^*\right)$ is independent of $h^T$ or $h^S$ and can be considered as a constant; and $m^S$ and $m^T$ are the data size of the source and target domain training set.
\begin{proof}
\begin{equation}
\small
\begin{split}
    &\mathcal{L}^T\left(h^T,\eta^T\right) = \mathcal{L}^T\left(h^T, \eta^T\right) + \mathcal{L}^S\left(h^S, \eta^S\right) - \mathcal{L}^S\left(h^S, \eta^S\right) \\\leq& \mathcal{L}^S\left(h^S, \eta^S\right) + \mathcal{L}^T(h^T, h^S) + \mathcal{L}^T\left(h^S, \eta^T\right) + \mathcal{L}^S\left(h^S, h^T\right) \\&-\mathcal{L}^S\left(h^T, \eta^S\right)\\\leq&2K|\bm{A}^S-\bm{A}^T| + \mathcal{L}^S\left(h^S,\eta^S\right)+\mathcal{L}^T\left(h^S,\eta^T\right)-\mathcal{L}^S\left(h^*, \eta^S\right) \\\leq& 2K|\bm{A}^S-\bm{A}^T| + \mathcal{L}^S\left(h^S,\eta^S\right) + \mathcal{L}^T\left(h^S,h^T\right) \\&+ \mathcal{L}^T\left(h^T,\eta^T\right)-\mathcal{L}^S\left(h^*,\eta^S\right) \\\leq& 
    3K|\bm{A}^S-\bm{A}^T| + \mathcal{L}^S\left(h^S,\eta^S\right) + \mathcal{L}^T\left(h^T,\eta^T\right) \\&- \mathcal{L}^S\left(h^*,\eta^S\right) \\\leq&  \frac{1}{m^S}\sum^{m^S}_{i=1}L\left(h^S(\bm{x}_i^S),\bm{y}^S_i\right) + 2\mu\mathfrak{R}_{m^S}\left(\mathcal{H}\right) +3M\sqrt{\frac{\log\frac{2}{\sigma}}{2m^S}} \\&+ \frac{1}{m^T}\sum^{m^T}_{i=1}L\left(h^T(\bm{x}_i^T),\bm{y}^T_i\right) + 2\mu\mathfrak{R}_{m^T}\left(\mathcal{H}\right) +3M\sqrt{\frac{\log\frac{2}{\sigma}}{2m^T}} \\&+ 3K|\bm{A}^S-\bm{A}^T| - \mathcal{L}^S\left(h^*,\eta^S\right)
\end{split}
\end{equation}
\end{proof}
\end{theorem}
According to the aforementioned generalization bound, we can find that the expected risk of $h^T$ is not only controlled by the empirical loss on the labeled source and target data, but also by the discrepancy between the source and target domain structures.

\section{Experiment}\label{exp}
\subsection{Datasets}
In this section, we give a brief introduction to the dataset and the preprocessing steps. For each dataset, we split it into the training set, validation set, and test set. For all the methods, we run five different random seeds and report the mean and variance. We choose the model with the best validation and evaluate the chosen model on the test set. For all the datasets, we randomly draw 5\% labeled data from the target distribution for training.
% All the code for simulation and prepossessing of the dataset will be released soon.

\subsubsection{Simulation Datasets}
In this part, we design a series of controlled experiments on the random causal structures with a given sample size and variable size. we simulate three different datasets as three domains with the following nonlinear function.
\begin{equation}
\begin{split}
\bm{z}_t=&A_1\cdot\left(\bm{z}_{t-1}+c\cdot Sin\left(\bm{z}_{t-1}\right)\right) +\cdots\\& +A_k\cdot\left(\bm{z}_{t-k}+c\cdot Sin\left(\bm{z}_{t-k}\right)\right) + \epsilon ,
\end{split}
\end{equation}
in which $c$ is the constant of the nonlinear term; $\epsilon$ is the variance of Gaussian distributions; $A_j$ is the substructure of causal structure with $j$ lag. \textcolor{black}{For each domain, the causal structures are similar with only a little difference.} We further employ different sample intervals for different domains. The details of the dataset are shown in Table 1. And Figure 5 shows an example of the simulated data from different domains. Since the value range of this dataset varies with different cities, we use Z-score Normalization to respectively preprocess all the datasets for each domain.

With the simulated datasets, we can use simulated datasets for both time-series forecasting and cross-domain Granger-causal discovery tasks, since we can achieve the ground truth Granger-causal structures. In this paper, we demonstrate the reasonability of the proposed GCA method and the advantages  of the well-reconstructed Granger Causality.

\begin{table}[t]
\caption{Different settings for the dataset of different domains.}
\centering
\begin{tabular}{c|ccc}
		\toprule
		\small{Domain}  & Variance of Noise & Sample interval &c\\
		\midrule
		\small{Domain 1}       	           & 1  & 1 & 0.02\\ 
		\small{Domain 2}              	   & 5 & 2 & 0.04\\
		\small{Domain 3}            	   & 10 & 3 & 0.06\\
		\bottomrule
\end{tabular}
\label{tab:params}
\end{table}

\begin{figure}[t]
	\centering
\includegraphics[width=\columnwidth]{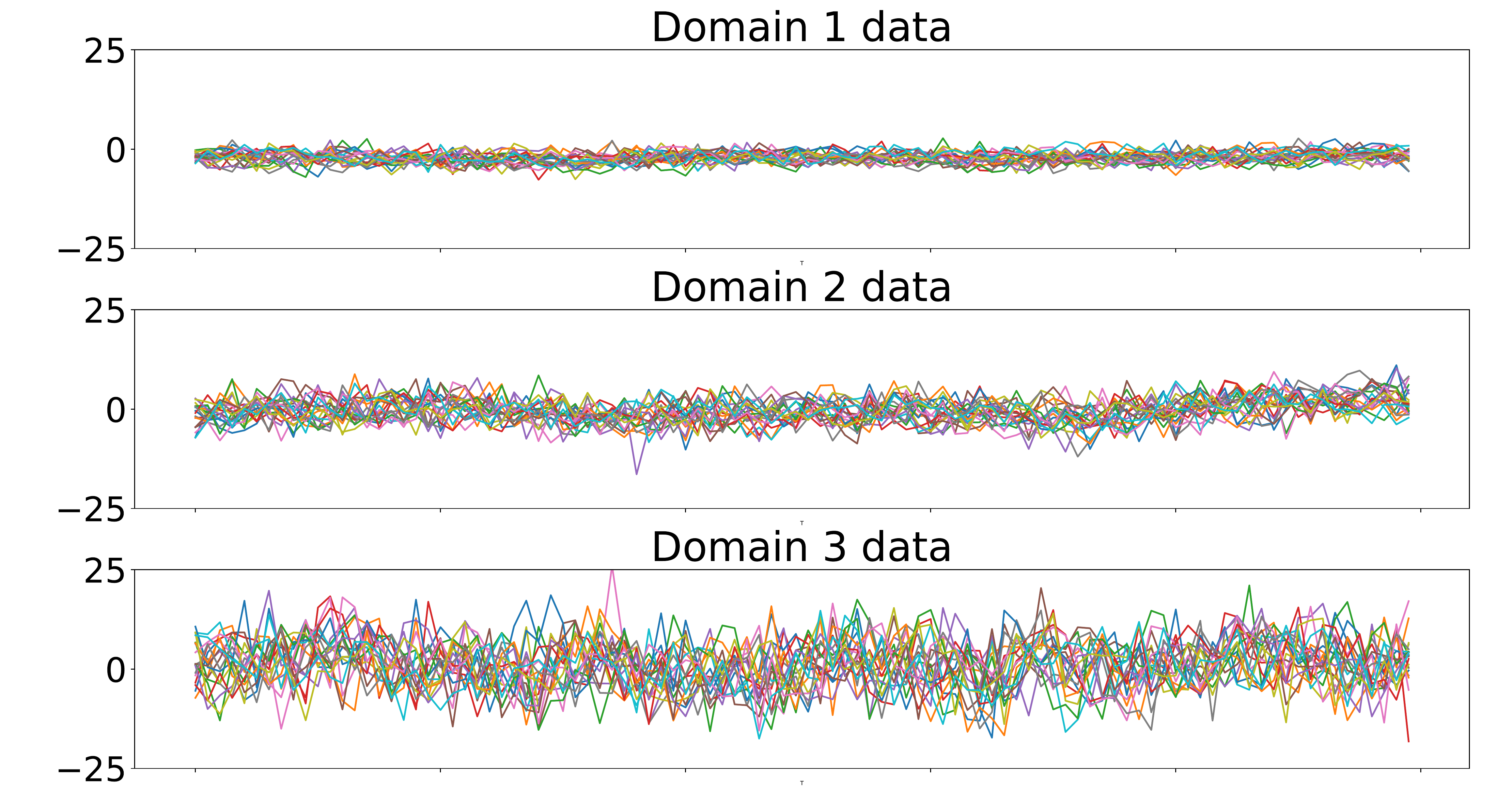}
	\caption{The Samples of three domains, we can find that the distributions are totally different, but they share a similar causal graph.}
	\label{fig:data}
\end{figure}

\subsection{Realworld Datasets}
\textcolor{black}{Besides the simulation dataset, we further evaluate the proposed GCA model on 6 real-world datasets: Air Quality Forecasting Dataset, Human Motion Capture Dataset, PPG-DaLiA dataset, Electricity Load Diagrams datasets, PEMS-BAY dataset, and Electricity Transformer Dataset. we publish the average performance of all the combinations. We repeat each experiment over 5 random seeds. Please refer to Appendix B and C for more detailed dataset descriptions and experiment results.}

\subsection{Baseline}
We first introduce the baselines of semi-supervised domain adaptation for time-series forecasting, including the latest semi-supervised domain adaptation method and domain adaptation methods for time-series data. We employ the same setting, e.g., the same labeled target domain data, for all the methods.
\begin{itemize}[leftmargin=*]
  \item \textbf{LSTM\_S+T }only uses the labeled source domain data and the labeled target domain data to train a vanilla LSTM model and apply it to the unlabeled target domain data.
%   \item \textbf{GCA\_T2T:} Since the proposed method is different from the previous method that is based on LSTM, so we further devise the GCA\_T2T baseline that is expected to provide the upper bound performance of our approach. As for the simulated dataset, we use the ground-truth Granger-Causality. As for the real-world dataset, we employ the inferred Granger-Causality.
  \item \textbf{R-DANN} \cite{da2020remaining} uses the domain adaptation architecture proposed in \cite{ganin2015unsupervised} with Gradient Reversal Layer (GRL) on LSTM, which is a straightforward solution for time-series domain adaptation.
  \item \textbf{RDC}: Deep domain confusion is a domain adaptation method proposed in \cite{tzeng2014deep} which minimizes the distance between the source and the target distribution by using Maximum Mean Discrepancy (MMD). Similar to the R-DANN method, we also employ LSTM as the feature extractor for time-series data.
  \item \textbf{VRADA} \cite{purushotham2016variational} is a time-series domain adaptation method that combines the Gradient Reversal Layer and VRNN \cite{chung2015recurrent}.
  \item \textbf{SASA} \cite{Cai_Chen_Li_Chen_Zhang_Ye_Li_Yang_Zhang_2021} is one of the state-of-the-art domain adaptation approaches for time-series data, that extracts and aligns the sparse associative structures.
  \item \textbf{LIRR} \cite{li2021learning} is one of the state-of-the-art semi-supervised domain adaptation approaches for classification and regression, we replace the feature extractor with LSTM for time-series data.
  \item \textcolor{black}{\textbf{CMTN} \cite{li2021causal} leverages the attention mechanisms to capture the dynamic and temporal causal mechanism of time series data for time-series domain adaptation.}
  \item \textcolor{black}{\textbf{DAF} \cite{jin2022domain}  uses an attention-based shared module with a domain discriminator to learn  domain-invariant and domain-specific features, respectively.}
  \item \textcolor{black}{\textbf{Autoformer} \cite{wu2021autoformer}: We employ the transformer-based backbone network, Autoformer, to extract the representation of time-series data. Since the Autoformer model is not devised for domain adaptation, we use a gradient reversal layer to extract the domain-invariant information. }
  \item \textcolor{black}{\textbf{S4D} \cite{gu2022parameterization}: We employ the state-space-model based backbone networks, the diagonal of structured state-space-model (S4D) \cite{gu2021efficiently}, to extract the representation of time-series data. Since the S4D model is not devised for domain adaptation, we use a gradient reversal layer to extract the domain-invariant information. }
\end{itemize}

% Baselines of domain adaptation for cross-domain Granger Causality Inference.
% \begin{itemize}
%     \item \textbf{VAR.} The vector autoregressive model is the extension of the autoregressive (AR) model and is used to infer the Granger Causality.
%     \item \textbf{GVAR.} GVAR \cite{marcinkevivcs2021interpretable} infers the multivariate Granger causality under nonlinear dynamics based on an extension of self-explaining neural networks.
%     \item \textbf{cLSTM} and \textbf{cMLP}. The \textbf{cLSTM} and \textbf{cMLP} combine the structured MLP and LSTM with sparsity-inducing penalties on the weights. 
% \end{itemize}

% \subsection{Result}

\subsection{Result on Simulation Dataset}

\begin{table*}
\setlength\tabcolsep{4pt}
\setlength{\abovecaptionskip}{0cm}  %段前
\setlength{\belowcaptionskip}{-0.2cm} %段后
% \renewcommand\arraystretch{1.2}
% \tiny
\small
\color{black}
	\centering
        \flushleft	\label{tab:simulate_forecast}
	\caption{The MAE and MSE on simulated datasets for the baselines and the proposed method. 
 % The values presented are averaged over 5 replicated with different random seeds. 
 }
	\begin{tabular}{p{1.1cm}<{\centering}|p{1.1cm}<{\centering}|p{1.05cm}<{\centering} p{1.05cm}<{\centering} p{1.05cm}<{\centering} p{1.05cm}<{\centering} p{1.05cm}<{\centering} p{1.05cm}<{\centering} p{1.25cm}
    <{\centering} p{1.25cm}<{\centering} p{1.25cm}
    <{\centering} p{1.05cm}<{\centering} p{1.25cm}<{\centering}}
	    \toprule
	    \footnotesize{Task} & \footnotesize{Metric} & \footnotesize{GCA} & \footnotesize{DAF}& \footnotesize{LIRR}& \footnotesize{SASA}& \footnotesize{VRADA}& \footnotesize{CMTN}& \footnotesize{Autoformer} &
     \footnotesize{S4D} &
     \footnotesize{R-DANN} &
     \footnotesize{RDC} & \footnotesize{LSTM\_S+T}\\
		\midrule
		\multirow{2}*{$1\!\rightarrow\!2$} & RMSE & $\bm{0.9232}$ & $0.9450$ & $0.9585$ & $1.0001$ & $0.9763$ & $0.9841$ & $1.0259$ & $0.9479$ & $1.0148$ & $1.0122$ & $1.0078$\\
% 		\cline{2-9}
		~ & MAE &$\bm{0.7648}$ & $0.7742$ & $0.7656$ & $0.7961$ & $0.7858$ & $0.7924$ & $0.8215$ & $0.7773$ & $0.8017$ & $0.8841$ & $0.7990$\\
		\midrule
		\multirow{2}*{$1\!\rightarrow\!3$} & RMSE & $\bm{0.8556}$ & $0.8799$ & $0.9146$ & $0.8734$ & $0.9005$ & $0.8830$ & $1.0253$ & $0.8624$ & $0.8884$ & $0.8841$ & $0.8739$\\
% 		\cline{2-9}
		~ & MAE & $\bm{0.7399}$ & $0.7479$ & $0.7419$ & $0.7485$ & $0.7598$ & $0.7519$ & $0.8186$ & $0.7404$ & $0.7539$ & $0.7527$ & $0.7489$\\
		\midrule
		\multirow{2}*{$2\!\rightarrow\!1$} & RMSE & $\bm{0.8323}$ & $0.9050$ & $0.9208$ & $0.9386$ & $0.9277$ & $0.9396$ & $1.0033$ & $0.9287$ & $0.9333$ & $0.9397$ & $0.9469$\\
% 		\cline{2-9}
		~ & MAE & $\bm{0.7258}$ & $0.7600$ & $0.7336$ & $0.7741$ & $0.7682$ & $0.7746$ & $0.8019$ & $0.7716$ & $0.7713$ & $0.7740$ & $0.7778$\\
		\midrule
		\multirow{2}*{$2\!\rightarrow\!3$} & RMSE & $\bm{0.8570}$ & $0.8815$ & $0.9237$ & $0.9084$ & $0.9229$ & $0.8912$ & $1.0053$ & $0.8651$ & $0.9136$ & $0.9125$ & $0.9178$\\
% 		\cline{2-9}
		~ & MAE & $\bm{0.7493}$ & $0.7500$ & $0.7601$ & $0.7625$ & $0.7688$ & $0.7541$ & $0.8036$ & $0.7494$ & $0.7649$ & $0.7640$ & $0.7669$\\
		\midrule
		\multirow{2}*{$3\!\rightarrow\!1$} & RMSE & $\bm{0.8314}$ & $0.8952$ & $0.9129$ & $0.8596$ & $0.9177$ & $0.9373$ & $1.0089$ & $0.9120$ & $0.8632$ & $0.8596$ & $0.8631$\\
% 		\cline{2-9}
		~ & MAE & $\bm{0.7255}$ & $0.7578$ & $0.7293$ & $0.7379$ & $0.7635$ & $0.7761$ & $0.8082$ & $0.7636$ & $0.7393$ & $0.7373$ & $0.7395$\\
		\midrule
		\multirow{2}*{$3\!\rightarrow\!2$} & RMSE & $\bm{0.9208}$ & $0.9326$ & $0.9539$ & $0.9685$ & $0.9753$ & $0.9739$ & $1.0156$ & $0.9331$ & $0.9752$ & $0.9710$ & $0.9759$ \\
% 		\cline{2-9}
		~ & MAE & $\bm{0.7634}$ & $0.7865$ & $0.7612$ & $0.7829$ & $0.7855$ & $0.7863$ & $0.8123$ & $0.7698$ & $0.7859$ & $0.7840$ & $0.7863$\\
		\midrule
		\multirow{2}*{Average} & RMSE & $\bm{0.8700}$ & $0.9065$ & $0.9307$ & $0.9249$ & $0.9368$ & $0.9348$ & $1.0140$ & $0.9082$ & $0.9314$ & $0.9299$ & $0.9309$\\
% 		\cline{2-9}
		~ & MAE & $\bm{0.7433}$ & $0.7597$ & $0.7670$ & $0.7481$ & $0.7719$ & $0.7725$ & $0.8110$ & $0.7620$ & $0.7695$ & $0.7688$ & $0.7693$\\
		\bottomrule
	\end{tabular}
\end{table*}

\begin{table*}

\color{black}
\small
	\label{tab:air}
	\centering
	\caption{The MAE and MSE on Air Quality Forecasting dataset for the baselines and the proposed method. 
 % The values presented are averaged over 5 replicated with different random seeds.
 }
	\begin{tabular}{p{1.0cm}<{\centering}|p{0.85cm}<{\centering}|p{0.85cm}<{\centering} p{0.85cm}<{\centering} p{0.85cm}<{\centering} p{0.85cm}<{\centering} p{0.85cm}<{\centering} p{0.85cm}<{\centering} p{1.20cm}
    <{\centering} p{1.15cm}<{\centering} p{1.20cm}
    <{\centering} p{0.85cm}<{\centering} p{1.20cm}<{\centering}}
	    \toprule
	    \footnotesize{Task} & \footnotesize{Metric} & \footnotesize{GCA} & \footnotesize{DAF}& \footnotesize{LIRR}& \footnotesize{SASA}& \footnotesize{VRADA}& \footnotesize{CMTN}& \footnotesize{Autoformer} &
     \footnotesize{S4D} &
     \footnotesize{R-DANN} &
     \footnotesize{RDC} & \footnotesize{LSTM\_S+T}\\
		\midrule
		\multirow{2}*{$B\!\rightarrow\! T$} & RMSE & $0.1684$ & $0.1907$ & $0.1824$ & $0.1952$ & $0.2536$ & $0.2129$ & $\bm{0.1548}$ & $0.1803$ & $0.1868$ & $0.2199$ & $0.1910$ \\
% 		\cline{2-9}
		~ & MAE & $0.2669$ & $0.3062$ & $0.2941$ & $0.3037$ & $0.3698$ & $0.3409$ & $\bm{0.2643}$ & $0.2938$ & $0.3036$ & $0.3369$ & $0.2975$ \\
		\midrule
		\multirow{2}*{$G\!\rightarrow\! T$} & RMSE & $\bm{0.1497}$ & $0.1869$ & $0.1757$ & $0.1585$ & $0.2365$ & $0.2356$ & $0.1565$ & $0.1590$ & $0.1693$ & $0.1420$ & $0.1785$\\
% 		\cline{2-9}
		~ & MAE & $\bm{0.2618}$ & $0.3099$ & $0.2777$ & $0.2740$ & $0.3544$ & $0.3671$ & $0.2862$ & $0.2678$ & $0.2938$ & $0.2675$ & $0.3027$\\
		\midrule
		\multirow{2}*{$S\!\rightarrow\! T$} & RMSE & $\bm{0.1673}$ & $0.1990$ & $0.1754$ & $0.2022$ & $0.2662$ & $0.2061$ & $0.1690$ & $0.1746$ & $0.2330$ & $0.1877$ & $0.2058$\\
% 		\cline{2-9}
		~ & MAE & $\bm{0.2639}$ & $0.3214$ & $0.2771$ & $0.3130$ & $0.4287$ & $0.3480$ & $0.2823$ & $0.2846$ & $0.3834$ & $0.3194$ & $0.3195$\\
		\midrule
		\multirow{2}*{$T\!\rightarrow\! B$} & RMSE & $0.2278$ & $0.2364$ & $0.2783$ & $0.2366$ & $0.3112$ & $0.2534$ & $\bm{0.2146}$ & $0.2413$ & $0.3222$ & $0.3182$ & $0.2652$\\
% 		\cline{2-9}
		~ & MAE & $0.2841$ & $0.3213$ & $0.3470$ & $0.2962$ & $0.4007$ & $0.3392$ & $\bm{0.2804}$ & $0.2940$ & $0.3522$ & $0.3580$ & $0.3243$\\
		\midrule
		\multirow{2}*{$G\!\rightarrow\! B$} & RMSE & $\bm{0.2169}$ & $0.2744$ & $0.2657$ & $0.2388$ & $0.2414$ & $0.2552$ & $0.2204$ & $0.2419$ & $0.2706$ & $0.2276$ & $0.3010$\\
% 		\cline{2-9}
		~ & MAE & $\bm{0.2844}$ & $0.3554$ & $0.3217$ & $0.3151$ & $0.3722$ & $0.3493$ & $0.2858$ & $0.2883$ & $0.3672$ & $0.3077$ & $0.3854$\\
		\midrule
		\multirow{2}*{$S\!\rightarrow \!B$} & RMSE & $\bm{0.2319}$ & $0.2446$ & $0.2571$ & $0.2671$ & $0.4222$ & $0.2858$ & $0.2508$ & $0.2758$ & $0.3017$ & $0.2744$ & $0.2503$\\
% 		\cline{2-9}
		~ & MAE & $\bm{0.2802}$ & $0.3259$ & $0.3067$ & $0.3273$ & $0.4621$ & $0.3616$ & $0.3184$ & $0.3271$ & $0.3856$ & $0.3499$ & $0.3303$\\
		\midrule
		\multirow{2}*{$B\!\rightarrow\! G$} & RMSE & $0.1996$ & $0.2203$ & $0.2214$ & $0.2281$ & $0.3135$ & $0.2832$ & $\bm{0.1994}$ & $0.2203$ & $0.2964$ & $0.2566$ & $0.2452$\\
% 		\cline{2-9}
		~ & MAE & $0.2999$ & $0.3341$ & $0.3186$ & $0.3315$ & $0.4256$ & $0.3975$ & $\bm{0.2942}$ & $0.3141$ & $0.4147$ & $0.3682$ & $0.3564$\\
		\midrule
		\multirow{2}*{$T\!\rightarrow\! G$} & RMSE & $\bm{0.1845}$ & $0.2222$ & $0.2071$ & $0.2225$ & $0.2302$ & $0.2325$ & $0.1909$ & $0.1969$ & $0.1941$ & $0.1912$ & $0.1959$\\
% 		\cline{2-9}
		~ & MAE & $\bm{0.2819}$ & $0.3334$ & $0.2988$ & $0.2892$ & $0.3643$ & $0.3573$ & $0.2952$ & $0.3018$ & $0.3259$ & $0.3101$ & $0.3069$\\
		\midrule
		\multirow{2}*{$S\!\rightarrow\! G$} & RMSE & $\bm{0.1769}$ & $0.1907$ & $0.2028$ & $0.2186$ & $0.4187$ & $0.2847$ & $0.1937$ & $0.1957$ & $0.2942$ & $0.2949$ & $0.1862$\\
% 		\cline{2-9}
		~ & MAE & $\bm{0.2887}$ & $0.3056$ & $0.2988$ & $0.3415$ & $0.5227$ & $0.1869$ & $0.2927$ & $0.3009$ & $0.4912$ & $0.4185$ & $0.3871$\\
		\midrule
		\multirow{2}*{$B\!\rightarrow\! S$} & RMSE & $\bm{0.1559}$ & $0.1708$ & $0.1857$ & $0.2281$ & $0.2240$ & $0.2389$ & $0.1787$ & $0.1833$ & $0.2530$ & $0.3350$ & $0.1861$\\
% 		\cline{2-9}
		~ & MAE & $\bm{0.2640}$ & $0.2955$ & $0.2967$ & $0.3315$ & $0.3698$ & $0.3668$ & $0.3067$ & $0.3021$ & $0.3957$ & $0.4521$ & $0.2940$\\
		\midrule
		\multirow{2}*{$T\!\rightarrow\! S$} & RMSE & $\bm{0.1536}$ & $0.1633$ & $0.1812$ & $0.1685$ & $0.2587$ & $0.2524$ & $0.1654$ & $0.1606$ & $0.2520$ & $0.2532$ & $0.1945$\\
% 		\cline{2-9}
		~ & MAE & $\bm{0.2640}$ & $0.2935$ & $0.3119$ & $0.3027$ & $0.4468$ & $0.3774$ & $0.2877$ & $0.2912$ & $0.3986$ & $0.3916$ & $0.2939$\\
		\midrule
		\multirow{2}*{$G\!\rightarrow\! S$} & RMSE & $\bm{0.1360}$ & $0.1538$ & $0.1710$ & $0.1552$ & $0.2914$ & $0.2303$ & $0.1536$ & $0.1411$ & $0.2017$ & $0.1839$ & $0.1529$ \\
% 		\cline{2-9}
		~ & MAE & $\bm{0.2574}$ & $0.2899$ & $0.2938$ & $0.2844$ & $0.4236$ & $0.3613$ & $0.2805$ & $0.2707$ & $0.3418$ & $0.3300$ & $0.2940$\\
		\midrule
		\multirow{2}*{Average} & RMSE & $\bm{0.1810}$ & $0.2044$ & $0.2087$ & $0.2067$ & $0.2889$ & $0.2476$ & $0.1873$ & $0.1976$ & $0.2479$ & $0.2404$ & $0.2128$ \\
% 		\cline{2-9}
		~ & MAE & $\bm{0.2750}$ & $0.3160$ & $0.3029$ & $0.3092$ & $0.4173$ & $0.3625$ & $0.2895$ & $0.2947$ & $0.3711$ & $0.3508$ & $0.3180$\\
		\bottomrule
	\end{tabular}
\end{table*}

The Mean Squared Error(MSE) and Mean Absolute Error (MAE) in the simulated dataset are shown in Table 3. We conduct the Wilcoxon signed-rank test on the reported accuracies, our method significantly outperforms the baselines, with a p-value threshold of 0.05. The proposed GCA method significantly outperforms the other baselines on all the tasks with a large gap. It is worth mentioning that our Granger Causality alignment method has a remarkable MSE promotion on most of the tasks (more than 10\% improvement compared with SASA and LIRR). For some easy tasks like $domain 1\rightarrow domain 2$ and $domain 2 \rightarrow domain 1$, the proposed method achieves the largest improvement among all the tasks. For the other harder tasks like $domain3\rightarrow domain 2$ and $domain 2 \rightarrow domain 3$, our method also achieves a very good result. As for the very challenging tasks like $domain 1\rightarrow domain 3$ and $domain 3 \rightarrow domain 1$, where the value ranges between the source and the target domain are large, GCA also obtains a comparable result. \textcolor{black}{It is noted that the transformer-based method, Autoformer, does not perform well. This is because Autoformer \cite{wu2021autoformer} employ transformer to capture the periodicity of time-series and the dataset generated via the causal mechanism do not contain any obvious periodicity.}

In order to study how the inferred Granger-causality has influenced the performance of the model, we further apply the proposed method to the test dataset under different performances of the Granger-Causality. In detail, we evaluate the GCA on the test set every epoch and record the Area Under the Precision-Recall Curve (AUPRC) of the inferred Granger  Causality. The experiment result is shown in Figure 8 (a). According to the experiment result, we have the following observations. (1) The proposed method can well discover the Granger causality among time-series data. (2) As the improves of the accuracy of the Granger-causality, the performance of our method increases. The proposed method achieves the best performance when the Granger causality is exactly discovered. This phenomenon shows the reasonability of the proposed GCA. 

\textcolor{black}{We also compare the convergence speed of the proposed GCA method with the other baselines, which is shown in Figure 8 (b). According to the convergence curve in the figure, we can find that the proposed GCA method not only achieves better performance but also enjoys faster convergence speed. This is because the GCA approach can discover the causal structures that follow the data generation process, so it can avoid the side effects of superfluous relationships and enjoys an efficient learning process.}

\begin{figure}[t]
\centering
\label{fig:exp1}
\subfigure[AUPRC and RMSE of GCA.]{
\begin{minipage}[t]{\linewidth}
\centering
\includegraphics[width=0.90\columnwidth]{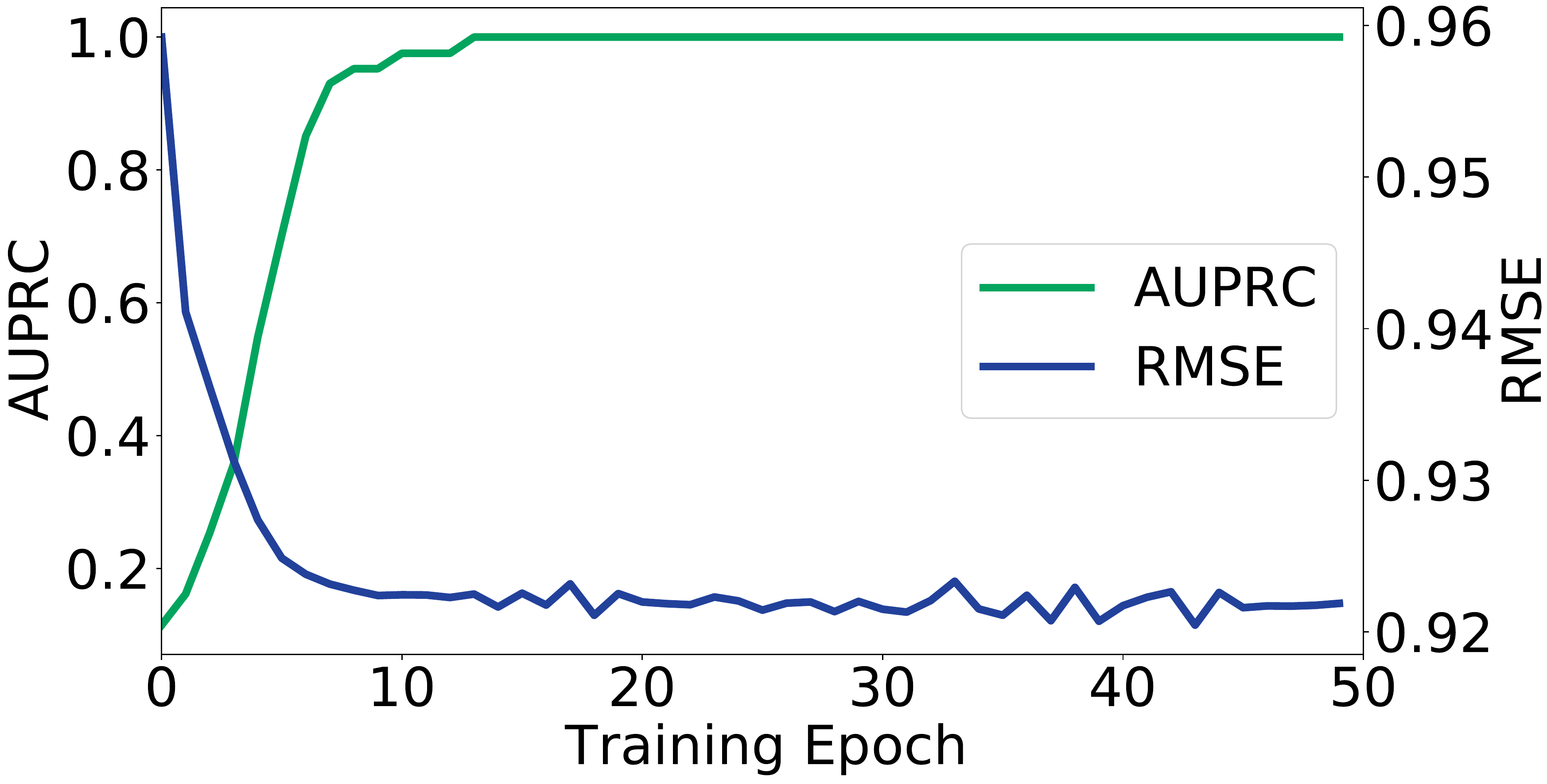}
%\caption{fig1}
\end{minipage}%
}%

\subfigure[Convergence of GCA and baselines.]{
\begin{minipage}[t]{\linewidth}
% \centering
\includegraphics[width=0.9\columnwidth]{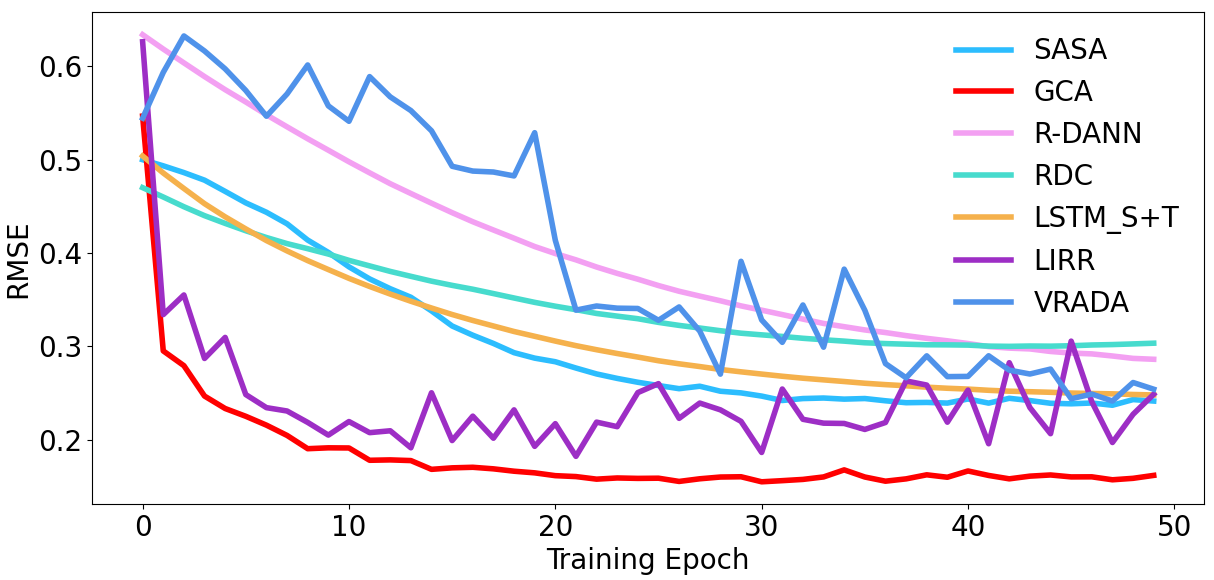}
%\caption{fig2}
\end{minipage}%
}%
\centering
\caption{(a) The convergence of the GCA method. We can find that the GCA method achieves the best performance when the Granger Causality is well reconstructed. (b) From the convergence curves of different methods, we can find that the proposed GCA method converges faster than the other methods. \textit{(best view in color.)}}
\end{figure}

\begin{table*}[t]
\color{black}
	\centering
 \small
	\caption{The MAE and MSE on Human Motion Forecasting dataset for the baselines and the proposed method. }
	\begin{tabular}{p{1.0cm}<{\centering}|p{0.85cm}<{\centering}|p{0.85cm}<{\centering} p{0.85cm}<{\centering} p{0.85cm}<{\centering} p{0.85cm}<{\centering} p{0.85cm}<{\centering} p{0.85cm}<{\centering} p{1.20cm}
    <{\centering} p{1.15cm}<{\centering} p{1.20cm}
    <{\centering} p{0.85cm}<{\centering} p{1.20cm}<{\centering}}
	    \toprule
	    \footnotesize{Task} & \footnotesize{Metric} & \footnotesize{GCA} & \footnotesize{DAF}& \footnotesize{LIRR}& \footnotesize{SASA}& \footnotesize{VRADA}& \footnotesize{CMTN}& \footnotesize{Autoformer} &
     \footnotesize{S4D} &
     \footnotesize{R-DANN} &
     \footnotesize{RDC} & \footnotesize{LSTM\_S+T}\\
		\midrule
		\multirow{2}*{$W\!\rightarrow \!G$} & RMSE & $\bm{0.1501}$ & $0.3250$ & $0.2182$ & $0.2316$ & $0.6770$ & $0.3534$ & $0.2710$ & $0.5052$ & $0.6186$ & $0.5160$ & $0.2295$\\
% 		\cline{2-9}
		~ & MAE & $0.2321$ & $0.4477$ & $0.3143$ & $0.3359$ & $0.6204$ & $0.4767$ & $\bm{0.1939}$ & $0.5260$ & $0.5815$ & $0.5301$ & $0.3355$\\
		\midrule
		\multirow{2}*{$W\!\rightarrow\! E$} & RMSE & $\bm{0.0941}$ & $0.2204$ & $0.1849$ & $0.2147$ & $0.8269$ & $0.3224$ & $0.2880$ & $0.4974$ & $0.5617$ & $0.6841$ & $0.2046$\\
% 		\cline{2-9}
		~ & MAE & $\bm{0.1852}$ & $0.3694$ & $0.2940$ & $0.3288$ & $0.6998$ & $0.4505$ & $0.1918$ & $0.5278$ & $0.5753$ & $0.6484$ & $0.3339$\\
		\midrule
		\multirow{2}*{$G\!\rightarrow\! W$} & RMSE & $\bm{0.1429}$ & $0.1504$ & $0.1542$ & $0.1585$ & $0.5944$ & $0.1979$ & $0.3879$ & $0.4161$ & $0.4044$ & $0.3516$ & $0.1680$\\
% 		\cline{2-9}
		~ & MAE & $\bm{0.2163}$ & $0.2994$ & $0.2601$ & $0.2754$ & $0.5843$ & $0.3527$ & $0.2652$ & $0.4897$ & $0.4751$ & $0.4561$ & $0.2893$\\
		\midrule
		\multirow{2}*{$G\!\rightarrow\! E$} & RMSE & $\bm{0.0851}$ & $0.1864$ & $0.1558$ & $0.1516$ & $0.6872$ & $0.3360$ & $0.2549$ & $0.4156$ & $0.4464$ & $0.5655$ & $0.1647$\\
% 		\cline{2-9}
		~ & MAE & $0.1714$ & $0.3234$ & $0.2622$ & $0.2760$ & $0.6265$ & $0.4645$ & $\bm{0.1661}$ & $0.4649$ & $0.5143$ & $0.5971$ & $0.2990$\\
		\midrule
		\multirow{2}*{$E\!\rightarrow \!W$} & RMSE & $\bm{0.1220}$ & $0.2167$ & $0.2072$ & $0.2400$ & $0.6642$ & $0.2425$ & $0.3964$ & $0.4612$ & $0.4904$ & $0.4038$ & $0.1978$\\
% 		\cline{2-9}
		~ & MAE & $\bm{0.1963}$ & $0.3636$ & $0.3042$ & $0.3487$ & $0.6133$ & $0.3891$ & $0.2751$ & $0.4985$ & $0.5305$ & $0.4755$ & $0.3183$\\
		\midrule
		\multirow{2}*{$E\!\rightarrow\! G$} & RMSE & $\bm{0.1383}$ & $0.1946$ & $0.1990$ & $0.2225$ & $0.7707$ & $0.3411$ & $0.3071$ & $0.4988$ & $0.4896$ & $0.5394$ & $0.2411$\\
% 		\cline{2-9}
		~ & MAE & $\bm{0.2042}$ & $0.3391$ & $0.3061$ & $0.3268$ & $0.6617$ & $0.4658$ & $0.2177$ & $0.5196$ & $0.5257$ & $0.5448$ & $0.3454$\\
		\midrule
		\multirow{2}*{Average} & RMSE & $\bm{0.1221}$ & $0.2156$ & $0.1865$ & $0.2031$ & $0.7034$ & $0.2989$ & $0.3176$ & $0.4652$ & $0.5018$ & $0.5101$ & $0.2010$\\
% 		\cline{2-9}
		~ & MAE & $\bm{0.2003}$ & $0.3571$ & $0.2902$ & $0.3152$ & $0.6343$ & $0.4332$ & $0.3282$ & $0.5044$ & $0.5337$ & $0.5420$ & $0.3202$\\
		\bottomrule
	\end{tabular}
\end{table*}

\begin{table*}
\color{black}
\small
	\label{tab:ppg}
	\centering
	\caption{The MAE and MSE on Heart Rate Forecasting dataset for the baselines and the proposed method.}
	\begin{tabular}{p{1.0cm}<{\centering}|p{0.85cm}<{\centering}|p{0.85cm}<{\centering} p{0.85cm}<{\centering} p{0.85cm}<{\centering} p{0.85cm}<{\centering} p{0.85cm}<{\centering} p{0.85cm}<{\centering} p{1.20cm}
    <{\centering} p{1.15cm}<{\centering} p{1.20cm}
    <{\centering} p{0.85cm}<{\centering} p{1.20cm}<{\centering}}
	    \toprule
	    \footnotesize{Task} & \footnotesize{Metric} & \footnotesize{GCA} & \footnotesize{DAF}& \footnotesize{LIRR}& \footnotesize{SASA}& \footnotesize{VRADA}& \footnotesize{CMTN}& \footnotesize{Autoformer} &
     \footnotesize{S4D} &
     \footnotesize{R-DANN} &
     \footnotesize{RDC} & \footnotesize{LSTM\_S+T}\\
		\toprule
		\multirow{2}*{$C\!\rightarrow\! S$} & RMSE & $\bm{0.0051}$ & 
		$0.0095$&
		$0.0144$& 
		$0.0076$& 
		$0.0164$& 
		$0.0109$& $0.0098$ & $0.0086$ &
		$0.0223$&
            $0.0118$& 
            $0.0147$ \\
% 		\cline{2-9}
		~ & MAE & $\bm{0.0448}$ &$0.0689$& $0.0905$& 
		$0.0633$& 
		$0.0942$& 
		$0.0733$& $0.0687$ & $0.0653$ &  
		$0.1131$& 
        $0.0765$& $0.0836$ \\
		\midrule
		\multirow{2}*{$C\!\rightarrow\! W$} & RMSE & $\bm{0.0122}$ & 
		$0.0173$&
		$0.0273$& 
		$0.0125$& 
		$0.0185$&
		$0.0222$& $0.0263$ & $0.0219$ &
	  $0.0316$  & $0.0224$  &$0.0201$\\
% 		\cline{2-9}
		~ & MAE & $\bm{0.0675}$ &$0.0918$& $0.1239$& 
		$0.0825$& 
		$0.0991$& 
		$0.0852$& $0.1070$ & $0.0853$ & 
		$0.1331$ & $0.1159$ &$0.0998$\\
		\midrule
		\multirow{2}*{$C\!\rightarrow\! D$} & RMSE & $\bm{0.0117}$ & 
		$0.0183$&
		$0.0275$& 
		$0.0150$& 
		$0.0327$& 
		$0.0191$& $0.0129$ & $0.0169$ & 
		$0.0382$& 
            $0.0194$&
            $0.0159$\\
% 		\cline{2-9}
		~ & MAE & $\bm{0.0710}$ &$0.0967$& $0.1238$& 
		$0.0781$& 
		$0.1428$& 
		$0.1024$& $0.0775$ & $0.0869$ & 
		$0.1488$& 
            $0.1044$&
            $0.0735$\\
		\midrule
		\multirow{2}*{$S\!\rightarrow\! C$} & RMSE & $\bm{0.0064}$ & 
		$0.0196$&
		$0.0256$& 
		$0.0099$&
		$0.0236$& 
		$0.0231$& $0.0465$ & $0.0112$ & 
		$0.0295$ & $0.0276$ & $0.0243$\\
% 		\cline{2-9}
		~ & MAE & $\bm{0.0402}$ & 
		$0.1019$&$0.1150$& 
		$0.0752$& 
		$0.1048$& 
		$0.1119$& $0.1590$ & $0.0610$ &
		$0.1237$ & $0.1164$&$0.1107$\\
		\midrule
		\multirow{2}*{$S\!\rightarrow\! W$} & RMSE & $\bm{0.0122}$ &
		$0.0268$&
		$0.0241$& 
		$0.0140$&
		$0.0282$&
		$0.0328$& $0.0396$ & $0.0211$ &
		$0.0361$ &$0.0287$ & $0.0433$\\
% 		\cline{2-9}
		~ & MAE & $\bm{0.0675}$ &
		$0.1213$& $0.1108$& 
		$0.0870$&
		$0.1268$&
		$0.1313$& $0.0983$ & $0.0838$ & 
		$0.1328$ &$0.1174$ &$0.1444$\\
		\midrule
		\multirow{2}*{$S\!\rightarrow \!D$} & RMSE & $\bm{0.0134}$ &
		$0.0185$&
		$0.0224$& 
		$0.0150$&
		$0.0270$&
		$0.0316$& $0.0167$ & $0.0182$ & 
		$0.0352$ &$0.0219$ & $0.0324$\\
% 		\cline{2-9}
		~ & MAE & $\bm{0.0785}$ &
		$0.0985$& $0.1088$&
		$0.0942$&
		$0.1305$&
		$0.1340$& $0.0911$ & $0.0878$ &
		$0.1396$ & $0.1052$ &$0.1315$\\
		\midrule
		\multirow{2}*{$W\!\rightarrow\! C$} & RMSE & $\bm{0.0052}$ &
		$0.0479$&
		$0.0224$&
		$0.0090$&
		$0.0134$&
		$0.0150$& $0.0163$ & $0.0089$ &
		$0.0181$ & $0.0146$&$0.0131$\\
% 		\cline{2-9}
		~ & MAE & $\bm{0.0297}$ & 
		$0.0654$&$0.1060$&
		$0.0676$& 
		$0.0790$& 
		$0.0831$& $0.0850$ & $0.0534$ &
		$0.0935$ & $0.0812$&$0.0761$\\
		\midrule
		\multirow{2}*{$W\!\rightarrow\! S$} & RMSE & $\bm{0.0047}$ &
		$0.0080$&
		$0.0172$& 
		$0.0102$&
		$0.0118$&
		$0.0121$& $0.0094$ & $0.0067$ &
		$0.0191$ &$0.0152$ & $0.0283$\\
% 		\cline{2-9}
		~ & MAE & $\bm{0.0412}$ & 
		$0.0643$&$0.0979$&
		$0.0752$&
		$0.0703$&
		$0.0795$& $0.0706$ & $0.0557$ &
		$0.0960$ &$0.0810$ &$0.0971$\\
		\midrule
		\multirow{2}*{$W\!\rightarrow\! D$} & RMSE & $\bm{0.0118}$ &
		$0.0157$&
		$0.0272$&
		$0.0137$&
		$0.0150$&
		$0.0182$& $0.0158$ & $0.0203$ &
		$0.0185$ &$0.0143$ & $0.0184$\\
% 		\cline{2-9}
		~ & MAE & $\bm{0.0712}$ & 
		$0.0921$&$0.1252$&
		$0.0863$&
		$0.0866$&
		$0.0978$& $0.0881$ & $0.0912$ &
		$0.0993$ &$0.0842$ &$0.0961$\\
		\midrule
		\multirow{2}*{$D\!\rightarrow\! C$} & RMSE & $\bm{0.0052}$ &
		$0.0105$&
		$0.0153$&
		$0.0086$&
		$0.0120$&
		$0.0161$& $0.0192$ & $0.0137$ &
		$0.0169$ &$0.0157$ & $0.0165$\\
% 		\cline{2-9}
		~ & MAE & $\bm{0.0309}$ & 
		$0.0644$&$0.0850$&
		$0.0672$&
		$0.0751$&
		$0.0863$& $0.0975$ & $0.0751$ &
		$0.0911$ &$0.0855$ &$0.0876$\\
		\midrule
		\multirow{2}*{$D\!\rightarrow\! S$} & RMSE & $\bm{0.0048}$ &
		$0.0066$&
		$0.0097$&
		$0.0084$&
		$0.0087$&
		$0.0141$& $0.0070$ & $0.0127$ &
		$0.0170$ &$0.0135$ & $0.0131$\\
% 		\cline{2-9}
		~ & MAE & $\bm{0.0415}$ &
		$0.0561$& $0.0709$&
		$0.0667$&
		$0.0672$&
		$0.0890$& $0.0574$ & $0.0772$ &
		$0.0895$ &$0.0790$ &$0.0794$\\
		\midrule
		\multirow{2}*{$D\!\rightarrow\! W$} & RMSE & $\bm{0.0113}$ &
		$0.0190$&
		$0.0203$&
		$0.0114$&
		$0.0246$&
		$0.0232$& $0.0273$ & $0.0254$ & 
		$0.0235$ &$0.0152$ &$0.0207$ \\
% 		\cline{2-9}
		~ & MAE & $\bm{0.0606}$ &
		$0.0986$& $0.1008$& 
		$0.0783$&
		$0.1256$&
		$0.1117$& $0.1240$ & $0.0972$ &
		$0.1135$ &$0.0826$ &$0.1001$\\
		\midrule
		\multirow{2}*{Average} & RMSE & $\bm{0.0086}$ & 
		$0.0181$&
		$0.0211$&
		$0.0113$& 
		$0.0194$&
		$0.0199$& $0.0206$ & $0.0155$ &
		$0.0255$ & $0.0185$ &$0.0217$ \\
% 		\cline{2-9}
		~ & MAE & $\bm{0.0537}$ &
		$0.0849$& $0.1049$&
		$0.0766$& 
		$0.1002$& 
		$0.0988$& $0.0937$ & $0.0767  $ &
		$0.1157$ &$0.0941$ &$0.0983$\\
		\bottomrule
	\end{tabular}
\end{table*}

\subsection{Result on Air Quality Forecasting}

In this section, we will report the experimental results on real-world data. The experiment results on the air quality forecasting are shown in Table 4. We conduct the Wilcoxon signed-rank test on the reported accuracies, our method significantly outperforms the baselines, with a p-value threshold of 0.05. According to the experiment results, we can find that the methods based on relationship modeling (GCA and SASA) outperform most of the baselines. And Autoformer also performs well since it captures the periodicity information. Moreover, the proposed GCA method outperforms the SASA and LIRR by a remarkable margin, which proves that the Granger Causality has more advantages than associative structures in modeling time-series data with time lags. However, we also find that the improvements of some domain adaptation tasks such as $G\rightarrow T$, $G\rightarrow B$, and $G\rightarrow S$ are not so remarkable. This is because the domain Guangzhou (G) contains many missing value data, which makes it difficult to reconstruct the real Granger-Causality.

\subsection{Result on Human Motion Foresting}
We also evaluate our method of human motion forecasting, another popular real-world time-series task. The experiment results are shown in Table 5. We conduct the Wilcoxon signed-rank test on the reported accuracies, our method significantly outperforms the baselines, with a p-value threshold of 0.05. Compared with the air quality forecasting dataset, the improvement of our method is even larger, possibly for the following reasons: (1) The human skeleton structure is a very sparse causal structure, which is why the proposed GCA method and SASA can outperform the other methods. What's more, our method can well remove the side effect of superfluous variables. (2) The human motion forecasting dataset contains more variables than the air quality forecasting dataset. (3) Different from the previous tasks that predict only one dimension time series, multivariate time series need to be predicted simultaneously, and it is difficult to achieve for the other baselines. 

\subsection{\textcolor{black}{Results on PPG-DaLiA dataset}}
\textcolor{black}{Finally, we analyze the experimental results of the PPG-DaLiA datasets, which are shown in Table 6. We conduct the Wilcoxon signed-rank test on the reported accuracies, our method significantly outperforms the baselines, with a p-value threshold of 0.05. According to the experiment results, we can find that the proposed GCA model still achieves the best performance. And the recently proposed DAF model also achieves comparative results, this is because the DAF model employs the domain-shared attention mechanism, which can capture the latent causal process and obtain the ideal performance. Moreover, we also find that our model performs better when we take the ``Sitting'' and ``Cycling'' as target domains. This is because the heart rating is influenced by not only the body signals collected by the sensors but also some unobserved variables like emotions and external environments. And the heart rating might be easily influenced by the body signal under the activities of ``Sitting'' and ``Cycling'', so we can obtain better prediction accuracy with the body signals.}

% 我们的效果依然是最好的。
% sitting 和 cycling的结果更好，因为心跳不仅仅由于动作影响，更容易受到周围隐变量的影响，所以可以解释为什么cycling的效果（受到动作比较打），而driving和working更容易受到隐变量，所以预测比较差。但是

\subsection{Ablation Study and Visualization}

In order to evaluate the modules of the proposed method, we devise the following model variants.
\begin{itemize}[leftmargin=*]
  \item \textbf{GCA-r:} In order to evaluate the effectiveness of the causal graph alignment regularization, we further devise a variant that removes the Granger causality Discrepancy regularization term $\mathcal{L}_r$.
  \item \textcolor{black}{\textbf{GCA-e:} In the case where the predicted results are univariate time-series and the Granger-causal structure is desired for interpretability, we aim to evaluate the effectiveness of the extra strengthen term, so we devise a variant that removes the strengthen term $\mathcal{L}_e$.}
  \item \textcolor{black}{\textbf{GCA-s:} In the case where the predicted results are univariate time-series and the Granger-causal structure is not necessary to be obtained accurately, we only predict and minimize the target dimension time-series, \textcolor{black}{.e.g, minimize $MSE(\widehat{\bm{z}}_{t+1}^d, \bm{z}_t^d)$ instead of $MSE(\widehat{\bm{z}}_{t+1},\bm{z}_t)$}}
  \item \textcolor{black}{\textbf{GCA-$\alpha$:} To evaluate the effectiveness of the trainable structural domain latent variables $\alpha^S,\alpha^T$, we remove the $\alpha^*$ from the standard GCA and add an extra domain of simulation dataset with different time lags.}
%   \item \textbf{GCA-s} In order to evaluate the effectiveness of domain-sensitive module, we remove the domain-sensitive module.
\end{itemize}

\subsubsection{The study of the Effectiveness of Extra strengthen term.}
\begin{figure}[t]
\label{fig:strengthen}
	\centering
	\includegraphics[width=0.9\columnwidth]{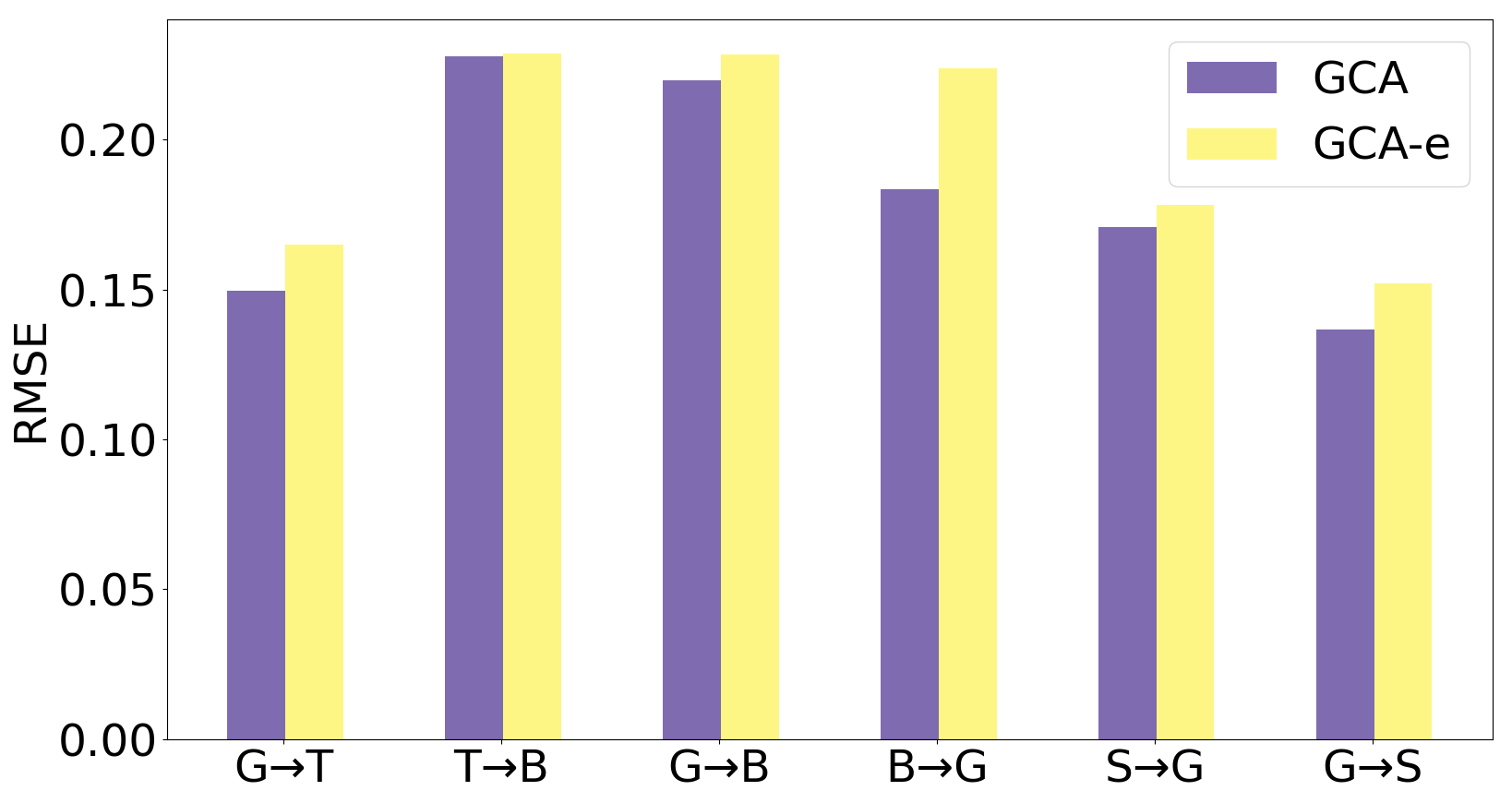}
	\caption{The Sample of three domains in the air quality forecasting dataset, we can find that the data are totally different, but they share the same summary graph. \textit{(best view in color.)}}
	\label{fig:data}
\end{figure}

Essentially, the proposed method is an autoregressive model, so all the variables are taken into consideration, which might lead to the performance degradation of the prediction performance on the target variable. In order to address this issue, we use an extra strengthen loss. As shown in Figure 9, the red and green bars respectively denotes the standard GCA method and the GCA-e method. We can find that (1) that the performance of the proposed GCA method is better than GCA-e, which reflects the effectiveness of the extra strengthen term, (2) that furthermore, the performance of GCA-e is comparable with the other baselines, which shows the stability of our method.

\begin{figure}[t]
\label{fig:alignment}
	\centering
	\includegraphics[width=0.9\columnwidth]{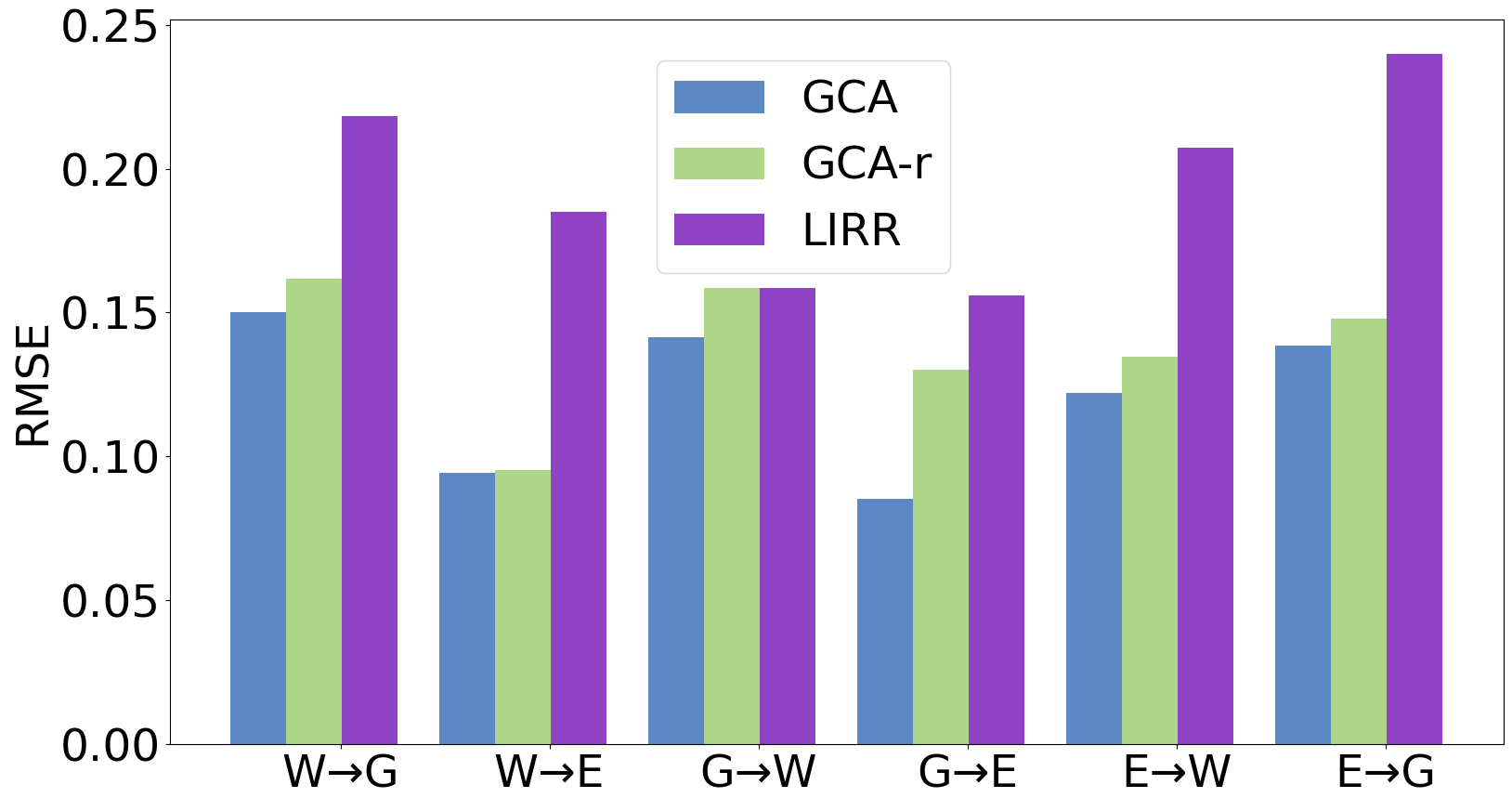}
	\caption{The Sample of three domains in the air quality forecasting dataset, we can find that the data are totally different, but they share the same summary graph. \textit{(best view in color.)}}
	\label{fig:gca_r}
\end{figure}

\subsubsection{The study of the Effectiveness of the Granger Causality Alignment.} To evaluate the effectiveness of the Granger Causality alignment, we remove the Granger Causality Discrepancy regularization term $\mathcal{L}_r$ and devise the variant model named GCA-r. According to the experiment result shown in Figure \ref{fig:gca_r}, we can find that the performance of GCA-r is almost better than that of LIRR. This is because the GCA-r model can leverage the source domain data and the limited labeled target domain data to reconstruct and predict the future value, which shows that the inference process with Granger Causality can exclude the obstruction of redundant information. We also see that the performance of GCA-r is slightly lower than that of GCA, which is because the small size of target domain data and the difference between the source and the target domain lead to the diversity and the inaccuracy Granger Causality of the target domain, further leading to the degeneration of the experiment results. In fact, The restriction of the discrepancy of the Granger Causality plays a key role in transferring the domain-invariant module from the source to the target domain.

\textcolor{black}{\subsubsection{Comparing the Multivariate Prediction and the Univariate Prediction}}
\textcolor{black}{In the scenario of only forecasting the object univariate time-series without interpretability, we devise the \textbf{GCA-s}, and the experiment results on the Air Quality Forecasting dataset are shown in Table 7. The experiment results suggest the following observations: \textcolor{black}{(1) The experiment results of all the model variants outperform the other baselines, which also proves the advantages of our method.} (2) Compared the experiment result between \textbf{GAC-e} and \textbf{GCA-s}, we can find that \textbf{GCA-s} performs better than \textbf{GCA-e}, which shows that simultaneously optimizing the other variables will hinder achieving the best results. (3) Compared to the experiment results between \textbf{GCA} and \textbf{GCA-s}, we can find that the experiment results of these two methods are quite comparable, which reflects that the extra strengthen loss can mitigate the side-effect of simultaneously regressing all variables.}

\textcolor{black}{
\begin{table}
% \tiny
% \renewcommand\arraystretch{1.2}
	\centering
	\label{tab:gca-s}
	\caption{\textcolor{black}{The MAE and RMSE on simulated datasets for the different model variants. The values presented are averaged over 5 replicated with different random seeds. The standard deviation is in the subscript.}}
	\scalebox{0.8}{
	\begin{tabular}{c|c|ccc}
	    \toprule
	    Task & Metric & GCA & GCA-s & GCA-e\\
		\midrule
		\multirow{2}*{$B\rightarrow T$} & RMSE & $0.1684_{\pm 0.0100}$ & $\bm{0.1663_{\pm 0.0036}}$ & $0.1685_{\pm 0.0080}$ \\
% 		\cline{2-9}
		~ & MAE & $0.2669_{\pm 0.0127}$ & $\bm{0.2532_{\pm 0.0009}}$ & $0.2682_{\pm 0.0082}$ \\
		\midrule
		\multirow{2}*{$G\rightarrow T$} & RMSE & $0.1497_{\pm 0.0128}$ & $\bm{0.1452_{\pm 0.0063}}$ & $0.1500_{\pm 0.0103}$ \\
% 		\cline{2-9}
		~ & MAE & $0.2618_{\pm 0.0155}$ & $\bm{0.2436_{\pm 0.0047}}$ & $0.2629_{\pm 0.0130}$\\
		\midrule
		\multirow{2}*{$S\rightarrow T$} & RMSE & $\bm{0.1673_{\pm 0.0040}}$ & $\bm{0.1670_{\pm 0.0043}}$ & $0.1714_{\pm 0.0064}$ \\
% 		\cline{2-9}
		~ & MAE & $\bm{0.2639_{\pm 0.0040}}$ & $\bm{0.2626_{\pm 0.0077}}$ & $0.2671_{\pm 0.0088}$\\
		\midrule
		\multirow{2}*{$T\rightarrow B$} & RMSE & $\bm{0.2276_{\pm 0.0055}}$ & $\bm{0.2276_{\pm 0.0050}}$ & $0.2337_{\pm 0.0084}$ \\
% 		\cline{2-9}
		~ & MAE & $\bm{0.2841_{\pm 0.0087}}$ & $\bm{0.2817_{\pm 0.0065}}$ & $0.2846_{\pm 0.0106}$\\
		\midrule
		\multirow{2}*{$G\rightarrow B$} & RMSE & $\bm{0.2169_{\pm 0.0056}}$ & $\bm{0.2150_{\pm 0.0068}}$ & $0.2273_{\pm 0.0135}$ \\
% 		\cline{2-9}
		~ & MAE & $0.2844_{\pm 0.0058}$ & $\bm{0.2731_{\pm 0.0087}}$ & $0.2922_{\pm 0.0125}$\\
		\midrule
		\multirow{2}*{$S\rightarrow B$} & RMSE & $\bm{0.2312_{\pm 0.0053}}$ & $\bm{0.2314_{\pm 0.2801}}$ & $0.2318_{\pm 0.0046}$ \\
% 		\cline{2-9}
		~ & MAE & $\bm{0.2802_{\pm 0.0056}}$ & $\bm{0.2901_{\pm 0.0043}}$ & $0.2807_{\pm 0.0050}$\\
		\midrule
		\multirow{2}*{$B\rightarrow G$} & RMSE & $0.1996_{\pm 0.0046}$ & $\bm{0.1969_{\pm 0.0060}}$ & $0.2013_{\pm 0.0052}$ \\
% 		\cline{2-9}
		~ & MAE & $\bm{0.2999_{\pm 0.0114}}$ & $\bm{0.2983_{\pm 0.0075}}$ & $0.3024_{\pm 0.0128}$\\
		\midrule
		\multirow{2}*{$T\rightarrow G$} & RMSE & $0.1845_{\pm 0.0110}$ & $\bm{0.1796_{\pm 0.0135}}$ & $0.1891_{\pm 0.0136}$ \\
% 		\cline{2-9}
		~ & MAE & $\bm{0.2819_{\pm 0.0166}}$ & $\bm{0.2808_{\pm 0.0075}}$ & $0.2991_{\pm 0.0195}$\\
		\midrule
		\multirow{2}*{$S\rightarrow G$} & RMSE & $\bm{0.1769_{\pm 0.0112}}$ & $\bm{0.1780_{\pm 0.0032}}$ & $0.1871_{\pm 0.0119}$ \\
% 		\cline{2-9}
		~ & MAE & $\bm{0.2887_{\pm 0.0138}}$ & $\bm{0.2880_{\pm 0.0059}}$ & $0.2923_{\pm 0.0106}$\\
		\midrule
		\multirow{2}*{$B\rightarrow S$} & RMSE & $\bm{0.1558_{\pm 0.0058}}$ & $\bm{0.1568_{\pm 0.0032}}$ & $0.1570_{\pm 0.0045}$ \\
% 		\cline{2-9}
		~ & MAE & $\bm{0.2640_{\pm 0.0095}}$ & $\bm{0.2640_{\pm 0.0062}}$ & $0.2652_{\pm 0.0151}$\\
		\midrule
		\multirow{2}*{$T\rightarrow S$} & RMSE & $0.1536_{\pm 0.0172}$ & $\bm{0.1440_{\pm 0.0065}}$ & $0.1543_{\pm 0.0174}$ \\
% 		\cline{2-9}
		~ & MAE & $0.2640_{\pm 0.0025}$ & $\bm{0.2532_{\pm 0.0095}}$ & $0.2683_{\pm 0.0231}$\\
		\midrule
		\multirow{2}*{$G\rightarrow S$} & RMSE & $\bm{0.1360_{\pm 0.0121}}$ & $\bm{0.1356_{\pm 0.0035}}$ & $0.1359_{\pm 0.0082}$ \\
% 		\cline{2-9}
		~ & MAE & $0.2574_{\pm 0.0107}$ & $\bm{0.2545_{\pm 0.0085}}$ & $0.2625_{\pm 0.0106}$\\
		\midrule
		\multirow{2}*{Average} & RMSE & $\bm{0.1810}$ & $\bm{0.1783}$ & $0.1839$ \\
% 		\cline{2-9}
		~ & MAE & $\bm{0.2750}$ & $\bm{0.2704}$ & $0.2788$\\
		\bottomrule
	\end{tabular}}
\end{table}}

\subsubsection{\textcolor{black}{The Study of the Effectiveness of the Trainable Structural Domain Latent Variables $\alpha^*$}}

\textcolor{black}{To evaluate the effectiveness of the trainable structural domain latent variables $\alpha^*$, we further add an extra domain with a time lag of 10 to the simulation data. Experiment results are shown in Figures 11 and 12. According to the experiment results, we can draw the following conclusions: 1) According to the experiment results in Figure 11, we can find that the performance of the standard GCA is better than the GCA-$\alpha$, which means that the trainable structural domain latent variables can mitigate the influence of different time lags. 2) According to Figure 11, we also find that the standard GCA achieves a higher score of AUPRC than the GCA-$\alpha$, which shows that the better reconstruction of causal structures can benefit the model performance and the reconstruction of causal structures also draws advantages from the trainable structural domain latent variables. }

\begin{figure}[t]
% \label{fig:alpha_rmse}
	\centering
	\includegraphics[width=0.9\columnwidth]{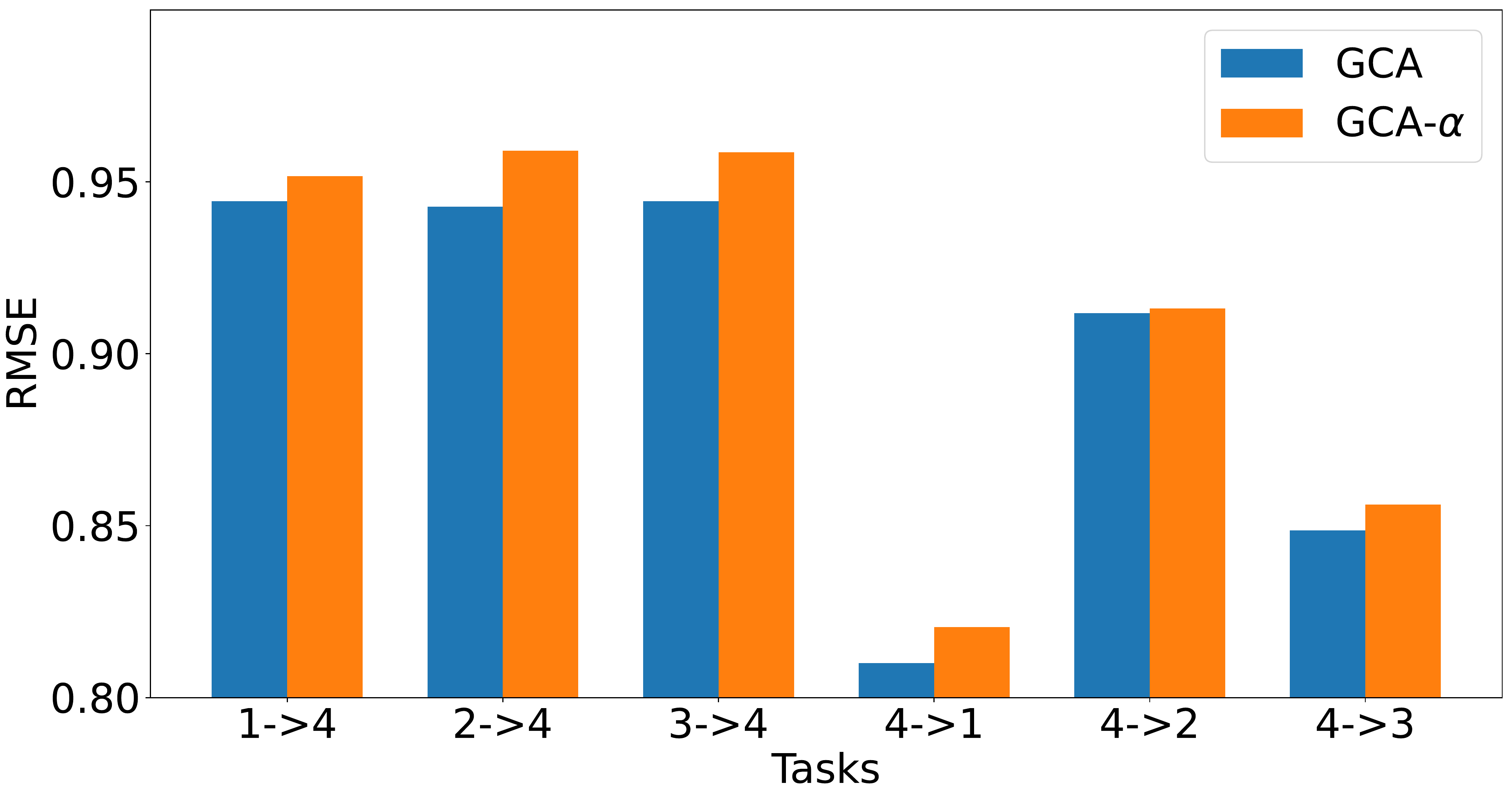}
	\caption{\textcolor{black}{The RMSE results of the standard GCA and the GCA-$\alpha$ between an extra domain with a time lag of 10 and domains with a time lag of 5. The lower the RMSE, the better the model performance.}}
	\label{fig:gca_alpha_rmse}
\end{figure}
\begin{figure}[t]
\label{fig:alpha_auprc}
	\centering
	\includegraphics[width=0.9\columnwidth]{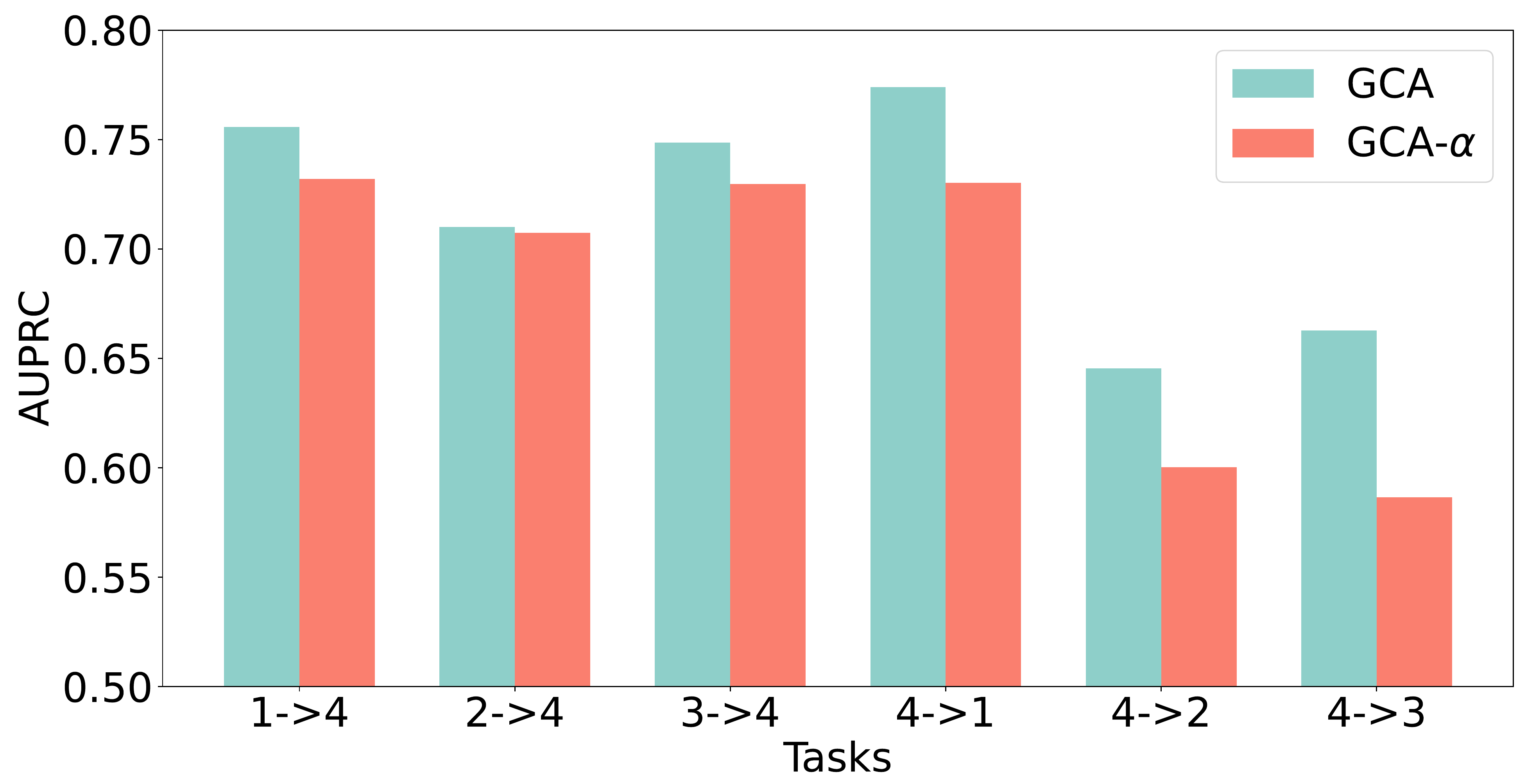}
	\caption{\textcolor{black}{The AUPRC results of the standard GCA and the GCA-$\alpha$ between an extra domain with a time lag of 10 and domains with a time lag of 5. The higher the RMSE, the better the model performance.}}
	\label{fig:gca_alpha_auprc}
\end{figure}

\subsubsection{\textcolor{black}{Sensitive Analysis of Hyper-parameters}} 
% \begin{figure}[t]
% \centering
% \label{fig:sensitive_alpha}
% \subfigure{
% \begin{minipage}[t]{\linewidth}
% \centering
% \includegraphics[width=0.95\columnwidth]{fig/sensitive_1}
% %\caption{fig1}
% \end{minipage}%
% }
% \centering
% \caption{\textcolor{black}{The RMSE performance of the proposed GCA model with different valuses $\lambda$ with the range $(0.0 \sim 0.14)$}.}
% \end{figure}

\begin{figure}[t]
\centering
\label{fig:exp1}
\subfigure[\textcolor{black}{The RMSE performance of the proposed GCA model with different valuses $\lambda$ with the range $(0.0 \sim 0.14)$}.]{
\begin{minipage}[t]{\linewidth}
\centering
\includegraphics[width=1.0\columnwidth]{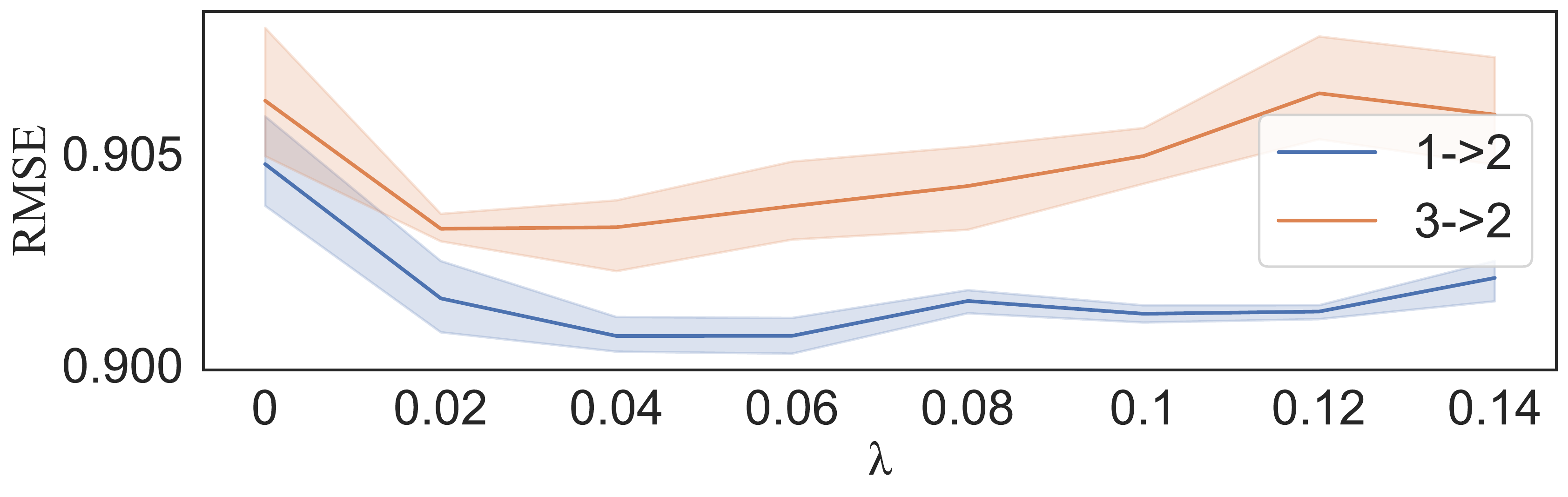}
%\caption{fig1}
\end{minipage}%
}%

\subfigure[\textcolor{black}{The RMSE performance of the proposed GCA on the Air Quality Forecasting dataset with different values of $\gamma$ with the range $0.0 \sim 3.5$}]{
\begin{minipage}[t]{\linewidth}
% \centering
\includegraphics[width=\columnwidth]{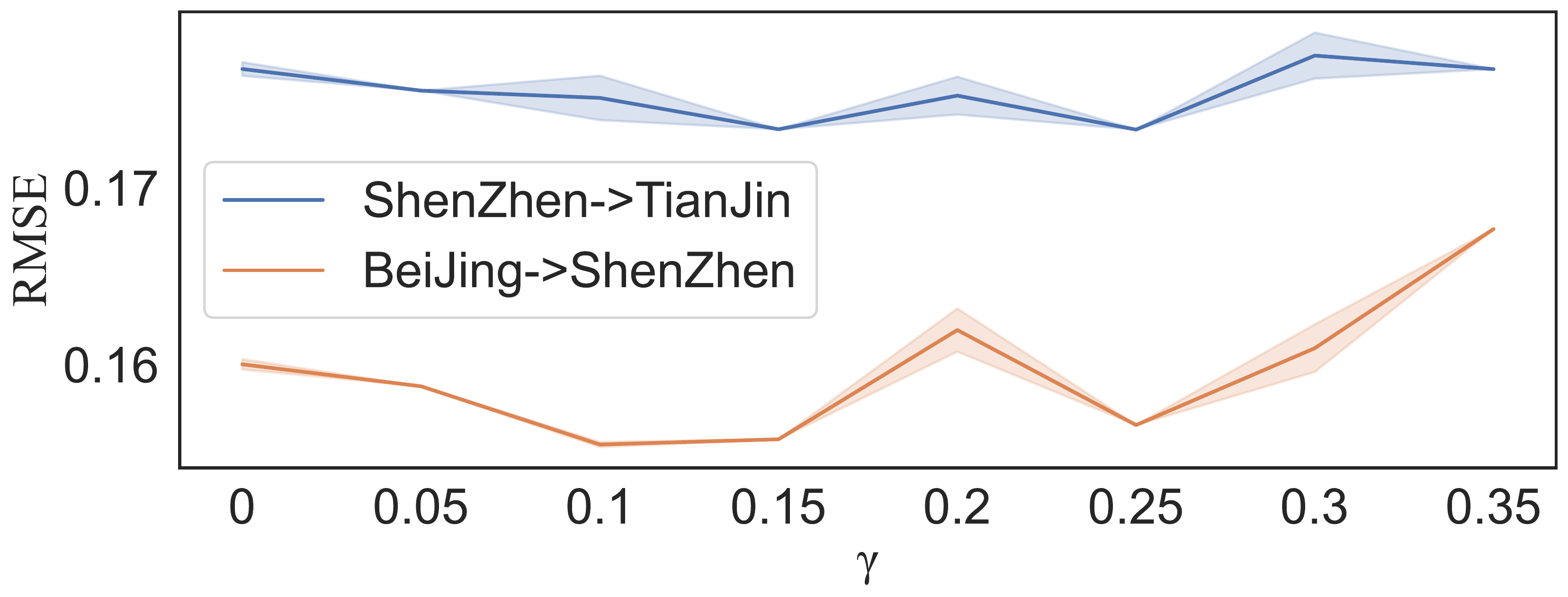}
%\caption{fig2}
\end{minipage}%
}%
\centering
\caption{\textcolor{black}{The experiment results of the sensitive analysis of the hyper-parameters $\lambda$ and $\gamma$. (a) The RMSE results of the standard GCA with different values of $\lambda$. (b) The RMSE results of the standard GCA with different values of $\gamma$.}}
\end{figure}

\textcolor{black}{To explore the importance of $\lambda$ and $\gamma$ in Equation (\ref{equ:loss2}), we conduct experiments for the sensitivity of these hyper-parameters. }

\textcolor{black}{First, to explore the importance of the $\lambda$, we try different values of $\lambda$, and the experiment results of task $1\rightarrow 2$ and $3\rightarrow 2$ are shown in Figure 13 (a). According to the experiment results, we can draw the following conclusions: 1) Our GCA method outperforms the other compared methods under different values of $alpha$, which reflects the stability of our method. 2) When the value of $\alpha$ is around $0.4$, the proposed GCA achieve the best results, this is because an appropriate hyperparameter can lead to reasonable sparsity, which benefits the reconstruction of Granger Causal structures and finally results in the ideal model performance.}

\textcolor{black}{We further evaluate how the different value of $\gamma$ affect the model performance. To achieve this, we try different values of $\gamma$, and the experiment results of task ShenZhen $\rightarrow$ TianJin and Beijing $\rightarrow$ ShenZhen are shown in Figure 13 (b). According to the experiment results, we can draw the following conclusions: 1) Compared with the experiment results of GCA ($\gamma=0$) and GCA ($\gamma=0.15$), we can find that the GCA model performs better with $\gamma=0.15$, which shows that the effectiveness of the causal structure discrepancy regularization term. 2) Note that the GCA-$\alpha$ still performs better than several baselines, this is because the frameworks of GCA can reconstruct the rough structures and the trainable domain-sensitive latent variables can also benefit the model prediction. 3) With the increase of the value of $\gamma$, the performance drops. This is because the too-heavy penalty from the causal structure discrepancy regularization term might lead to the loss of domain-specific information, which further influences the model generalization. }

\subsubsection{\textcolor{black}{The study of missing value processing.}}

\textcolor{black}{We further provide the robustness analysis of missing data. In detail, we randomly mask the data of the Heart Rate Forecasting dataset with the ratio of $0\%$, $10\%$, and $20\%$. Experiment results on several transfer tasks are shown in Figures \ref{fig:missing_rmse} and \ref{fig:missing_mae}. According to the experiment results, we can find that the performance of GCA drops as the increment of the missing ratio. This is because it is hard the reconstruct the accurate Granger causal structures with the incomplete dataset. And the suboptimal Granger causal structures further lead to worse forecasting performance.} 

\begin{figure}[t]
\label{fig:alpha_auprc}
	\centering
	\includegraphics[width=0.9\columnwidth]{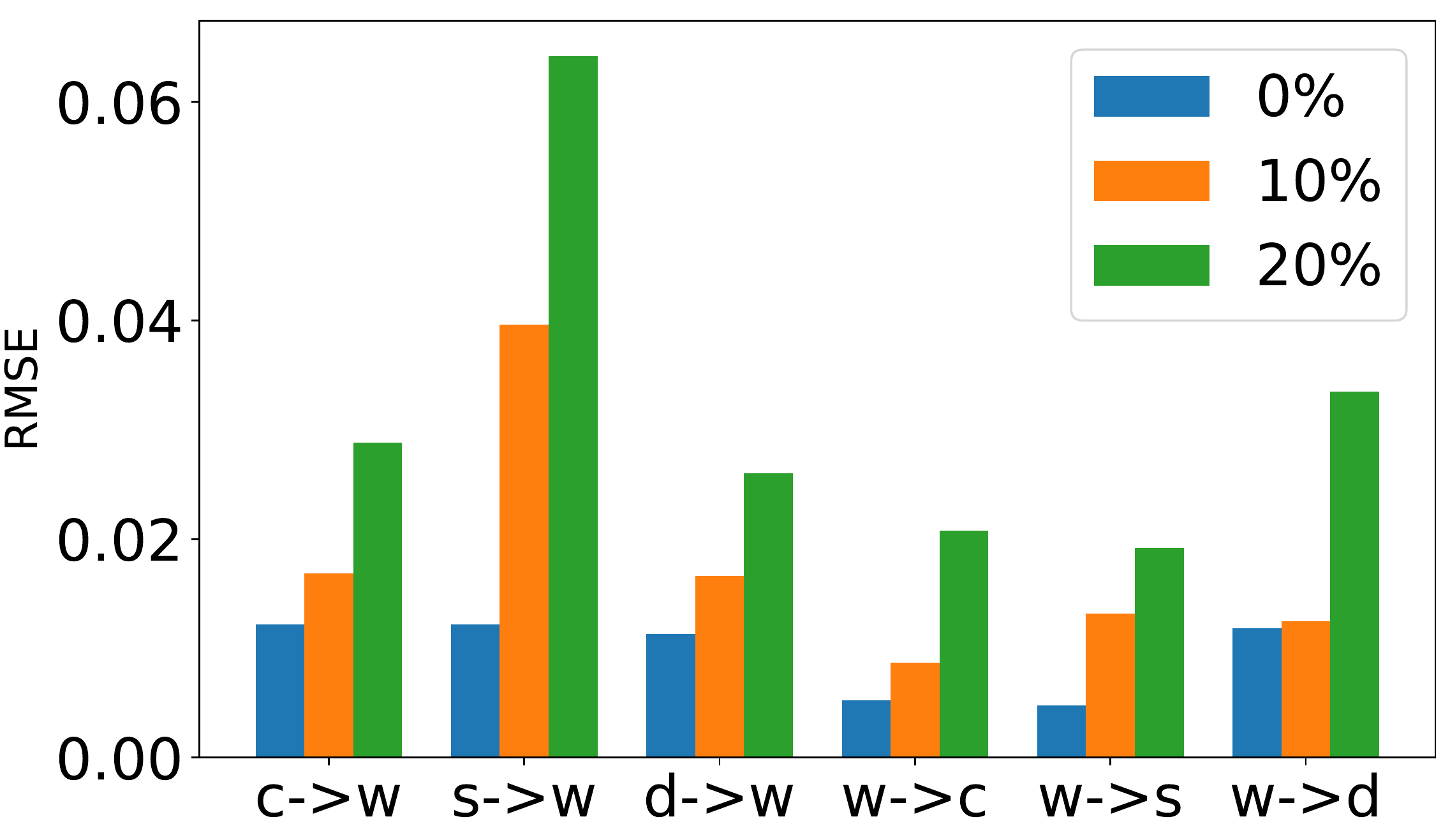}
	\caption{\textcolor{black}{The RMSE results of the proposed GCA on Heart Rate Forecasting dataset with different the missing ratio of $0\%$, $10\%$, $20\%$.}}
	\label{fig:missing_rmse}
\end{figure}

\begin{figure}[t]
% \label{fig:alpha_rmse}
	\centering
	\includegraphics[width=0.9\columnwidth]{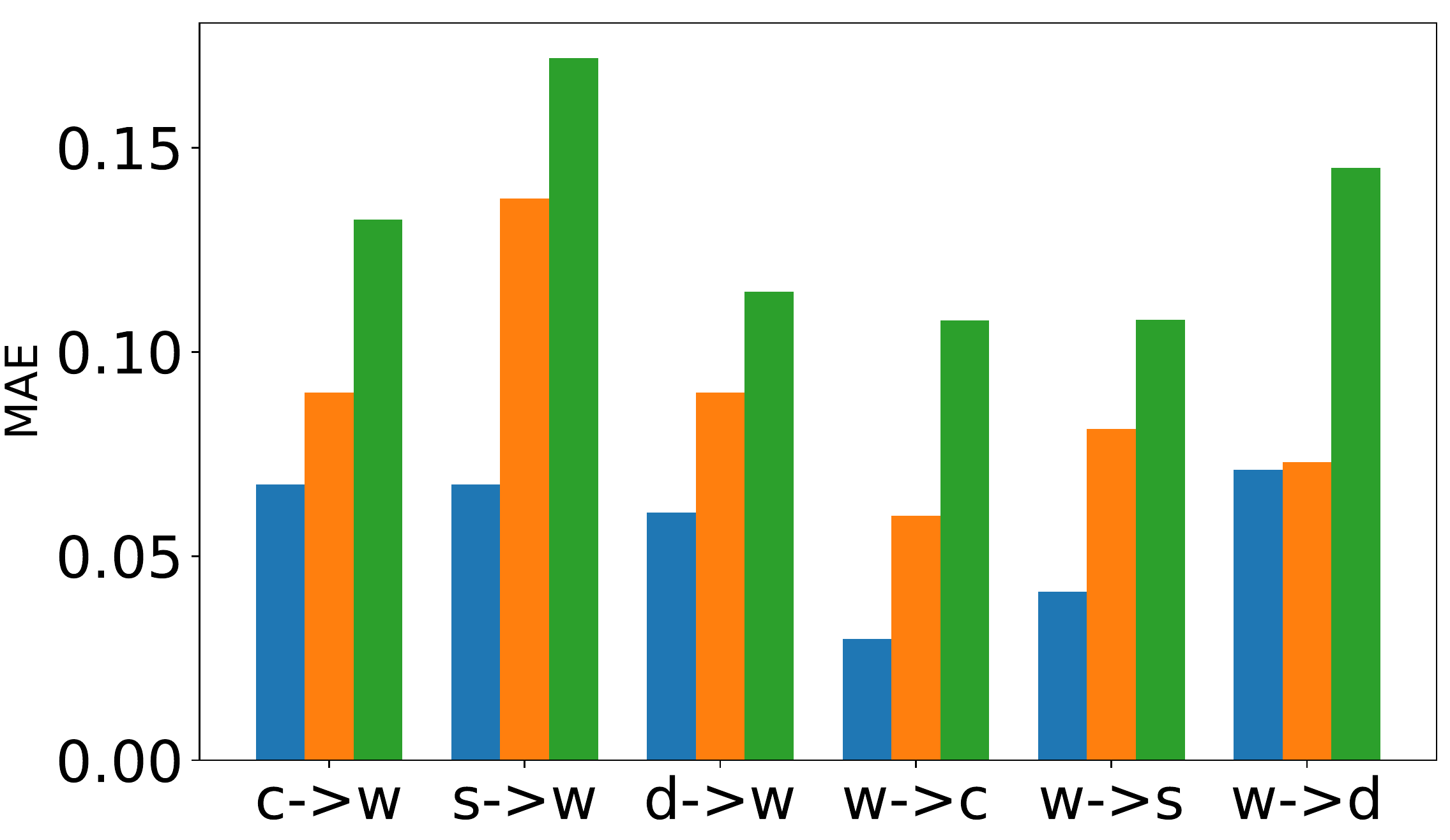}
	\caption{\textcolor{black}{The MAE results of the proposed GCA on Heart Rate Forecasting dataset with different the missing ratios of $0\%$, $10\%$, $20\%$.}}
	\label{fig:missing_mae}
\end{figure}
% \subsubsection{\textcolor{black}{The study of different prediction length.}}

\subsubsection{Visualization and Interpretability}
% \begin{figure}[t]
% \centering
% \subfigure[TianJin]{
% \begin{minipage}[t]{0.45\linewidth}
% \centering
% \includegraphics[width=\columnwidth]{fig/domain1}
% %\caption{fig1}
% \end{minipage}%
% }%
% \subfigure[BeiJing]{
% \begin{minipage}[t]{0.45\linewidth}
% \centering
% \includegraphics[width=\columnwidth]{fig/domain2}
% %\caption{fig2}
% \end{minipage}%
% }%
% \centering
% \label{fig:air_causal}
% \caption{The illustration of visualization of the summary Granger-Causal structure of TianJin $\rightarrow$ BeiJing. The Blue blocks denote the Granger-causal relationship among different variables.}
% \end{figure}

\begin{figure}[t]
\label{fig:alignment}
	\centering
	\includegraphics[width=\columnwidth]{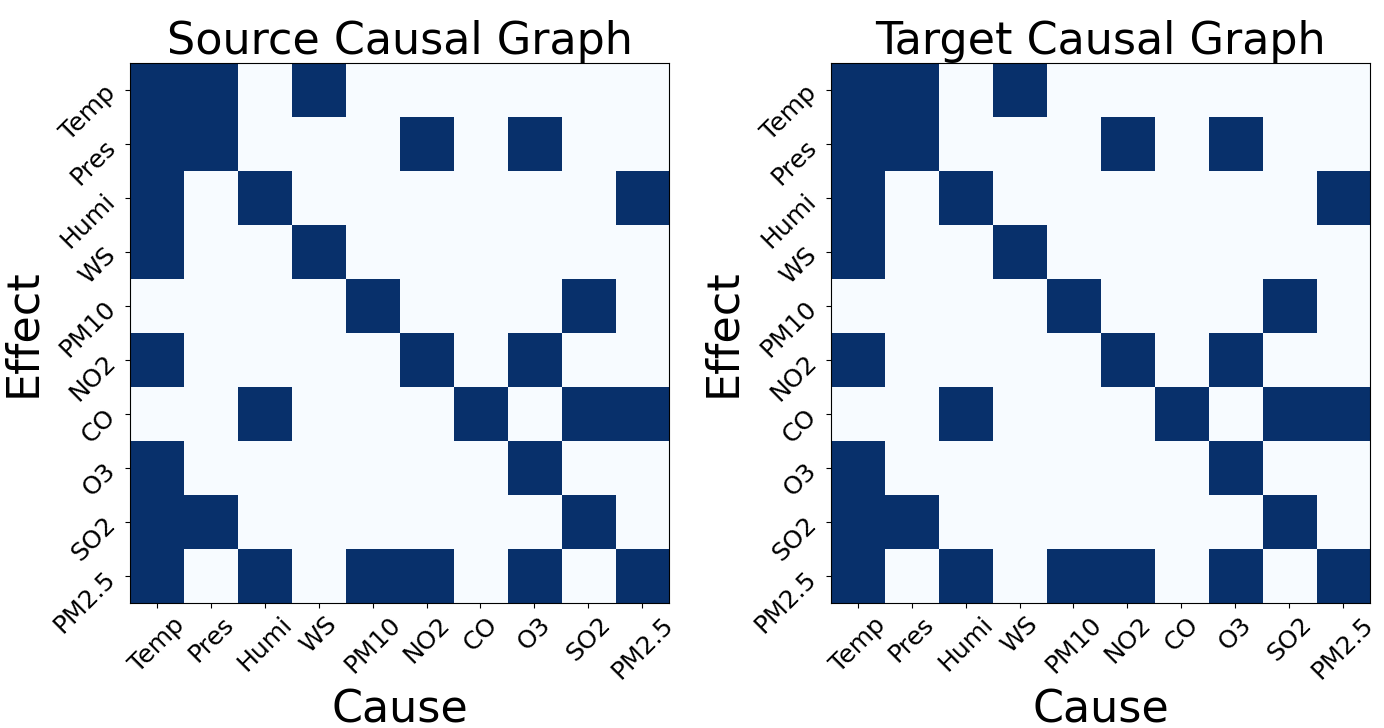}
	\caption{The illustration of visualization of the summary Granger-Causal structure of TianJin $\rightarrow$ BeiJing. The Blue blocks denote the Granger-causal relationship among different variables. Note that `Temp', `Pres', `Humi' and `WS' denote respectively `Temperature', `Pressure', `Humidity' and `Wind Speed'.}
	\label{fig:air_causal}
\end{figure}

\textbf{Visualization of the Air Quality Forecasting dataset.} To further investigate our approach, we extract the Granger-Causal structures and perform visualization over the air quality forecasting dataset. Figure 14 shows the Granger-causal structures whose task is TianJin $\rightarrow$ BeiJing. According to the visualization, we can find that the Granger-Causal relations are very sparse, which means that the PM2.5 is generated by a similar causal mechanism. The aforementioned visualization also provides some interpretability to a certain extent. For example, in Figure 14, we can find that the ``humidity'' and ``PM10\_Concentration'' simultaneously have an influence on the generation of PM2.5, while the other variables like ``pressure'' and ``Co'' have little influence on PM2.5. The results also provide the suggestion that we can mitigate the effect of PM2.5 by controlling the release of PM10, NO, and O3.

\begin{figure}[t]
\label{fig:motion}
\subfigure[GCA]{
\begin{minipage}[t]{0.48\linewidth}
\flushleft
\includegraphics[width=\columnwidth]{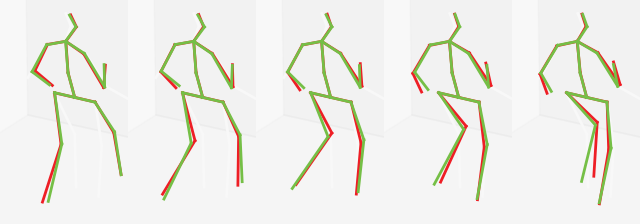}
%\caption{fig1}
\end{minipage}%
}%
\subfigure[SASA]{
\begin{minipage}[t]{0.48\linewidth}
\flushleft
\includegraphics[width=\columnwidth]{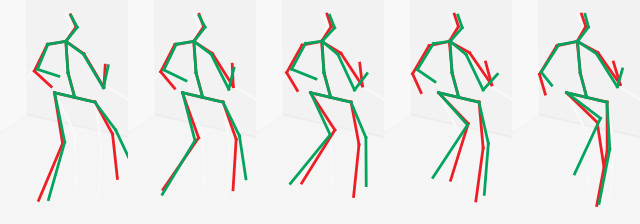}
%\caption{fig2}
\end{minipage}%
}%

\subfigure[LSTM\_S+T]{
\begin{minipage}[t]{0.48\linewidth}
\flushleft
\includegraphics[width=\columnwidth]{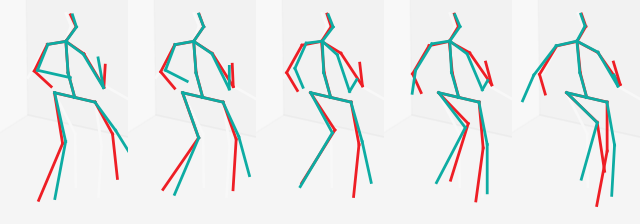}
%\caption{fig1}
\end{minipage}%
}%
\subfigure[R-DANN]{
\begin{minipage}[t]{0.48\linewidth}
\flushleft
\includegraphics[width=\columnwidth]{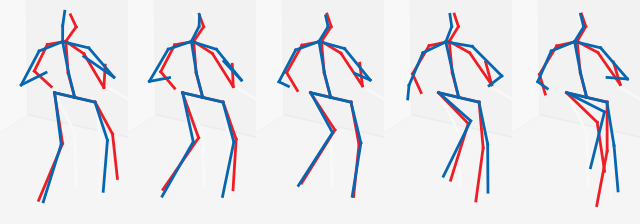}
%\caption{fig2}
\end{minipage}%
}%

\subfigure[RDC]{
\begin{minipage}[t]{0.48\linewidth}
\flushleft
\includegraphics[width=\columnwidth]{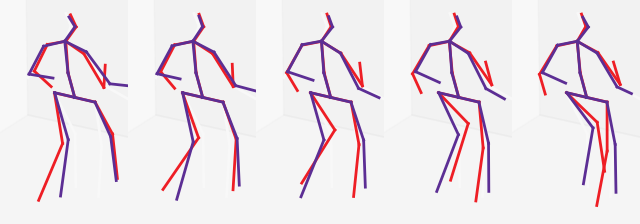}
%\caption{fig1}
\end{minipage}%
}%
\subfigure[VRADA]{
\begin{minipage}[t]{0.48\linewidth}
\flushleft
\includegraphics[width=\columnwidth]{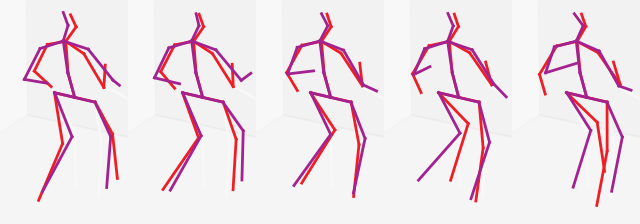}
%\caption{fig2}
\end{minipage}%
}%
\caption{The illustration of visualization of the human Motion forecasting of the task Guesting$\rightarrow$Walking. The red lines denote the ground truth motion and the other colors denote the different methods.}
\end{figure}

\textbf{Visualization of the Human Motion Forecasting dataset. }We further conduct visualization on the human motion forecasting dataset, and Figure 15 illustrates the walking motion 5 steps later. We can find that the visualized 3D motions of the proposed GCA method are consistent with the ground truth motions, while the other baselines are not. For some joints like the legs and the elbows, our method markedly surpasses the other methods, reflecting that the Granger Causality is beneficial to capture the relationship between different joints and results in better prediction. 

Note that the proposed method aims to leverage the Granger Causality to address the time-series domain adaptation problem, so we do not follow the same experiment setting of the standard human motion estimation works that forecast a long-term human motion but just predict the future a few steps (e.g., 5 steps).

% \subsubsection{Computational Complexity Analysis}
% Different from the previous LSTM based method that extracts the feature on multivariate level, the architecture of the proposed GCA method works on the univariate level. Given $p$ univariate time-series data whose time lag is $k$, the proposed GCA method only requires $p$ small MLP for recurrent Granger Causality reconstruction and $k \times p$ small MLPs for Granger Causality inference networks. Since the hidden size for each univariate is small, we set the dimension of the hidden state to 3 in practice, and the proposed method does not have very large number of parameters. However, it is incontestable that the time complexity of our method compares unfavorably with the LSTM based methods, since we need to recurrently reconstruct Granger-causal structure. Fortunately, the time lag is small in practice (usually less than 10). Therefore, given a small time lag $k$ and a large $p$, the time cost of the proposed GCA method is acceptable and is even faster than the sparse associative structure alignment model.

\section{Conclusion}\label{conclusion}
In this paper, we present a Granger Causality Alignment method for semi-supervised time-series forecasting. By simultaneously reconstructing the Granger Causality and forecasting future values based on the Granger-causal structure, we well transfer domain knowledge via the \textcolor{black}{recurrent Granger Causality reconstruction and the domain-sensitive Granger Causality based prediction. We further leverage the Granger Causality discrepancy regularization to relax the causal conditional shift assumption.}. The proposed method not only addresses the semi-supervised domain adaptation problem on the time-series forecasting task but also provides insightful interpretability. 
\textcolor{black}{However, the proposed method needs to take the time lag as a prior or hyperparameter when we apply our method to real-world datasets. Therefore, exploring how to learn the Granger Causality with adaptive time lag would be a future direction.}
% 需要提前知道lag，之后需要改进成lag自适应

\ifCLASSOPTIONcompsoc
  \section*{Acknowledgments}
\else
  \section*{Acknowledgment}
\fi

The authors would like to thank Jie Qiao and Jiahao Li from the Guangdong University of Technology for their help and support in this work.

\bibliographystyle{IEEEtran}
\bibliography{gca_v3}

% \include{supplementary}

\ifCLASSOPTIONcaptionsoff
  \newpage
\fi

\clearpage

\end{document}

\title{Appendix for ``Transferable Time-Series Forecasting under Causal Conditional Shift''}

\author{Zijian Li,~\IEEEmembership{Member,~IEEE,}
        Ruichu Cai*,~\IEEEmembership{Member,~IEEE,}
        Tom Z. J. Fu, Zhifeng Hao and Kun Zhang% <-this % stops a space
}
\onecolumn

\markboth{Journal of \LaTeX\ Class Files,~Vol.~14, No.~8, August~2015}%
{Shell \MakeLowercase{\textit{et al.}}: Bare Demo of IEEEtran.cls for Computer Society Journals}
\maketitle

\IEEEdisplaynontitleabstractindextext

% \fi
% \IEEEpeerreviewmaketitle
% \parskip
% \vspace{-1cm}
% \appendices
% \section{\textcolor{black}{The derivation of the logarithm of the joint likelihood $\ln P(z_t|z_{t-1}, z_{t-2}, \cdots, z_{t-k})$}}\label{ELBO}
{ 
\tableofcontents
}

\section{Proof}
\subsection{Proof of Theorem 1}
\begin{theorem}
    \textcolor{black}{Suppose that the maximum lag is $k$, $\ln P(z_t|z_{t-1}, z_{t-2}, \cdots, z_{t-k})$ can be derived as follows:}
\begin{equation}
\small
\color{black}
\label{equ:variation}
\begin{split}
    lnP(&\bm{z}_t|\bm{z}_{t-1}, \bm{z}_{t-2}, \cdots, \bm{z}_{t-k}) =\\&\mathbb{E}_{Q(A_1|\cdot)},\cdots\mathbb{E}_{Q(A_k|\cdot)}\left[\ln P(\bm{z}_t|\bm{z}_{t-1},\cdots,\bm{z}_{t-k},A_1,\cdots, A_k) + \ln \frac{P(A_1|\cdot)\prod_{j=2}^k P(A_j|\cdot)}{Q(A_1|\cdot)\prod_{j=2}^k Q(A_j|\cdot)}\right] + \sum_{j=1}^k{D_{KL}(Q(A_j|\cdot)||P(A_j|\cdot)})
\end{split}
\end{equation}
\textcolor{black}{in which $P(A_j|\cdot)=P(A_j|A_1,\cdots,A_{j-1},\bm{z}_{t-1},\cdots,\bm{z}_{t-k})$ and $Q(A_j|\cdot) = Q(A_j|A_1,\cdots,A_{j-1},\bm{z}_{t-1},\cdots,\bm{z}_{t-k})$ denote the prior distribution and the approximated distribution, respectively.}
\end{theorem}

\begin{proof}
\textcolor{black}{The proof of the derivation of the logarithm of the joint likelihood $\ln P(z_t|z_{t-1}, z_{t-2}, \cdots, z_{t-k})$ is composed of two steps, which are shown as follows.} 

\noindent \textcolor{black}{First, We factorize the conditional distribution according to the Bayes theorem.}
\begin{equation}
\small
\color{black}
\begin{split}
& \ln P(\bm{z}_t|\bm{z}_{t-1},\bm{z}_{t-2}, \cdots, \bm{z}_{t-k})=\ln \frac{P(\bm{z}_t, A_1, A_2, \cdots, A_k|\bm{z}_{t-1},\bm{z}_{t-2},\cdots,\bm{z}_{t-k})}{P(A_1, A_2, \cdots, A_k|\bm{z}_{t},\bm{z}_{t-1},\cdots,\bm{z}_{t-k})}\\
=& \ln P(\bm{z}_t|\bm{z}_{t-1}, \cdots, \bm{z}_{t-k},A_1, \cdots, A_k) + ln\frac{P(A_1|\bm{z}_{t-1}, \cdots, \bm{z}_{t-k})\prod_{j=2}^{k}P(A_j|A_1,\cdots,A_{j-1},\bm{z}_{t-1},\cdots,\bm{z}_{t-k})}{P(A_1|\bm{z}_{t},\!\cdots,\bm{z}_{t-k})\prod_{j=2}^{k}P(A_j|A_1,\cdots,A_{j-1},\bm{z}_t,\bm{z}_{t-1},\cdots,\bm{z}_{t-k})}
\end{split}
\end{equation}

\noindent \textcolor{black}{Second, we add the expectation operator on both sides of the equation and reformalize the equation as follows:}
\begin{equation}
\color{black}
\begin{split}
\small
\ln{P(\bm{z}_t|\bm{z}_{t-1},\bm{z}_{t-2},\cdots,\bm{z}_{t-k})}=\mathbb{E}_{Q(A_1|\bm{z}_{t-1},\cdots,\bm{z}_{t-k})}\cdots \mathbb{E}_{Q(A_k|\bm{z}_{t-1},\cdots,\bm{z}_{t-k},A_1,\cdots,A_{k-1})}\ln{P(\bm{z}_t|\bm{z}_{t-1},\bm{z}_{t-2},\cdots,\bm{z}_{t-k})}.
\end{split}
\end{equation}
\textcolor{black}{By combining Equation (4) with Equation (3), we can obtain Equation (5) as follows:}
\begin{equation}
% \nonumber
\color{black}
\small
\begin{split}
% & \ln P(\bm{z}_t|\bm{z}_{t-1},\bm{z}_{t-2}, \cdots, \bm{z}_{t-k})\\=&
&\mathbb{E}_{Q(A_1|\bm{z}_{t-1},\cdots,\bm{z}_{t-k})}\cdots \mathbb{E}_{Q(A_k|\bm{z}_{t-1},\cdots,\bm{z}_{t-k},A_1,\cdots,A_{k-1})}\ln{P(\bm{z}_t|\bm{z}_{t-1},\bm{z}_{t-2},\cdots,\bm{z}_{t-k})}\\
=&\mathbb{E}_{Q(A_1|\bm{z}_{t-1},\cdots,\bm{z}_{t-k})}\cdots \mathbb{E}_{Q(A_k|\bm{z}_{t-1},\cdots,\bm{z}_{t-k},A_1,\cdots,A_{k-1})} \Bigg[\ln P(\bm{z}_t|\bm{z}_{t-1}, \cdots, \bm{z}_{t-k},A_1, \cdots, A_k).
\\&+ \left. \ln \frac{P(A_1|\bm{z}_{t-1}, \cdots,\bm{z}_{t-k})\prod_{j=2}^k P(A_j|A_1,\cdots,A_{j-1},\bm{z}_{t-1},\cdots,\bm{z}_{t-k})}{P(A_1|\bm{z}_{t},\cdots,\bm{z}_{t-k})\prod_{j=2}^k P(A_j|A_1, \cdots,A_{j-1},\bm{z}_{t},\cdots,\bm{z}_{t-k})} \right.
\\&+ \left.\ln \frac{Q(A_1|\bm{z}_{t-1},\cdots,\bm{z}_{t-k})\prod_{j-2}^k Q(A_j|A_1,\cdots,A_{j-1},\bm{z}_{t-1},\cdots,\bm{z}_{t-k})}{Q(A_1|\bm{z}_{t-1},\cdots,\bm{z}_{t-k})\prod_{j-2}^k Q(A_j|A_1,\cdots,A_{j-1},\bm{z}_{t-1},\cdots,\bm{z}_{t-k})}\right]\\
=&
\mathbb{E}_{Q(A_1|\bm{z}_{t-1},\cdots,\bm{z}_{t-k})}\cdots \mathbb{E}_{Q(A_k|\bm{z}_{t-1},\cdots,\bm{z}_{t-k},A_1,\cdots,A_{k-1})}\left[\ln P(\bm{z}_t|\bm{z}_{t-1}, \cdots, \bm{z}_{t-k},A_1, \cdots, A_k) \right. \\&+ \left. \ln \frac{P(A_1|\bm{z}_{t-1},\cdots,\bm{z}_{t-k})\prod_{j=2}^k P(A_j|A_1,\cdots,A_{j-1},\bm{z}_{t-1},\cdots,\bm{z}_{t-k})}{Q(A_1|\bm{z}_{t-1},\cdots,\bm{z}_{t-k})\prod_{j=2}^k Q(A_j|A_1,\cdots,A_{j-1},\bm{z}_{t-1},\cdots,\bm{z}_{t-k})} \right. 
\\&+ \left. \ln \frac{Q(A_1|\bm{z}_{t-1},\cdots,\bm{z}_{t-k})\prod_{j=2}^k Q(A_j|A_1,\cdots,A_{j-1},\bm{z}_{t-1},\cdots,\bm{z}_{t-k})}{P(A_1|\bm{z}_{t-1},\cdots,\bm{z}_{t-k})\prod_{j=2}^k P(A_j|A_1,\cdots,A_{j-1}, \bm{z}_t,\bm{z}_{t-1},\cdots,\bm{z}_{t-k})}  \right],
\\=&\mathbb{E}_{Q(A_1|\bm{z}_{t-1},\cdots,\bm{z}_{t-k})}\cdots \mathbb{E}_{Q(A_k|\bm{z}_{t-1},\cdots,\bm{z}_{t-k},A_1,\cdots,A_{k-1})}\left[\ln P(\bm{z}_t|\bm{z}_{t-1}, \cdots, \bm{z}_{t-k},A_1, \cdots, A_k) \right. \\&+ \left. \ln \frac{P(A_1|\bm{z}_{t-1},\cdots,\bm{z}_{t-k})\prod_{j=2}^k P(A_j|A_1,\cdots,A_{j-1},\bm{z}_{t-1},\cdots,\bm{z}_{t-k})}{Q(A_1|\bm{z}_{t-1},\cdots,\bm{z}_{t-k})\prod_{j=2}^k Q(A_j|A_1,\cdots,A_{j-1},\bm{z}_{t-1},\cdots,\bm{z}_{t-k})} \right] + \sum_{j=1}^k{D_{KL}(Q_j||P_j}),
\\=&\mathbb{E}_{Q(A_1|\cdot)},\cdots\mathbb{E}_{Q(A_k|\cdot)}\left[\ln P(\bm{z}_t|\bm{z}_{t-1},\cdots,\bm{z}_{t-k},A_1,\cdots, A_k) + \ln \frac{P(A_1|\cdot)\prod_{j=2}^k P(A_j|\cdot)}{Q(A_1|\cdot)\prod_{j=2}^k Q(A_j|\cdot)}\right] + \sum_{j=1}^k{D_{KL}(Q(A_j|\cdot)||P(A_j|\cdot)}),
\end{split}
\end{equation}

\textcolor{black}{in which $P(A_j|\cdot)=P(A_j|A_1,\cdots,A_{j-1},\bm{z}_{t-1},\cdots,\bm{z}_{t-k})$ and $Q(A_j|\cdot) = Q(A_j|A_1,\cdots,A_{j-1},\bm{z}_{t-1},\cdots,\bm{z}_{t-k})$.}

\end{proof}

\subsection{Proof of Corollary 1.1}
\begin{corollary}
\textcolor{black}{Assuming that the maximum lag is $k$, we can model the logarithm of the joint likelihood $\ln P(z_t|z_{t-1}, z_{t-2}, \cdots, z_{t-k})$ by maximizing the evidence lower bound (\textit{ELBO}) as shown as follows:}
\begin{equation}
\color{black}
\begin{split}
\mathcal{L}_{\textit{ELBO}}=&\mathbb{E}_{Q(A_1|\bm{z}_{t-1},\cdots,\bm{z}_{t-k})}\cdots \mathbb{E}_{Q(A_k|\bm{z}_{t-1},\cdots,\bm{z}_{t-k},A_1,\cdots,A_{k-1})}\ln P(\bm{z}_t|\bm{z}_{t-1}, \cdots, \bm{z}_{t-k},A_1, \cdots, A_k) \\& - D_{KL}(Q(A_1|\cdot)||P(A_1|\cdot)) -\sum_{j=2}^k \mathbb{E}_{Q(A_1|\cdot)},\cdots\mathbb{E}_{Q(A_k|\cdot)} D_{KL}(Q(A_j|\cdot)||P(A_j|\cdot)),
\end{split}
\end{equation}
\begin{proof}
\textcolor{black}{According to Theorem 1, the logarithm of the joint likelihood can be derived into the expectation term and the KL-divergence term $\sum_{j=1}^k{D_{KL}(Q(A_j|\cdot)||P(A_j|\cdot)})$. Since the  $\sum_{j=1}^k{D_{KL}(Q(A_j|\cdot)||P(A_j|\cdot)})\geq 0$, we can obtain:}
\begin{equation}
\color{black}
\begin{split}
    lnP(&\bm{z}_t|\bm{z}_{t-1}, \bm{z}_{t-2}, \cdots, \bm{z}_{t-k}) =\\&\mathbb{E}_{Q(A_1|\cdot)},\cdots\mathbb{E}_{Q(A_k|\cdot)}\left[\ln P(\bm{z}_t|\bm{z}_{t-1},\cdots,\bm{z}_{t-k},A_1,\cdots, A_k) + \ln \frac{P(A_1|\cdot)\prod_{j=2}^k P(A_j|\cdot)}{Q(A_1|\cdot)\prod_{j=2}^k Q(A_j|\cdot)}\right] + \sum_{j=1}^k{D_{KL}(Q(A_j|\cdot)||P(A_j|\cdot)}) \\\geq& \mathbb{E}_{Q(A_1|\cdot)},\cdots\mathbb{E}_{Q(A_k|\cdot)}\left[\ln P(\bm{z}_t|\bm{z}_{t-1},\cdots,\bm{z}_{t-k},A_1,\cdots, A_k) + \ln \frac{P(A_1|\cdot)\prod_{j=2}^k P(A_j|\cdot)}{Q(A_1|\cdot)\prod_{j=2}^k Q(A_j|\cdot)}\right]\\=&\mathbb{E}_{Q(A_1|\cdot)},\cdots\mathbb{E}_{Q(A_k|\cdot)}\ln P(\bm{z}_t|\bm{z}_{t-1}, \cdots, \bm{z}_{t-k},A_1, \cdots, A_k) - D_{KL}(Q(A_1|\cdot)||P(A_1|\cdot)) \\&-\sum_{j=2}^k \mathbb{E}_{Q(A_1|\cdot)},\cdots\mathbb{E}_{Q(A_k|\cdot)} D_{KL}(Q(A_j|\cdot)||P(A_j|\cdot))=\mathcal{L}_{\textit{ELBO}},
\end{split}
\end{equation}
\textcolor{black}{in which $P(A_j|\cdot)=P(A_j|A_1,\cdots,A_{j-1},\bm{z}_{t-1},\cdots,\bm{z}_{t-k})$ and $Q(A_j|\cdot) = Q(A_j|A_1,\cdots,A_{j-1},\bm{z}_{t-1},\cdots,\bm{z}_{t-k})$. Hence, by maximizing Equation (5), we can model the logarithm of the joint likelihood $\ln P(z_t|z_{t-1}, z_{t-2}, \cdots, z_{t-k})$. }
\end{proof}
\end{corollary}

\clearpage
\section{\textcolor{black}{Realworld Dataset Descriptions}}

\subsection{\textcolor{black}{Air Quality Forecasting Dataset}}
\begin{wraptable}{r}{7cm}
\vspace{-1.5cm}
\caption{Features of Air Quality Forecast Dataset.}
	\centering
    \begin{tabular}{c|c}
		\toprule
		Feature Type & Feature Name  \\
		\midrule
		\multirow{5}{*}{Air quality} & PM10\\
		~ & PM2.5\\
        ~ & NO2 Concentration\\
        ~ & CO Concentration\\
        ~ & O3 Concentration\\
        ~ & SO2 Concentration\\
        \midrule
        \multirow{6}*{Meteorology} & Weather\\
        ~ & Temperature\\
        ~ & Pressure\\
        ~ & Humidity\\
        ~ & Wind Speed\\
		\bottomrule
	\end{tabular}
 \vspace{-0.5cm}
	\label{tab:air_features}
\end{wraptable}

\textcolor{black}{The air quality forecast dataset \cite{10.1145/2783258.2788573} is collected in the Urban Air project \footnote{https://www.microsoft.com/en-us/research/project/urban-air/} from 2014/05/01 to 2015/04/30, which contains the air quality data, the meteorological data, and the weather forecast data, etc. This dataset contains four major Chinese cities: Beijing (B), Tianjin (T), Guangzhou (G), and Shenzhen (S). We employ air quality data as well as meteorological data to predict PM2.5. The details of the features in this data are shown in Table \ref{tab:air_features}. We use all air quality stations and take each city as a domain. We use this dataset because the sensors in smart city systems usually contain complex and compact causality. \textcolor{black}{Introducing causality into the prediction model will not only enhance the robustness of the model but also provide the interpretability of the PM2.5-producing reasons.}  Since the value range of this dataset varies across cities, we use Z-score Normalization to preprocess the data from each city. }

\subsection{\textcolor{black}{Human Motion Capture Dataset}}
\textcolor{black}{We further show the effectiveness of our method on the human motion capture dataset. Human motion has been studied in many fields, e.g., Granger Causality Inference \cite{9376668,zhang2017causal}, the linear and the nonlinear dynamical systems \cite{fox2014joint}. Many researchers use it to study the human motion capture dataset for Granger Causality since the relationships among the joints of humans are a natural causal structure. In this paper, we use the Human3.6M dataset \cite{ionescu2013human3} \footnote{http://vision.imar.ro/human3.6m/description.php}, one of the most popular benchmarks for human motion prediction \cite{barsoum2018hp}, for cross-domain human motion prediction. This dataset contains 15 motions and we choose three of them as three different domains: ``Walking'', ``Greeting'' and ``Eating'', the examples of this dataset are shown in Figure 1 (a). We choose 9 primary sensors which record the three-dimensional coordinate and the dimension of the processed dataset is 27. \textcolor{black}{In this dataset, we predict the value of all the dimensions, so it is a transferable multivariate forecasting problem.}} 

\subsection{\textcolor{black}{
PPG-DaLiA Dataset}}
\textcolor{black}{We also consider the PPG-DaLiA dataset,\footnote{https://archive.ics.uci.edu/ml/datasets/PPG-DaLiA} which is publicly available for PPG-based heart rate estimation. This multimodal dataset features physiological and motion data, recorded with both a wrist- and a chest-worn device from 15 volunteers while they perform a wide range of activities. Raw sensor data was recorded with two devices: a chest-worn device and a wrist-worn device. The chest-worn device provides an electrocardiogram, respiration, and three-axis acceleration. All signals are sampled at 700 Hz. The wrist-worn device provides the following sensor data: blood volume pulse, electrodermal activity, body temperature, and three-axis acceleration.}
\textcolor{black}{Though different activities lead to the variance of heart rate, how the body signals influence the heart rates usually follows the same causal mechanism, so we can employ this dataset to evaluate the performance of the proposed GCA model. To reduce the computation complexity, we randomly choose four of them as different domains: ``Cycling (C)'', ``Sitting (S)'', ``Working (W)'', and ``Driving (D)''. An example of this dataset is shown in Figure 1 (b) (c) and (d). We employ the data collected from all volunteers and use Z-score Normalization to process all the datasets for each domain, respectively.}
% \begin{figure}
% 	\centering
% \includegraphics[width=0.9\columnwidth]{./fig/heart_rate}
% 	\caption{\textcolor{black}{The Sample of the PPG-DaLiA dataset. The x-coordinate denotes the time and the y-coordinate denotes the heart rate. The red areas denote four different activities.}}
% 	\label{fig:heart_rate}
% \end{figure}

\begin{figure}[t]
\centering
\label{fig:human_motion}
\subfigure[PPG-DaLiA Dataset]{
\begin{minipage}[t]{0.4\linewidth}
\centering
\includegraphics[width=2.9in]{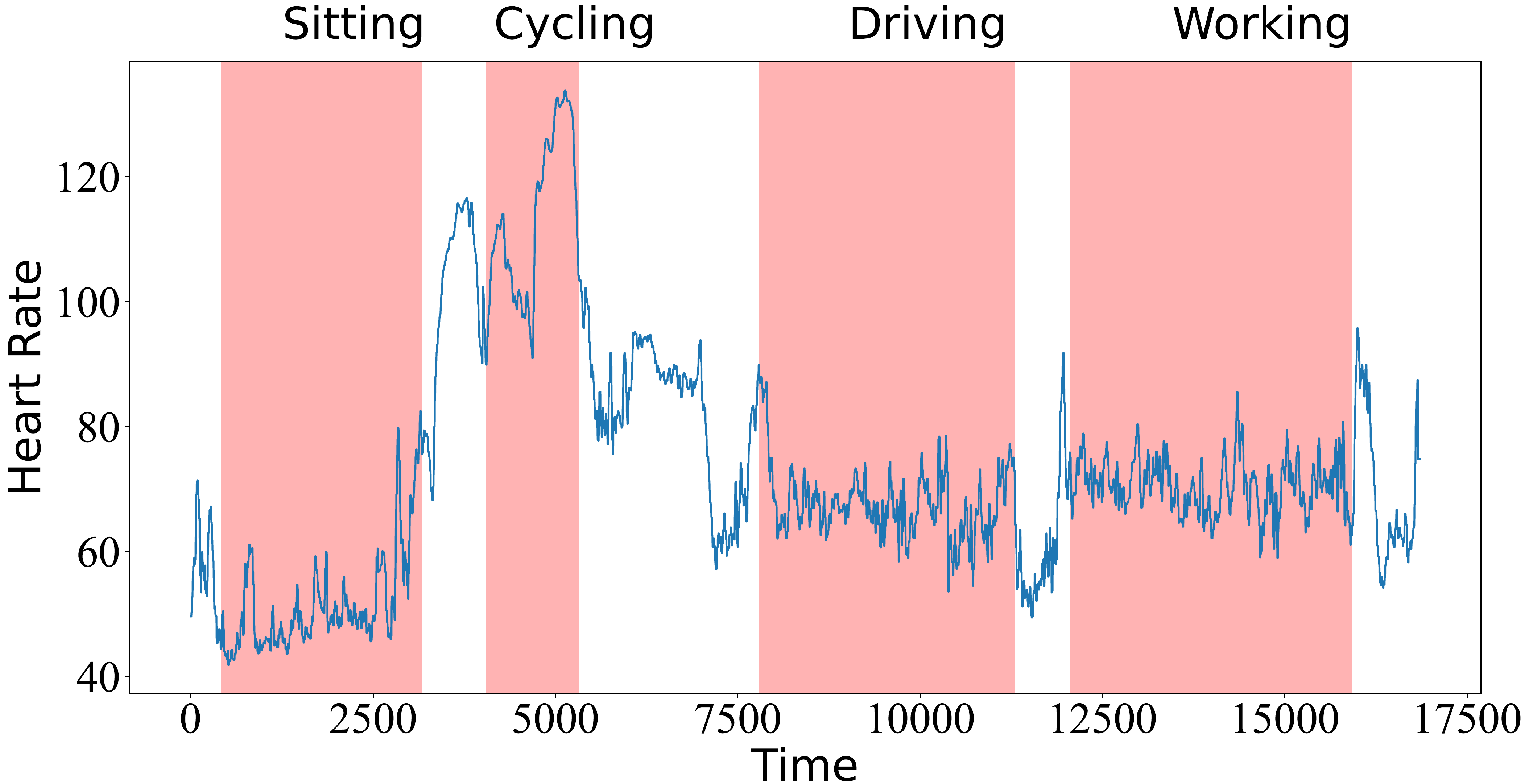}
%\caption{fig1}
\end{minipage}%
}%
\label{fig:human_motion}
\subfigure[Walking]{
\begin{minipage}[t]{0.18\linewidth}
\centering
\includegraphics[width=1.2in]{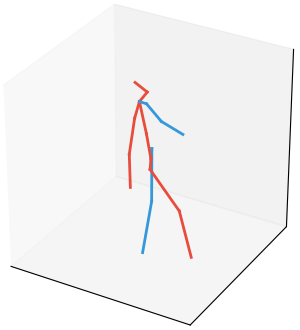}
%\caption{fig1}
\end{minipage}%
}%
\subfigure[Eating]{
\begin{minipage}[t]{0.18\linewidth}
\centering
\includegraphics[width=1.2in]{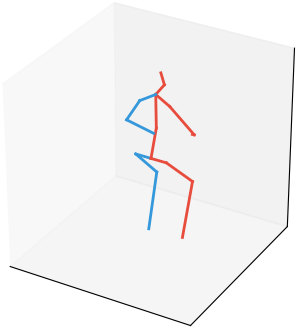}
%\caption{fig2}
\end{minipage}%
}%
\subfigure[Greeting]{
\begin{minipage}[t]{0.18\linewidth}
\centering
\includegraphics[width=1.2in]{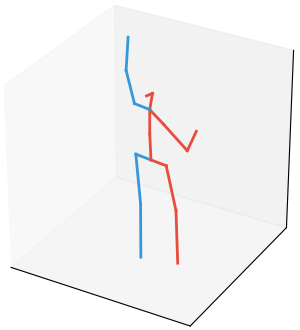}
%\caption{fig2}
\end{minipage}
}
\centering
\caption{\textcolor{black}{Illustration of the PPG-DaLiA dataset and human motion capture dataset. (a) denotes the Sample of the PPG-DaLiA dataset. The x-coordinate denotes the time and the y-coordinate denotes the heart rate. The red areas denote four different activities. (b), (c) and (d) denote the visualized examples of a frame in the human motion capture dataset. (b), (c) and (d) respective denote ``Walking'', ``Eating'' and ``Greeting''.}}
\end{figure}

\subsection{\textcolor{black}{Electricity Transformer Temperature Dataset}}
\textcolor{black}{The ETT (Electricity Transformer Temperature) \footnote{https://github.com/zhouhaoyi/ETDataset} is a crucial indicator in the electric power long-term deployment, so Zhou et.al \cite{zhou2021informer} collect 2-year data from two separated counties
in China, which consists of seven features,
including power load features and oil temperature. We employ the ETT-small subset which contains two stations. We consider each station as a domain and use Z-score Normalization to process the dataset in each domain.}

\subsection{\textcolor{black}{Electricity Load Diagrams Datasets}}
\textcolor{black}{The Electricity Load Diagrams 2011-2014 dataset \footnote{https://archive.ics.uci.edu/dataset/321/electricityloaddiagrams20112014} was created by Artur Trindade and shared on the UCI Machine Learning Repository. This comprehensive dataset captures the electricity consumption of 370 substations in Portugal from January 2011 to December 2014, with a sampling period of 15 minutes. We use Z-score Normalization to process the dataset and select 90 substations. Since this dataset is not devised for domain adaptation problem, we consider that the domain shift occurs in different seasons and divide them into four different domains based on the months: domain1 (January, February, March), domain2 (April, May, June), domain3 (July, August, September), and domain4 (October, November, December). We employ this dataset to evaluate the performance of the GCA model on high-dimensional multivariate time series datasets and aim to predict the electricity load of the substations.}

\subsection{\textcolor{black}{PEMS-BAY Dataset}}
\textcolor{black}{PEMS-BAY \footnote{https://pems.dot.ca.gov/} is a dataset of traffic speeds collected from the California Transportation Agencies (CalTrans) Performance Measurement System (PeMS) \cite{chen2001freeway}. It includes data collected from 325 sensors located throughout the Bay Area, covering a period of 6 months from January 1st, 2017 to May 31st, 2017 \cite{li2017diffusion}. The dataset provides detailed traffic information recorded at a frequency of every 5 minutes. we use Z-score Normalization to process the dataset and select 100 sensors. Since this dataset is not devised for domain adaptation problems, we consider that the domain shift occurs in different seasons and divide them into three domains based on months: domain1 (January), domain2 (February), and domain3 (March). Since the sensors are deployed in the traffic networks, the speed of the sensors is influenced causally. }

\section{\textcolor{black}{Other Experiment Results}}

\subsection{\textcolor{black}{Experiments Results on Electricity Transformer Dataset}}

\textcolor{black}{In this subsection, we show the experiment results on the Electricity Transformer Temperature dataset as shown in Table \ref{tab:ett}. According to the experiment results, we can find that the proposed GCA model achieves the ideal performance in most of the metrics. However, compared with other datasets, the improvement of our method is less obvious. This is because electricity power is influenced by many latent factors like weekdays, holidays, seasons, weather, and temperatures, which might result in redundant causal structures and further suboptimal performance.}

\begin{table*}[t]
\color{black}
	\centering
 \small
	\caption{The MAE and MSE on Electricity Transformer Temperature dataset for the baselines and the proposed method. }
 \label{tab:ett}
	\begin{tabular}{p{1.0cm}<{\centering}|p{0.85cm}<{\centering}|p{0.85cm}<{\centering} p{0.85cm}<{\centering} p{0.85cm}<{\centering} p{0.85cm}<{\centering} p{0.85cm}<{\centering} p{0.85cm}<{\centering} p{1.20cm}
    <{\centering} p{1.15cm}<{\centering} p{1.20cm}
    <{\centering} p{0.85cm}<{\centering} p{1.20cm}<{\centering}}
	    \toprule
	    \footnotesize{Task} & \footnotesize{Metric} & \footnotesize{GCA} & \footnotesize{DAF}& \footnotesize{LIRR}& \footnotesize{SASA}& \footnotesize{VRADA}& \footnotesize{CMTN}& \footnotesize{Autoformer} &
     \footnotesize{S4D} &
     \footnotesize{R-DANN} &
     \footnotesize{RDC} & \footnotesize{LSTM\_S+T}\\
		\midrule
		\multirow{2}*{$1\!\rightarrow \!2$} & RMSE & $\bm{0.1185}$ & $0.1217$ & $0.1493$ & $0.1332$ & $0.1391$ & $0.1327$ & $0.1219$ & $0.1937$ & $0.1489$ & $0.1380$ & $0.1513$\\
% 		\cline{2-9}
		~ & MAE & $0.2677$ & $\bm{0.2572}$ & $0.2827$ & $0.2646$ & $0.2799$ & $0.2767$ & $0.2699$ & $0.3321$ & $0.2860$ & $0.2789$ & $0.2807$\\
		\midrule
		\multirow{2}*{$2\!\rightarrow\! 1$} & RMSE & $\bm{0.0563}$ & $0.1058$ & $0.0990$ & $0.0866$ & $0.1067$ & $0.0840$ & $0.0765$ & $0.0889$ & $0.0729$ & $0.0544$ & $0.0803$\\
% 		\cline{2-9}
		~ & MAE & $0.1822$ & $0.2540$ & $0.2267$ & $0.1880$ & $0.2395$ & $\bm{0.1820}$ & $0.2104$ & $0.1996$ & $0.2090$ & $0.1713$ & $0.2114$\\
		\midrule
		\multirow{2}*{Average} & RMSE & $\bm{0.0874}$ & $0.1137$ & $0.1241$ & $0.1099$ & $0.1229$ & $0.1083$ & $0.0992$ & $0.1413$ & $0.1109$ & $0.0962$ & $0.1158$\\
% 		\cline{2-9}
		~ & MAE & $\bm{0.2249}$ & $0.2556$ & $0.2547$ & $0.2263$ & $0.2597$ & $0.2293$ & $0.2401$ & $0.2658$ & $0.2475$ & $0.2251$ & $0.2460$\\
		\bottomrule
	\end{tabular}
\end{table*}

\subsection{\textcolor{black}{Experiments Results on Electricity Load Diagrams Datasets}}
\textcolor{black}{We also evaluate the GCA model on high-dimension time-series datasets like Electricity Load Diagrams, which are shown in Table \ref{tab:eld}. According to the experiment results, the GCA model achieves comparable performance in several tasks like $1 \rightarrow 2$ and $1 \rightarrow 3$. However, we can also find that our method does not perform well in some tasks like $1 \rightarrow 4$,  $3 \rightarrow 2$, and $4 \rightarrow 1$. This is because the Electricity Load Diagrams datasets do not follow a stable data generation process mentioned in Section 3. And each substation (variable) in this dataset tends to be independent of each other.}

\begin{table*}
\color{black}
\small
	\centering
	\caption{The MAE and MSE on Electricity Load Diagrams Datasets for the baselines and the proposed method. 
 % The values presented are averaged over 5 replicated with different random seeds.
 }
 \label{tab:eld}
	\begin{tabular}{p{1.0cm}<{\centering}|p{0.85cm}<{\centering}|p{0.85cm}<{\centering} p{0.85cm}<{\centering} p{0.85cm}<{\centering} p{0.85cm}<{\centering} p{0.85cm}<{\centering} p{0.85cm}<{\centering} p{1.20cm}
    <{\centering} p{1.15cm}<{\centering} p{1.20cm}
    <{\centering} p{0.85cm}<{\centering} p{1.20cm}<{\centering}}
	    \toprule
	    \footnotesize{Task} & \footnotesize{Metric} & \footnotesize{GCA} & \footnotesize{DAF}& \footnotesize{LIRR}& \footnotesize{SASA}& \footnotesize{VRADA}& \footnotesize{CMTN}& \footnotesize{Autoformer} &
     \footnotesize{S4D} &
     \footnotesize{R-DANN} &
     \footnotesize{RDC} & \footnotesize{LSTM\_S+T}\\
		\midrule
		\multirow{2}*{$1\!\rightarrow\! 2$} & RMSE & $\bm{0.0783}$ & $0.0889$ & $0.1114$ & $0.0787$ & $0.1010$ & $0.0857$ & $0.1035$ & $0.0811$ & $0.1166$ & $0.0896$ & $0.0973$ \\
% 		\cline{2-9}
		~ & MAE & $0.2154$ & $0.2284$ & $0.2545$ & $\bm{0.1967}$ & $0.2470$ & $0.2454$ & $0.2551$ & ${0.2151}$ & $0.2470$ & $0.2175$ & $0.2242$ \\
		\midrule
		\multirow{2}*{$1\!\rightarrow\! 3$} & RMSE & $\bm{0.0768}$ & $0.0832$ & $0.0871$ & $0.0771$ & $0.0993$ & $0.0969$ & $0.0832$ & $0.0870$ & $0.0944$ & $0.0927$ & $0.0866$\\
% 		\cline{2-9}
		~ & MAE & ${0.2209}$ & $0.2272$ & $0.2274$ & $\bm{0.2162}$ & $0.2405$ & $0.2293$ & $0.2318$ & $0.2216$ & $0.2322$ & $0.2269$ & $0.2266$\\
		\midrule
		\multirow{2}*{$1\!\rightarrow\! 4$} & RMSE & $0.0933$ & $0.1082$ & $0.0950$ & $\bm{0.0808}$ & $0.1035$ & $0.1076$ & $0.1222$ & $0.0920$ & $0.0952$ & $0.0939$ & $0.0979$\\
% 		\cline{2-9}
		~ & MAE & $0.2286$ & $0.2398$ & $0.2286$ & $\bm{0.2124}$ & $0.2294$ & $0.2387$ & $0.2701$ & $0.2203$ & $0.2276$ & $0.2293$ & $0.2291$\\
		\midrule
		\multirow{2}*{$2\!\rightarrow\! 1$} & RMSE & $\bm{0.0603}$ & $0.0904$ & $0.0830$ & $0.0615$ & $0.0624$ & $0.0688$ & $0.0917$ & $0.0612$ & $0.0710$ & $0.0668$ & $0.0708$\\
% 		\cline{2-9}
		~ & MAE & $\bm{0.1808}$ & $0.2457$ & $0.2084$ & $0.1859$ & $0.1815$ & $0.1896$ & $0.2199$ & $0.1831$ & $0.1879$ & $0.1861$ & $0.1882$\\
		\midrule
		\multirow{2}*{$2\!\rightarrow\! 3$} & RMSE & $0.0647$ & $0.0827$ & $0.0640$ & $0.0684$ & $0.0640$ & $0.0709$ & $0.0843$ & $0.0665$ & $\bm{0.0627}$ & $0.0704$ & $0.0669$\\
% 		\cline{2-9}
		~ & MAE & $0.1915$ & $0.2111$ & $0.1857$ & $0.1982$ & $0.1850$ & $0.1962$ & $0.2159$ & $0.1941$ & $\bm{0.1797}$ & $0.1980$ & $0.1948$\\
		\midrule
		\multirow{2}*{$2\!\rightarrow \!4$} & RMSE & $\bm{0.0923}$ & $0.1052$ & $0.1149$ & $0.0999$ & $0.1034$ & $0.1018$ & $0.1318$ & $0.0947$ & $0.1033$ & $0.0975$ & $0.1003$\\
% 		\cline{2-9}
		~ & MAE & $0.2232$ & $0.2420$ & $0.2583$ & $0.2361$ & $\bm{0.2075}$ & $0.2360$ & $0.2652$ & $0.2248$ & $0.2430$ & $0.2271$ & $0.2335$\\
		\midrule
		\multirow{2}*{$3\!\rightarrow\! 1$} & RMSE & $\bm{0.0830}$ & $0.0899$ & $0.0962$ & $0.0898$ & $0.0874$ & $0.0913$ & $0.1145$ & $0.0831$ & $0.0861$ & $0.0864$ & $0.4633$\\
% 		\cline{2-9}
		~ & MAE & $0.2199$ & $0.2228$ & $0.2226$ & $\bm{0.1966}$ & $0.2163$ & $0.2301$ & $0.2389$ & $0.2024$ & $0.2071$ & $0.2118$ & $0.2100$\\
		\midrule
		\multirow{2}*{$3\!\rightarrow\! 2$} & RMSE & $0.0649$ & $0.0762$ & $0.0709$ & $0.0620$ & $\bm{0.0590}$ & $0.0764$ & $0.0850$ & $0.0693$ & $0.0613$ & $0.0670$ & $0.0596$\\
% 		\cline{2-9}
		~ & MAE & $0.1947$ & $0.2087$ & $0.1999$ & $0.1857$ & $0.1790$ & $0.2022$ & $0.2185$ & $0.1977$ & $\bm{0.1754}$ & $0.1950$ & $0.1807$\\
		\midrule
		\multirow{2}*{$3\!\rightarrow\! 4$} & RMSE & $0.1072$ & $0.1220$ & $0.1155$ & $\bm{0.0955}$ & $0.1221$ & $0.1148$ & $0.1129$ & $0.1081$ & $0.1124$ & $0.1089$ & $0.1101$\\
% 		\cline{2-9}
		~ & MAE & $0.2452$ & $0.2659$ & $0.2550$ & $\bm{0.2301}$ & $0.2574$ & $0.2548$ & $0.2505$ & $0.2382$ & $0.2474$ & $0.2462$ & $0.2446$\\
		\midrule
		\multirow{2}*{$4\!\rightarrow\! 1$} & RMSE & $0.0559$ & $0.0632$ &$0.0567$  & $0.0539$ & $0.0562$ & $0.0624$ & $0.0724$ & $\bm{0.0504}$ & $0.2940$ & $0.5513$ & $0.0519$\\
% 		\cline{2-9}
		~ & MAE & $0.1759$ & $0.1868$ & $0.1759$ & $\bm{0.1601}$ & $0.1775$ & $0.1826$ & $0.1985$ & $0.1621$ & $0.1710$ & $0.1735$ & $0.1684$\\
		\midrule
		\multirow{2}*{$4\!\rightarrow\! 2$} & RMSE & $\bm{0.0603}$ & $0.0691$ & $0.0915$ & $0.0695$ & $0.0786$ & $0.0790$ & $0.0828$ & $0.0701$ & $0.0756$ & $0.0764$ & $0.0698$\\
% 		\cline{2-9}
		~ & MAE & $\bm{0.1930}$ & $0.1988$ & $0.2303$ & $0.2077$ & $0.2077$ & $0.2065$ & $0.2174$ & $0.1935$ & $0.2095$ & $0.2034$ & $0.1948$\\
		\midrule
		\multirow{2}*{$4\!\rightarrow\! 3$} & RMSE & $\bm{0.0661}$ & $0.0745$ & $0.0737$ & $0.0664$ & $0.0732$ & $0.0793$ & $0.0767$ & $0.0697$ & $0.0718$ & $0.0706$ & $0.0686$ \\
% 		\cline{2-9}
		~ & MAE & $\bm{0.1965}$ & $0.1995$ & $0.2044$ & $0.2003$ & $0.1993$ & $0.2038$ & $0.2089$ & $0.1968$ & $0.1992$ & $0.1986$ & $0.1975$\\
		\midrule
		\multirow{2}*{Average} & RMSE & $\bm{0.0753}$ & $0.0878$ & $0.0883$ & $\bm{0.0753}$ & $0.0842$ & $0.0862$ & $0.0967$ & $0.0778$ & $0.1037$ & $0.1226$ & $0.1119$ \\
% 		\cline{2-9}
		~ & MAE & $0.2071$ & $0.2230$ & $0.2209$ & $0.2024$ & $0.2107$ & $0.2179$ & $0.2325$ & $\bm{0.2042}$ & $0.2106$ & $0.2094$ & $0.2077$\\
		\bottomrule
	\end{tabular}
\end{table*}

\subsection{\textcolor{black}{Experiments Results on PEMS-BAY Dataset}}
\textcolor{black}{We also consider the PEMS-Bay dataset, another high-dimension time-series forecasting dataset. Experiment results are shown in Table \ref{tab:traffic}. Compared with the results on Electricity Load Diagrams datasets the GCA model performs better in the PEMS-Bay dataset. This is because clear causal relationships exist in sensors and our method can capture and further utilize these relationships for adaptive forecasting.}

\begin{table*}[t]
\color{black}
	\centering
 \small
	\caption{The MAE and MSE on Traffic dataset for the baselines and the proposed method. }
 \label{tab:traffic}
	\begin{tabular}{p{1.0cm}<{\centering}|p{0.85cm}<{\centering}|p{0.85cm}<{\centering} p{0.85cm}<{\centering} p{0.85cm}<{\centering} p{0.85cm}<{\centering} p{0.85cm}<{\centering} p{0.85cm}<{\centering} p{1.20cm}
    <{\centering} p{1.15cm}<{\centering} p{1.20cm}
    <{\centering} p{0.85cm}<{\centering} p{1.20cm}<{\centering}}
	    \toprule
	    \footnotesize{Task} & \footnotesize{Metric} & \footnotesize{GCA} & \footnotesize{DAF}& \footnotesize{LIRR}& \footnotesize{SASA}& \footnotesize{VRADA}& \footnotesize{CMTN}& \footnotesize{Autoformer} &
     \footnotesize{S4D} &
     \footnotesize{R-DANN} &
     \footnotesize{RDC} & \footnotesize{LSTM\_S+T}\\
		\midrule
		\multirow{2}*{$1\!\rightarrow \!2$} & RMSE & $\bm{0.1277}$ & $0.1317$ & $0.1392$ & $0.1349$ & $0.1430$ & $0.1430$ & $0.1295$ & $0.1296$ & $0.1353$ & $0.1288$ & $0.1336$\\
% 		\cline{2-9}
		~ & MAE & $\bm{0.1917}$ & $0.2095$ & $0.2120$ & $0.2083$ & $0.2155$ & $0.2061$ & $0.1953$ & $0.1988$ & $0.1971$ & $0.2002$ & $0.1972$\\
		\midrule
		\multirow{2}*{$1\!\rightarrow\! 3$} & RMSE & $0.1277$ & $0.1374$ & $0.1736$ & $0.1553$ & $0.1634$ & $0.1376$ & $\bm{0.0809}$ & $0.1294$ & $0.1406$ & $0.1586$ & $0.1730$\\
% 		\cline{2-9}
		~ & MAE & $0.2180$ & $0.2370$ & $0.2330$ & $0.2194$ & $0.2223$ & $0.2207$ & $\bm{0.1708}$ & $0.2196$ & $0.2167$ & $0.2244$ & $0.2399$\\
		\midrule
		\multirow{2}*{$2\!\rightarrow\! 1$} & RMSE & $\bm{0.0974}$ & $0.1298$ & $0.1407$ & $0.1304$ & $0.1266$ & $0.1289$ & $0.1321$ & $0.1154$ & $0.1350$ & $0.1358$ & $0.1260$\\
% 		\cline{2-9}
		~ & MAE & $\bm{0.1821}$ & $0.2211$ & $0.2313$ & $0.2139$ & $0.2017$ & $0.2144$ & $0.2177$ & $0.1950$ & $0.2098$ & $0.2117$ & $0.2064$\\
		\midrule
		\multirow{2}*{$2\!\rightarrow\! 3$} & RMSE & $\bm{0.1156}$ & $0.1295$ & $0.1540$ & $0.1639$ & $0.1650$ & $0.1338$ & $0.1260$ & $0.1409$ & $0.1427$ & $0.1632$ & $0.1347$\\
% 		\cline{2-9}
		~ & MAE & $\bm{0.1942}$ & $0.2125$ & $0.2237$ & $0.2324$ & $0.2285$ & $0.2152$ & $0.2139$ & $0.2205$ & $0.2061$ & $0.2254$ & $0.2044$\\
		\midrule
		\multirow{2}*{$3\!\rightarrow \!1$} & RMSE & $\bm{0.1104}$ & $0.1197$ & $0.1523$ & $0.1258$ & $0.1294$ & $0.1277$ & $0.1378$ & $0.1173$ & $0.1332$ & $0.1362$ & $0.1364$\\
% 		\cline{2-9}
		~ & MAE & $\bm{0.1911}$ & $0.2052$ & $0.2354$ & $0.2169$ & $0.2140$ & $0.2046$ & $0.2324$ & $0.2034$ & $0.2106$ & $0.2195$ & $0.2068$\\
		\midrule
		\multirow{2}*{$3\!\rightarrow\! 2$} & RMSE & $0.1370$ & $0.1385$ & $0.1369$ & $0.1433$ & $0.1362$ & $\bm{0.1249}$ & $0.1615$ & $0.1415$ & $0.1315$ & $0.1375$ & $0.1346$\\
% 		\cline{2-9}
		~ & MAE & $0.2044$ & $0.2180$ & $0.2103$ & $0.2057$ & $\bm{0.1905}$ & $0.2030$ & $0.2358$ & $0.2075$ & $0.1925$ & $0.1957$ & $0.1920$\\
		\midrule
		\multirow{2}*{Average} & RMSE & $\bm{0.1193}$ & $0.1311$ & $0.1494$ & $0.1423$ & $0.1439$ & $0.1327$ & $0.1279$ & $0.1290$ & $0.1364$ & $0.1433$ & $0.1397$\\
% 		\cline{2-9}
		~ & MAE & $\bm{0.1969}$ & $0.2172$ & $0.2243$ & $0.2161$ & $0.2121$ & $0.2106$ & $0.2110$ & $0.2075$ & $0.2054$ & $0.2128$ & $0.2078$\\
		\bottomrule
	\end{tabular}
\end{table*}

\subsection{\textcolor{black}{The Study of Different Prediction Length.}}
\begin{figure}[t]
  \centering
  \subfigure[Cycling $\rightarrow$ Driving]{\begin{minipage}[t]{0.33\textwidth}
    \centering
    \includegraphics[width=\textwidth]{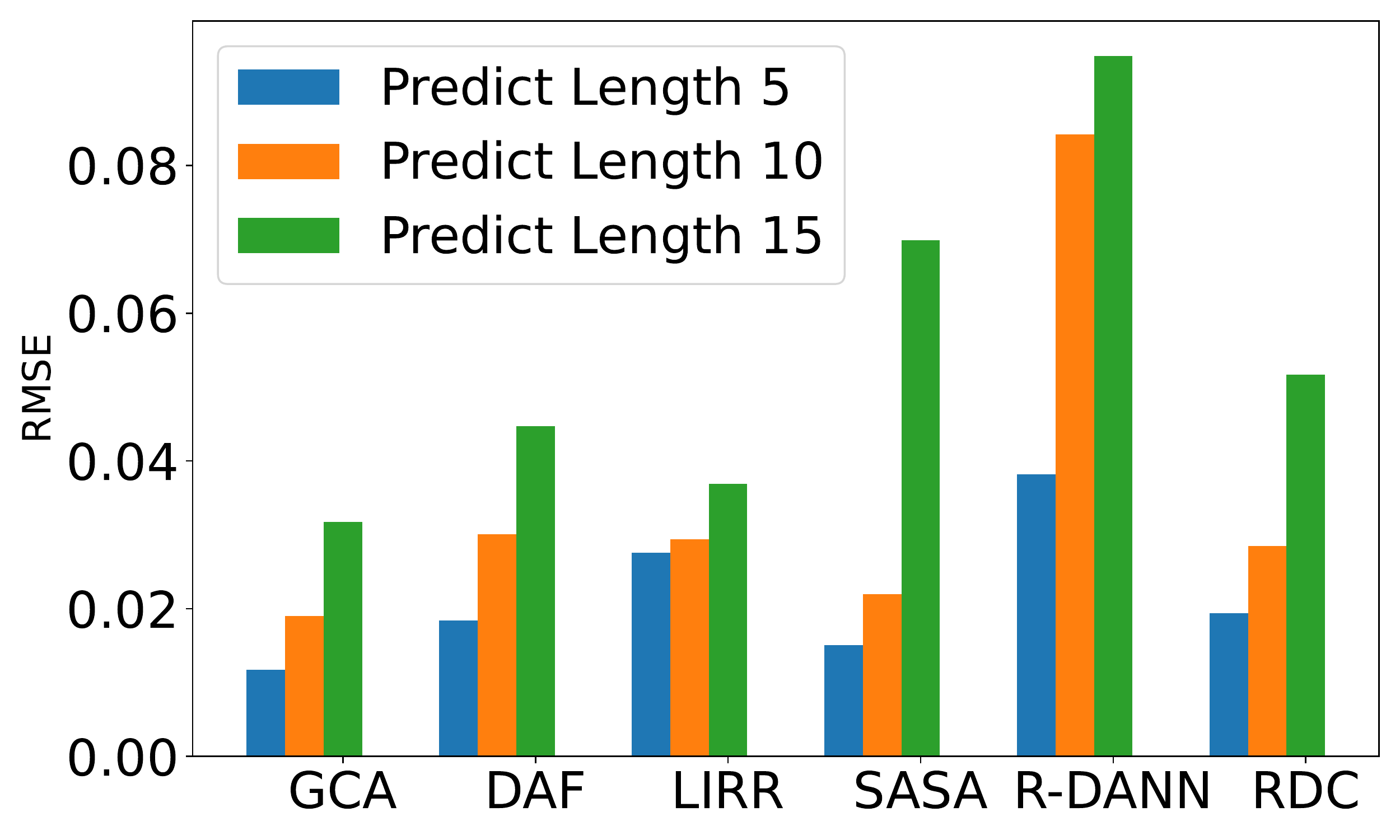}\captionsetup{justification=centering}
  \end{minipage}}
  \hfill
  \subfigure[Cycling $\rightarrow$ Sitting]{\begin{minipage}[t]{0.33\textwidth}
    \centering
\includegraphics[width=\textwidth]{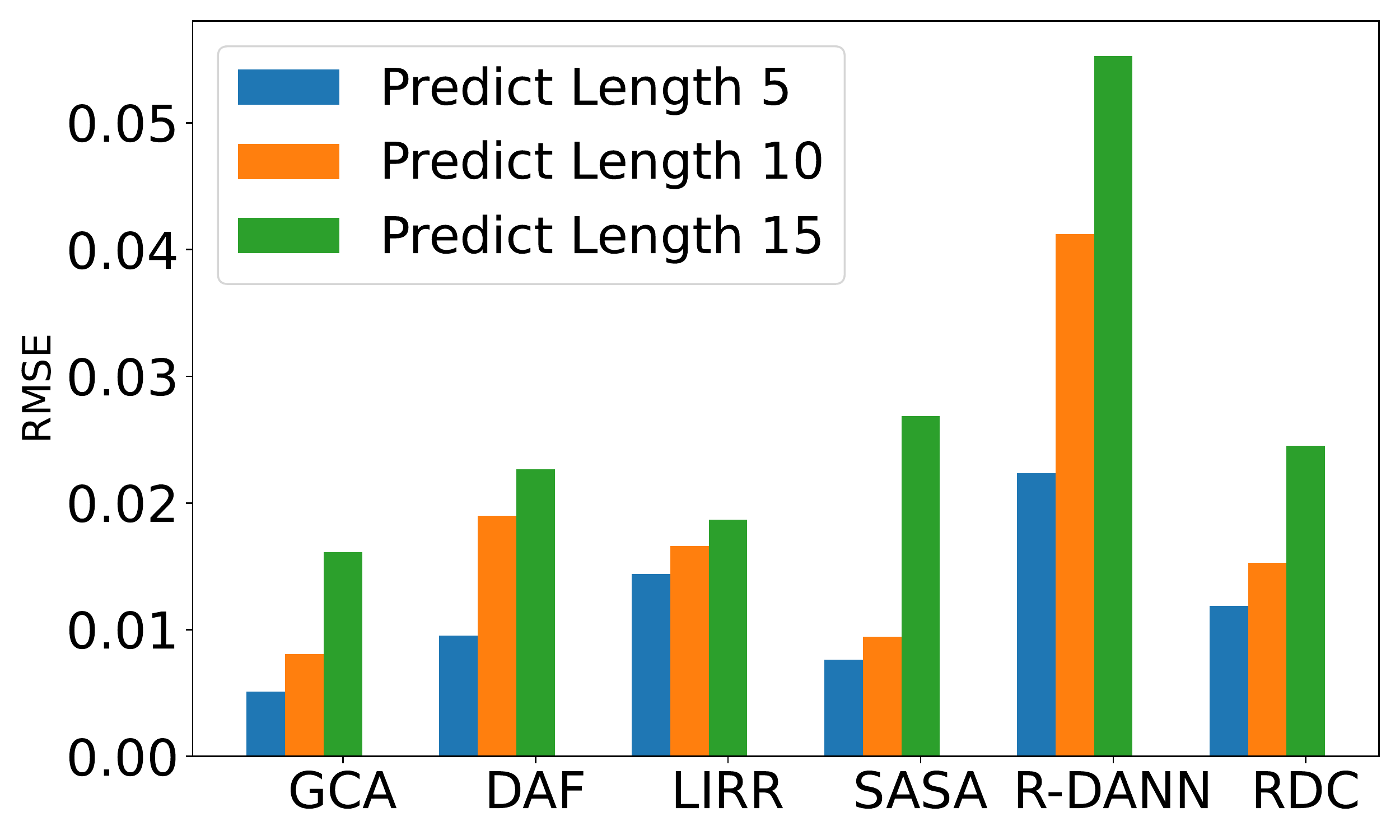}\captionsetup{justification=centering}
  \end{minipage}}
  \hfill
  \subfigure[Cycling $\rightarrow$ Working]{\begin{minipage}[t]{0.33\textwidth}
    \centering
\includegraphics[width=\textwidth]{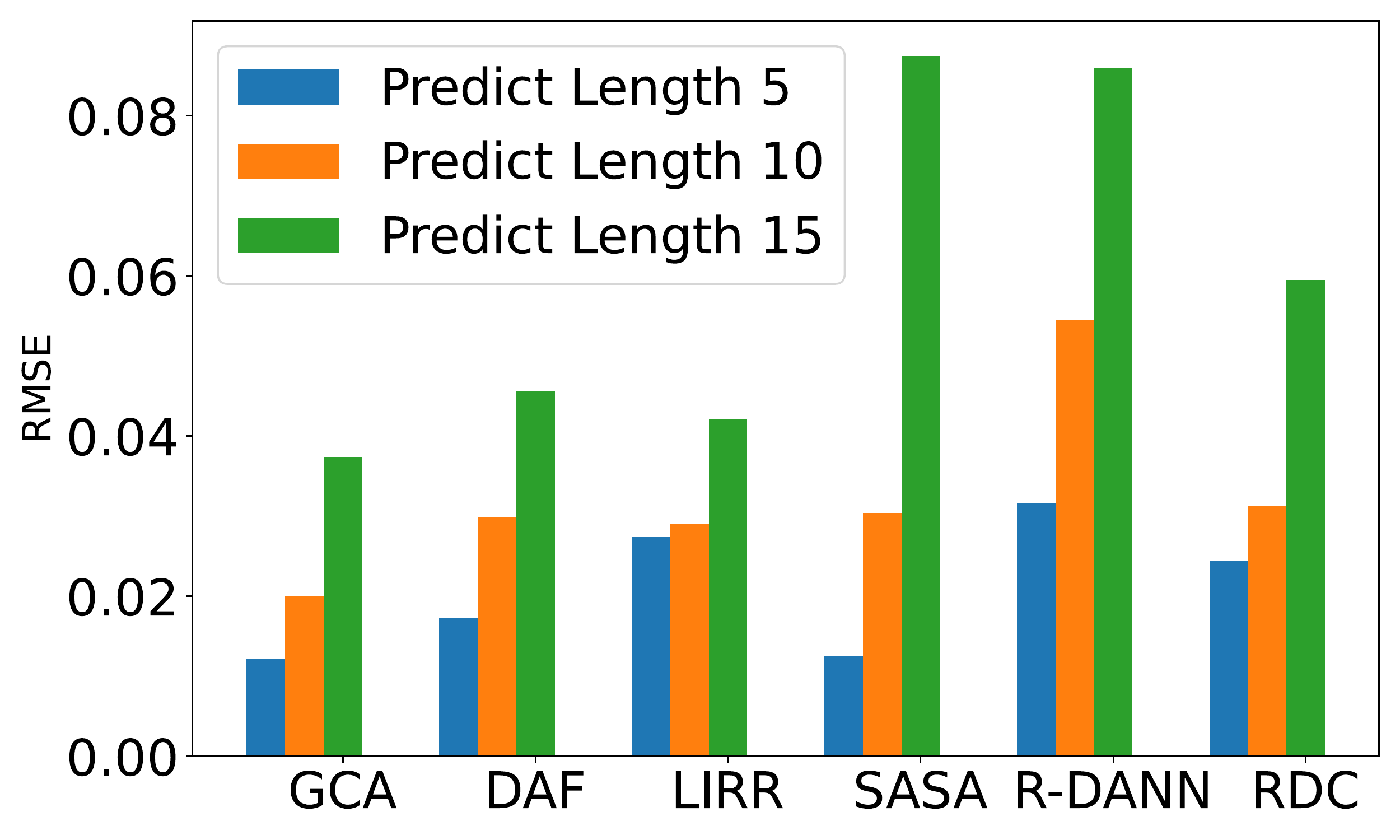}\captionsetup{justification=centering}
  \end{minipage}}
\centering
\caption{\textcolor{black}{Experiment results of different predicted length.}}
\label{fig:length}
\end{figure}

\textcolor{black}{we further provide the robustness analysis of prediction length. In detail, we consider different prediction lengths on the Heart Rate Prediction task, e.g. 5, 10, 
15. Experiment results on several tasks are shown in Figure \ref{fig:length} (a)(b) and (c). According to the experiment results, we can find that the value of RMSE drops as the increment of prediction length, but our GCA method can still achieve the best results, which evaluates its robustness.}

% \clearpage
\bibliographystyle{IEEEtran}
\bibliography{ref}